\documentclass[11pt,letterpaper]{article}
\ifdefined\pdfminorversion
\else
  \AtBeginDvi{}
\fi
\PassOptionsToPackage{table}{xcolor}
\usepackage[margin=1in]{geometry}
\usepackage{newtxtext}
\usepackage{graphicx,booktabs,amsmath,amssymb,caption,subcaption,etoolbox}
\usepackage[numbers,sort&compress]{natbib}
\usepackage[T1]{fontenc}
\usepackage{microtype,xurl,amsfonts,mathtools,bm,enumitem,amsthm,aliascnt}
\usepackage{multirow,nicematrix,needspace,placeins}
\usepackage{algorithm,algorithmic,titletoc}
\usepackage{tikz}
\usetikzlibrary{positioning,arrows.meta,fit,calc}
\usepackage{xr-hyper}
\definecolor{cvprblue}{rgb}{0.21,0.49,0.74}
\usepackage[breaklinks,colorlinks,allcolors=cvprblue]{hyperref}
\usepackage[capitalize]{cleveref}
\definecolor{lightCyan}{RGB}{230,243,250}
\definecolor{deltaGray}{gray}{0.40}
\definecolor{groupShade}{gray}{0.94}
\newcommand{\best}[1]{\textbf{#1}}
\newcommand{\second}[1]{\underline{#1}}
\newcommand{\dt}[1]{\textcolor{deltaGray}{\scriptsize #1}}
\newcommand{\sg}{\operatorname{sg}}

\hypersetup{pdfauthor={Ruoyu Wang, Xinyi Ye, Yunshuo Shi, Liangyu Yuan, Beier Zhu, Chi Zhang},pdftitle={Variational Optimal Control for Diffusion Guidance}}

\newcommand{\sharedthm}[3]{%
  \newaliascnt{#1}{theorem}%
  \newtheorem{#1}[#1]{#2}%
  \aliascntresetthe{#1}%
  \crefname{#1}{#2}{#3}%
  \Crefname{#1}{#2}{#3}%
}
\newtheoremstyle{vocplain}{0.7\baselineskip}{0.7\baselineskip}{\itshape}{0pt}{\bfseries}{.}{0.5em}{}
\newtheoremstyle{vocdefn}{0.7\baselineskip}{0.7\baselineskip}{\normalfont}{0pt}{\bfseries}{.}{0.5em}{}
\newtheoremstyle{vocremark}{0.7\baselineskip}{0.7\baselineskip}{\normalfont}{0pt}{\bfseries}{.}{0.5em}{}
\theoremstyle{vocplain}
\newtheorem{theorem}{Theorem}[section]
\crefname{theorem}{Theorem}{Theorems}
\Crefname{theorem}{Theorem}{Theorems}
\sharedthm{proposition}{Proposition}{Propositions}
\sharedthm{lemma}{Lemma}{Lemmas}
\sharedthm{corollary}{Corollary}{Corollaries}
\theoremstyle{vocdefn}
\sharedthm{definition}{Definition}{Definitions}
\sharedthm{assumption}{Assumption}{Assumptions}
\sharedthm{approximation}{Approximation}{Approximations}
\theoremstyle{vocremark}
\sharedthm{remark}{Remark}{Remarks}

\crefname{appendix}{Appendix}{Appendices}
\Crefname{appendix}{Appendix}{Appendices}
\crefname{equation}{Eq.}{Eqs.}
\Crefname{equation}{Eq.}{Eqs.}
\crefname{section}{Sec.}{Secs.}
\Crefname{section}{Section}{Sections}
\allowdisplaybreaks[1]

\newcommand{\ours}{\textnormal{VOG}}

\input{frontmatter-style}

\title{Variational Optimal Control for Diffusion Guidance}
\newcommand{\affmark}[1]{\raisebox{0.6ex}{\scriptsize\normalfont #1}}
\newcommand{\equalcontrib}{\affmark{*}}
\newcommand{\corresponding}{\affmark{\textdagger}}

\author{%
{\normalfont Ruoyu Wang}\affmark{1,*} \quad
{\normalfont Xinyi Ye}\affmark{1,*} \quad
{\normalfont Yunshuo Shi}\affmark{1,2,*}
\\
Liangyu Yuan\affmark{3} \quad
Beier Zhu\affmark{4} \quad
Chi Zhang\affmark{1,\textdagger}
\\[0.4em]
{\small
\affmark{1}AGI Lab, Westlake University \quad
\affmark{2}Shandong University
}
\\
{\small
\affmark{3}Shanghai Jiao Tong University \quad
\affmark{4}University of Science and Technology of China
}
\\[0.4em]
{\small
\texttt{wangruoyu71@westlake.edu.cn}
}
\\[0.3em]
{\small
\equalcontrib Equal contribution \quad
\corresponding Corresponding author
}
}

\date{\small\textbf{Code:} \url{https://github.com/RuoyuWang-2077/VOG}}
\begin{document}
\raggedbottom
\makeagilabfrontmatter{\begin{abstract}
Sampling-time guidance and reward-based post-training act on different variables but share a core geometric objective: realizing an output-space improvement through a generative model. We organize both around a Jacobian pullback evaluated with respect to either the current noisy state or the trainable parameters. Formulating current-state conditional variational optimal control (VOC) yields Variational Optimal Guidance (VOG), which transports a terminal reward proxy into a covariance-weighted sampling correction. Under a linear-subspace denoising model, the exact pullback has no component normal to the data subspace. Distilled Variational Optimal Control (DVOC) realizes the parameter-space counterpart of this local control principle: projected reward ascent constructs a bounded endpoint displacement within a local trust region, and a sparse, forward-pinned Euler surrogate transports it to LoRA parameters. Empirically, VOG reduces DiT-XL/2 FID on ImageNet from 2.27 to 2.04 at 250 sampling steps and from 7.58 to 4.44 at 25 steps. After 20 updates on Z-Image-Turbo, DVOC improves HPSv3++ and four held-out evaluator scores over the untrained model and leads all trained comparators on these five means, while CLIP remains close to the base and DeQA declines across all trained methods. The shared pullback interface cleanly separates the desired output change from the model derivative and the actuation variable.
\end{abstract}
}
\section{Introduction}
\label{sec:intro}
Diffusion and flow-based generative models transform noise into images through a learned sequence of denoising or transport steps~\cite{Ho2020DDPM,Song2021SDE,Peebles2023DiT,Karras2024EDM2}. A preference is evaluated at the completed output, yet the available intervention acts elsewhere: on a noisy intermediate state during sampling, or on model parameters during post-training. The central problem is therefore one of \emph{transport}: how should a desired output change be propagated back to the variable we can control while preserving the quality and coverage supported by the learned prior?

Classifier-free guidance (CFG)~\cite{ho2022classifierfreediffusionguidance} steers sampling by extrapolating between conditional and unconditional predictions, while gradient-based guidance methods evaluate objectives on a predicted clean image~\cite{chung2023diffusion,yu2023freedom,bansal2024universal}. Increasing guidance strength often trades distributional coverage for alignment and introduces off-manifold artifacts~\cite{chidambaram2024guidance,sadat2024eliminating,guo2024gradient}. Reward-based fine-tuning~\cite{Xu2023ImageReward,Prabhudesai2023AlignProp,Clark2023DirectlyFinetuning,Black2023DDPO} instead embeds preferences directly into model parameters, but requires a reliable update signal and a tractable derivative through the multi-step generation process. Despite acting at different stages, both paradigms must specify an output-space signal, a model-dependent backward map, and an actuation variable.

We use this three-part structure to unify sampling-time guidance and reward-based parameter adaptation under \emph{variational optimal control} (VOC). A model Jacobian pulls an output-space signal back to either the current noisy state or the trainable parameters: the signal may be an infinitesimal terminal reward gradient or a finite, bounded target displacement. This shared pullback interface makes the geometric connection between inference-time control and post-training precise without requiring the two realizations to optimize an identical objective. \Cref{fig:vog_overview} illustrates the common geometry and its two actuation spaces.

For \emph{state-space control}, we formulate current-state conditional VOC: optimize only the future continuation from a realized noisy state, using the pretrained conditional process as the reference path measure. The Gibbs variational identity characterizes the optimal reward tilt, and a short-time Gaussian discretization yields a covariance-preconditioned drift correction. Pulling a denoised terminal reward proxy back through the denoising Jacobian gives \emph{Variational Optimal Guidance} (VOG). Under a linear-subspace denoising model, the exact Jacobian pullback projects strictly onto the data subspace, eliminating components orthogonal to the data support under the stated idealization.

For \emph{parameter-space adaptation}, we introduce \emph{Distilled Variational Optimal Control} (DVOC), which realizes the parameter-space counterpart of this local variational control principle. Rather than backpropagating raw rewards through a full or truncated rollout, DVOC formulates a local trust-region terminal improvement problem at the actual sampled endpoint $\bm{y}=\bm{x}_0$, constructs a bounded displacement via normalized projected reward ascent through image-to-latent decoder pullbacks, and amortizes the resulting detached target into LoRA parameters through the sparse velocity Jacobian of a forward-pinned Euler surrogate. Here distillation refers to amortizing locally optimized terminal control targets directly into the generator without requiring a separate VOG teacher. We further derive an implementation-matched directional condition characterizing when local endpoint improvements transfer to the generated output under sparse parameter updates.

Our contributions are threefold:
\begin{itemize}[leftmargin=*,nosep]
    \item \textbf{Variational state-space guidance.} We derive VOG from current-state conditional VOC with a path-space KL penalty, establish the subspace-preservation guarantee of its exact denoising Jacobian pullback under a linear MMSE model, and introduce a practical factorized, strided, momentum-smoothed controller.
    \item \textbf{Variational parameter-space adaptation.} We formulate DVOC from the shared local control principle, coupling bounded projected endpoint target construction with a sparse forward-pinned Euler surrogate and characterizing its exact surrogate gradient and first-order target realization condition.
    \item \textbf{Comprehensive evaluation across guidance and post-training.} VOG improves FID across all twelve ImageNet backbone--guidance pairings and lowers 25-step DiT-XL/2 FID from 7.58 to 4.44. In a matched 20-update comparison on Z-Image-Turbo across two training rewards and five held-out evaluators, DVOC improves HPSv3++ and four held-out scores over the untrained base and leads all five post-trained comparators on these five means, supported by active-step, target-construction, and local mechanism diagnostics.
\end{itemize}

\section{Related Work}
\label{sec:related}
\paragraph{Sampling-time guidance.}
Classifier guidance~\cite{Dhariwal2021DiffusionBeatGANs} steers sampling with a learned noise-aware classifier gradient, whereas CFG~\cite{ho2022classifierfreediffusionguidance} linearly extrapolates between conditional and unconditional predictions. Subsequent methods modify the direction, magnitude, or schedule of guidance to improve the fidelity--coverage trade-off~\cite{chung2024cfg++,sadat2024cads,sadat2024eliminating,kynkaanniemi2024applying,Karras2024autoguidance,xia2025rectified}. Manifold-aware analyses motivate corrections tied to denoiser geometry~\cite{guo2024gradient,he2024manifold}, while training-free inverse-problem and energy-guided samplers propagate external objectives through a one-step clean prediction~\cite{chung2023diffusion,yu2023freedom,bansal2024universal}. Our inference formulation derives a local Jacobian-aware state correction from a current-state conditional variational objective using the endogenous conditional--unconditional score contrast, with explicit distinctions between the exact control, its short-time approximation, and its factorized practical surrogate.

\paragraph{Reward-based post-training.}
Preference optimization and reinforcement learning adapt diffusion and flow models using policy gradients, pairwise preferences, or advantage-weighted flow matching~\cite{Black2023DDPO,Fan2023DPOK,Wallace2024DiffusionDPO,Liu2025FlowGRPO,Zheng2025DiffusionNFT,Xue2026AWM}. Differentiable-reward methods—including ReFL, DRaFT, AlignProp, and DRTune—instead backpropagate gradients of learned evaluators through full or truncated sampling rollouts~\cite{Xu2023ImageReward,Clark2023DirectlyFinetuning,Prabhudesai2023AlignProp,wu2024drtune}, while on-policy distillation methods regress against reward-updated clean predictions~\cite{zhou2026diffusionopsd}. In contrast, we formulate reward-based post-training as amortized terminal control under the same Jacobian pullback interface as sampling-time guidance: DVOC solves a bounded trust-region improvement problem at the actual sampled endpoint and transports the detached displacement into LoRA parameters through a sparse, forward-pinned Euler Jacobian, accompanied by an explicit local directional realization analysis.

\paragraph{Optimal control perspectives.}
Stochastic optimal control provides a principled language for steering and adapting generative dynamics. Adjoint Matching~\cite{DomingoEnrich2025AdjointMatching}, for example, studies globally reward-tilted diffusion and flow processes to derive a continuous-time adjoint regression objective, while DEFT and h-Edit~\cite{denker2024deft,nguyen2025hedit} employ Doob's $h$-transform for conditional sampling and image editing. Our VOG analysis conditions on a realized current state during sampling and optimizes only its future continuation. DVOC provides a deterministic post-training counterpart inspired by local terminal control, amortizing bounded endpoint improvements into parameters while explicitly distinguishing its Euclidean endpoint trust region from a stochastic path-space KL constraint.

\section{Variational Optimal Control and Model Geometry}
\begin{figure*}[t]
\centering
\includegraphics[width=\textwidth]{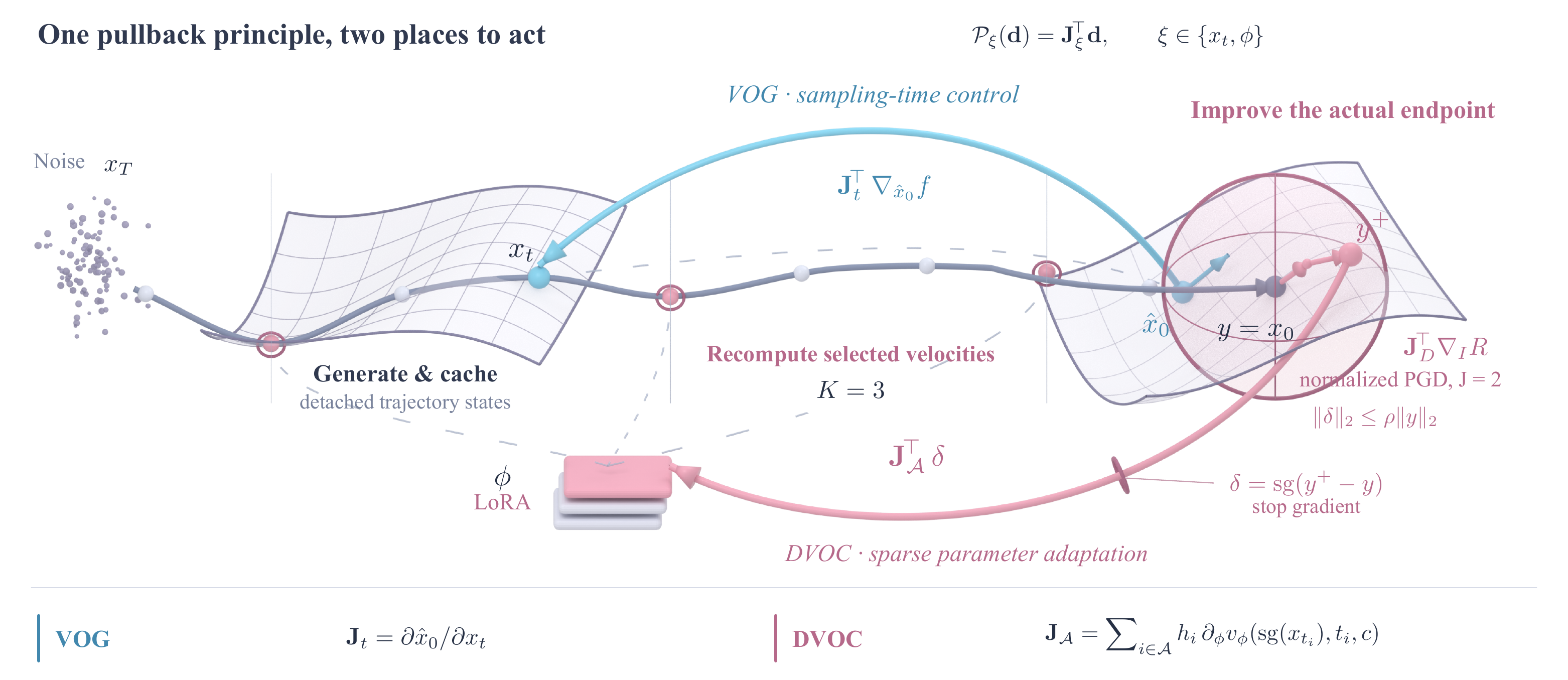}
\caption{\textbf{One Jacobian pullback interface, two actuation spaces.} Along a reverse generation trajectory, VOG (blue, top) pulls a terminal reward proxy at the one-step denoised prediction $\hat{\bm{x}}_0$ back through the denoising Jacobian $\mathbf{J}_t$ to correct the current noisy state $\bm{x}_t$, using momentum smoothing and strided, factorized derivative evaluations. DVOC (pink, bottom) anchors optimization at the actual sampled endpoint $\bm{y}=\bm{x}_0$, applies $J=2$ normalized projected ascent steps via decoder pullbacks $\mathbf{J}_D^\top\nabla_{\bm{I}} R$ within $\mathcal{B}(\bm{y},r)$, and pulls the detached displacement $\bm{\delta}$ back to LoRA parameters $\phi$ through the sparse velocity Jacobian $\mathbf{J}_{\mathcal{A}}$. Ringed states mark the $K=3$ active velocity recomputations, and the vertical gate denotes the stop-gradient boundary.}
\label{fig:vog_overview}
\end{figure*}
\paragraph{A shared pullback interface.}
An output-space signal $\bm{d}$ becomes an actionable correction via a Jacobian pullback, where the differentiation variable $\xi$ selects the actuation space:
\begin{equation}
 \mathcal{P}_\xi(\bm{d})=\mathbf{J}_\xi^\top\bm{d},
 \qquad \xi\in\{\bm{x}_t,\phi\}.
 \label{eq:shared_pullback}
\end{equation}
More generally, let $G_\xi$ map the control variable $\xi$ to an output representation and write $\mathbf{J}_\xi=\partial G_\xi/\partial\xi$. Linearizing the output response around the current evaluation and penalizing control energy by a quadratic cost yields the local variational control rule
\begin{equation}
 \bm{u}_\xi^*=\arg\max_{\bm{u}}\left\{
 \langle\bm{d},\mathbf{J}_\xi \bm{u}\rangle-\frac{\|\bm{u}\|_2^2}{2\gamma}
 \right\}=\gamma\mathbf{J}_\xi^\top\bm{d},
 \label{eq:local_control_principle}
\end{equation}
where $\gamma>0$ sets the control scale. VOG instantiates this pullback in state space ($\xi=\bm{x}_t$), using the denoising Jacobian $\mathbf{J}_{\bm{x}_t}=\mathbf{J}_t$ to transport a terminal reward proxy gradient back to the current noisy state $\bm{x}_t$. DVOC instantiates the parameter-space counterpart ($\xi=\phi$): it first pulls the reward gradient through the frozen image decoder to construct a bounded endpoint displacement within a local trust region, and then transports that detached displacement to LoRA parameters through the sparse velocity Jacobian $\mathbf{J}_\phi=\mathbf{J}_{\mathcal{A}}$. Below, we derive VOG's state-space control objective; \Cref{sec:training} then develops DVOC's two derivative stages, the stop-gradient boundary between them, and its local target realization condition.

To analyze VOG's geometry, we adopt a low-dimensional linear-subspace model~\cite{pope2021intrinstic,guo2024gradient}. This idealization isolates how the denoiser constrains a state correction, without assuming a general neural sampler preserves a nonlinear manifold.
\begin{assumption}[Linear Latent Subspace]
\label{assump:linear_subspace}
We assume the clean data $\bm{x}^* \in \mathbb{R}^D$ resides strictly within a low-dimensional linear subspace $\mathcal{S}$ of intrinsic dimension $d \ll D$. 
Specifically, the data is generated as $\bm{x}^* = \mathbf{U}\bm{z}$, where $\mathbf{U} \in \mathbb{R}^{D \times d}$ is an orthonormal basis matrix satisfying $\mathbf{U}^\top \mathbf{U} = \mathbf{I}_d$, and the latent variable $\bm{z}$ follows a standard Gaussian prior $\bm{z} \sim \mathcal{N}(\bm{0}, \mathbf{I}_d)$. 
Consequently, the data support is defined as $\mathcal{S} \triangleq \operatorname{span}(\mathbf{U})$.
\end{assumption}

Any vector $\bm{v} \in \mathbb{R}^D$ then decomposes uniquely as $\bm{v} = \mathbf{P}_{\parallel}\bm{v} + \mathbf{P}_{\perp}\bm{v}$ with $\mathbf{P}_{\parallel} \triangleq \mathbf{U}\mathbf{U}^\top$ and $\mathbf{P}_{\perp} \triangleq \mathbf{I}_D - \mathbf{U}\mathbf{U}^\top$: the tangential part lies in the data support, while the normal part measures deviation from the subspace.

\subsection{The Geometric Pitfall of CFG}
\label{subsec:cfg_pitfalls}
\begin{figure*}[t]
    \centering
    \includegraphics[width=1.0\linewidth]{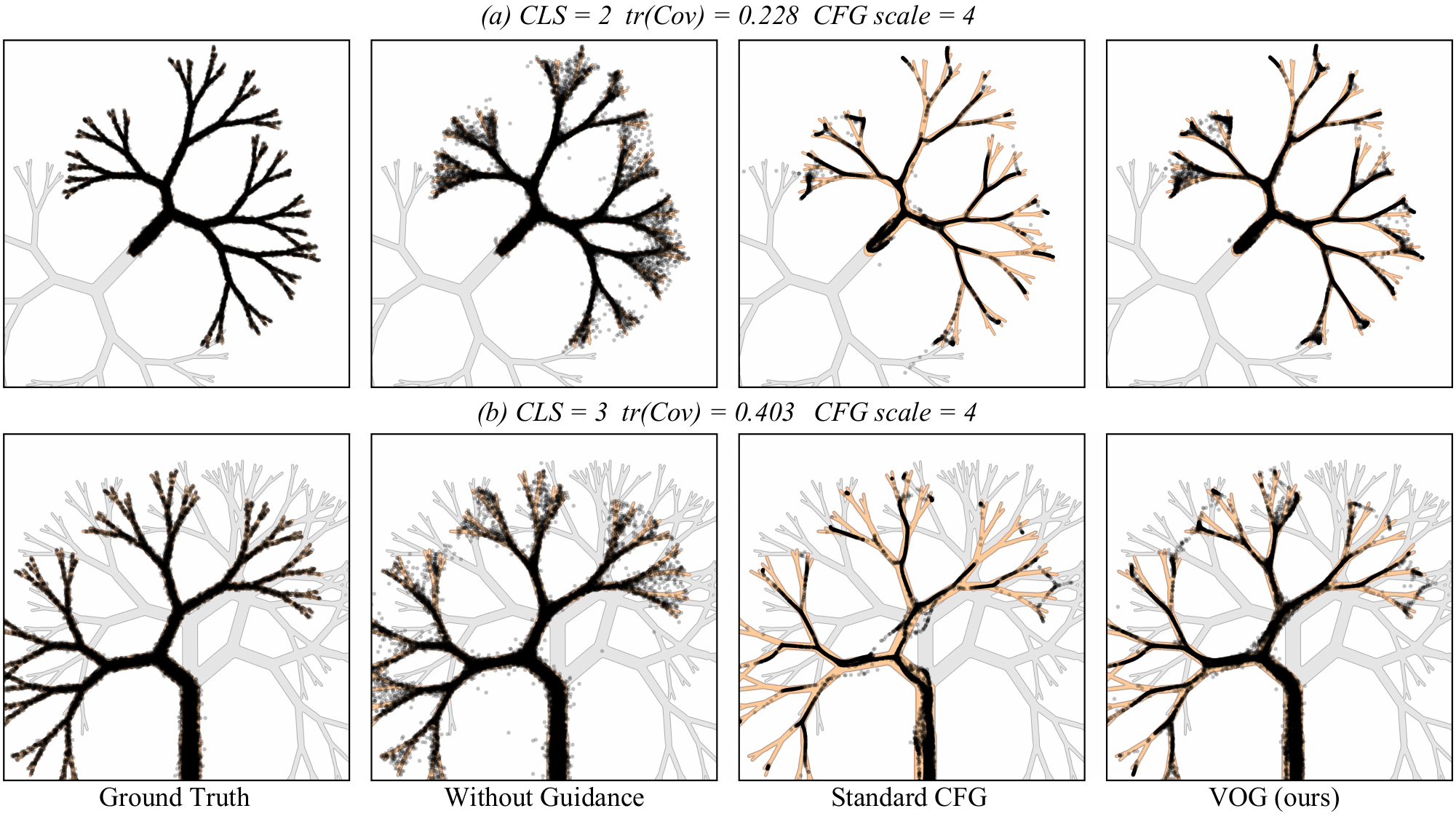}
    \caption{\textbf{Class-conditional generation on 2D branching-tree distributions~\cite{Karras2024autoguidance}.}
Top and bottom rows compare a two-class benchmark (CLS$=2$, $\mathrm{tr}(\mathrm{Cov})=0.228$) and a harder three-class variant with stronger inter-class overlap and higher spatial dispersion (CLS$=3$, $\mathrm{tr}(\mathrm{Cov})=0.403$; see \Cref{tab:synthetic_metrics} for distribution complexity metrics).
From left to right: \textbf{(a)} ground-truth samples from the target class distribution, \textbf{(b)} unguided conditional sampling ($\omega=1$) exhibiting off-branch outliers and class leakage, \textbf{(c)} standard CFG ($\omega=4$) suppressing outliers at the cost of severe branch contraction and peripheral mode loss, and \textbf{(d)} $\ours$ (ours) preserving full branch coverage while maintaining sharp manifold alignment (quantitative scores in \Cref{tab:toy-merged}).}
    \label{fig:toy_example}
    \vspace{-2mm}
\end{figure*}

While Classifier-Free Guidance (CFG) is the standard mechanism for conditional generation, it relies on a geometrically agnostic linear extrapolation in noise space. \Cref{fig:toy_example} illustrates the resulting trade-off on a 2D branching-tree distribution, where strong guidance suppresses outliers but visibly contracts branch coverage. To examine the mechanism behind this high-guidance instability, we analyze CFG under the Probability Flow ODE framework; full derivations are deferred to~\Cref{sec:cfg_span_shift}.

\noindent\textbf{Guidance-induced mean shift.}
A first-order perturbation analysis of the VP probability flow ODE yields
\begin{equation}
\label{eq:cfg_final_shift}
    \begin{aligned}
        \mathbb{E}_{q_\omega}[\bm{x}_0] &\approx \bm{\mu}_{\text{global}} + \omega\,\bm{\Delta}_{\text{eff}}, \\
        \bm{\Delta}_{\text{eff}}
        &\triangleq
        \mathbb{E}\!\left[\int_0^T \bm{\Psi}(\tau)\big(\bm{\mu}_c-\bar{\bm{\mu}}(\bm{x}_\tau)\big)\,d\tau\right],
    \end{aligned}
\end{equation}
where the expectation is over samples $\bm{x}_0$ generated by the $\omega$-guided sampler, \(\bm{\mu}_{\text{global}}=\mathbb{E}_{q}[\bm{x}_0]\) is the unconditional data mean, \(\bm{\mu}_c\) is the target class mean, \(\bar{\bm{\mu}}(\bm{x}_\tau)=\sum_y p(y\mid \bm{x}_\tau)\bm{\mu}_y\) is the posterior-weighted dynamic mean, and \(\bm{\Psi}(\tau)\) is the probability flow kernel of~\Cref{app:flow_kernel}. The integrand captures the class-uncertainty gap transported to $t=0$ by the linearized flow. Since \(\bm{\Delta}_{\text{eff}}\) is evaluated along the guided trajectory, it depends implicitly on \(\omega\), so Eq.~\eqref{eq:cfg_final_shift} does not establish a globally linear dependence on \(\omega\).

Writing $\omega=1+w$ separates ordinary conditional prediction ($w=0$) from further extrapolation ($w>0$), motivating our choice in \Cref{sec:voc} to anchor control at the conditional policy. This first-order mean expansion does not itself establish an exact decomposition of the nonlinear guided distribution.

The analysis assumes Bayes-optimal branch mixing and locally affine class denoisers with shared contraction (\Cref{ass:mmse,ass:affine}). These abstractions isolate an analytical mechanism rather than describing every neural denoiser.

\noindent\textbf{Geometric interpretation.}
The shift identifies one mechanism through which guidance alters coverage, though a mean shift alone does not prove mode collapse. Under~\Cref{assump:linear_subspace} and~\Cref{ass:mmse}, an exact subspace denoiser keeps the guided mean inside $\mathcal{S}$. An imperfect denoiser can introduce cross-subspace coupling; whether this yields a nonzero normal mean component also depends on the class contrast and cancellation along the trajectory (\Cref{rem:offspan_mechanism}). The 2D toy benchmark illustrates the empirical precision--recall trade-off.

\subsection{Variational Optimal Control (VOC)}
\label{sec:voc}

We formulate sampling-time guidance as variational optimal control: maximize a terminal objective while penalizing path-space KL deviation from a conditional reference process. \Cref{fig:vog_overview} relates the resulting state correction to DVOC's parameter-space pullback.

\noindent\textbf{Current-state conditional formulation.}
We parameterize denoising time as \(t\in[0,T]\) (\(t{=}T\): noise, \(t{=}0\): clean sample). At time \(t\), the state \(\bm{X}_t=\bm{x}\) is already realized, so guidance can modify only the future continuation from \(t\) to \(0\). Let \(P_{b,t:0}(\cdot\mid \bm{X}_t{=}\bm{x})\) denote the future path measure of a fixed baseline policy \(p_b\), and let \(f(\bm{X}_0,c)\) denote the terminal preference objective on clean samples. We seek a controlled path measure \(Q_{t:0}(\cdot\mid \bm{X}_t{=}\bm{x})\) solving
\begin{equation}
\max_Q\left\{\lambda\,\mathbb{E}_Q[f(\bm{X}_0,c)]-D_{\mathrm{KL}}(Q\|P_b)\right\},
\label{eq:conditional_voc_continuous}
\end{equation}
where $Q$ and $P_b$ abbreviate the future path measures conditioned on $\bm{X}_t=\bm{x}$. Conditioning fixes the current state, though subsequent tractable approximations still incur approximation bias.

\noindent\textbf{Conditional reference.}
For the stochastic VOG formulation, $p_b$ is the pretrained conditional policy predicting $\bm{\epsilon}_c=\bm{\epsilon}_\theta(\bm{x}_t,t,c)$, so zero control recovers ordinary conditional sampling. The terminal objective is motivated by the endogenous CFG log-likelihood ratio with control strength $w\triangleq\omega-1$,
\begin{equation}
f_{\mathrm{cfg}}(\bm{X}_0,c)=\log p_\theta(\bm{X}_0\mid c)-\log p_\theta(\bm{X}_0),
\label{eq:cfg_terminal_reward}
\end{equation}
whereas DVOC uses an external endpoint reward and trust-region surrogate (\Cref{sec:training}). Because the reference in VOG already carries the conditioning, the active control corresponds to the extrapolation beyond it.

\noindent\textbf{Optimal solution via exponential tilting.}
Define the desirability \(h_t(\bm{x})\triangleq\mathbb{E}_{P_{b,t:0}}[\exp(\lambda f(\bm{X}_0,c))\mid \bm{X}_t{=}\bm{x}]\) and its log \(\ell_t(\bm{x})\triangleq\log h_t(\bm{x})\).
By the Gibbs variational identity, the unique maximizer of Eq.~\eqref{eq:conditional_voc_continuous} is the exponentially tilted path measure (proved in~\Cref{app:voc_to_jacobian_reg}):
\begin{equation}
\frac{dQ_{t:0}^{*}}{dP_{b,t:0}}
=
\frac{\exp(\lambda f(\bm{X}_0,c))}{h_t(\bm{x})}.
\label{eq:continuous_conditional_tilt}
\end{equation}
For a standard It\^o SDE \(d\bm{X}_t=\bm{b}_\theta(\bm{X}_t,t)dt+\bm{\sigma}_t(\bm{X}_t)d\bm{W}_t\) with \(\mathbf{a}_t(\bm{x})=\bm{\sigma}_t(\bm{x})\bm{\sigma}_t(\bm{x})^\top\) and generator \(\mathcal{L}_t\varphi = \bm{b}_\theta^\top\nabla\varphi + \tfrac12\mathrm{Tr}(\mathbf{a}_t\nabla^2\varphi)\), the Doob \(h\)-transform adds the first-order term \((\mathbf{a}_t\nabla_{\bm{x}}\ell_t)^\top\nabla\varphi\). Thus the diffusion covariance is preserved and only the drift is shifted, by the covariance-preconditioned desirability gradient \(\Delta \bm{b}^*(\bm{x},t) = \mathbf{a}_t(\bm{x})\nabla_{\bm{x}}\ell_t(\bm{x})\).

\noindent\textbf{Practical discretization.}
Over a single denoising step of size \(\delta>0\) (from \(t\) to \(t-\delta\)), we use the local Gaussian approximation \(\bm{X}_{t-\delta}\mid \bm{X}_t{=}\bm{x} \approx \mathcal{N}(\bm{\mu}_t(\bm{x}),\bm{\Sigma}_t)\) with \(\bm{\mu}_t(\bm{x})=\bm{x}+\bm{b}_\theta(\bm{x},t)\delta\) and \(\bm{\Sigma}_t=\mathbf{a}_t(\bm{x})\delta\).
Applying the short-time Doob tilt and linearizing the log-desirability yields a Gaussian transition with unchanged covariance and shifted mean:
\begin{equation}
\bm{\mu}_t^{\ours}(\bm{x}) = \bm{\mu}_t(\bm{x}) + \Delta\bm{\mu}_t,
\qquad
\Delta\bm{\mu}_t = \bm{\Sigma}_t\nabla_{\bm{x}}\ell_t(\bm{x}).
\label{eq:local_mean_shift}
\end{equation}
Because the exact \(\ell_t\) integrates over all future continuations, we approximate it via the one-step denoised reward proxy and pull the gradient back through the denoising Jacobian \(\mathbf{J}_t=\partial\hat{\bm{x}}_0/\partial \bm{x}\), yielding \(\nabla_{\bm{x}}\ell_t(\bm{x}) \approx \lambda\,\mathbf{J}_t(\bm{x})^\top\nabla_{\hat{\bm{x}}_0}f(\hat{\bm{x}}_0,c)\).
Mapping this mean-space shift to the noise-prediction parameterization (\(\bm{\mu}_t = A_t \bm{x} - C_t\bm{\epsilon}\), \(C_t>0\)) and absorbing all positive sampler constants into a schedule \(\zeta(t)\) gives our main result:

\begin{proposition}[Variational Optimal Guidance]
\label{thm:vog}
Under the current-state conditional VOC of Eq.~\eqref{eq:conditional_voc_continuous}, the Doob \(h\)-transform preserves the diffusion covariance and shifts the drift by \(\mathbf{a}_t\nabla_{\bm{x}}\ell_t\). Under the short-time Gaussian approximation and denoised reward proxy above, the first-order local approximation to the optimal noise-space correction is
\begin{equation}
\bm{\epsilon}_{\mathrm{vog}}(\bm{x},t,c)
=
\bm{\epsilon}_c(\bm{x},t)
-
\zeta(t)\,
\mathbf{J}_t^\top
\nabla_{\hat{\bm{x}}_0}f(\hat{\bm{x}}_0,c),
\label{eq:final_form}
\end{equation}
where $\bm{\epsilon}_c$ is the conditional reference prediction, $\mathbf{J}_t=\partial\hat{\bm{x}}_0/\partial \bm{x}$, and $\zeta(t)>0$ is a sampler-dependent schedule. The derivation is given in \Cref{app:voc_to_jacobian_reg}. A separate propagation bound in \Cref{app:errorbound} assumes a stipulated local step error and does not bound the bias of the reward proxy or the factorized controller.
\end{proposition}

\noindent\textbf{Role of the denoising Jacobian.}
The key structural distinction from CFG is the Jacobian \(\mathbf{J}_t\): CFG adds the guidance contrast directly in noise space, whereas $\ours$ pulls it back through the denoiser via a vector-Jacobian product (VJP). Under the linear-subspace model below, the MMSE denoising Jacobian is a scaled orthogonal projector, so this VJP projects strictly onto $\mathcal{S}$; for nonlinear denoisers, the pullback is sensitivity-weighted rather than an exact tangent-space projection. This motivates, rather than universally guarantees, mitigation of the off-manifold drift (\(\mathbf{P}_\perp\bm{\Delta}_{\mathrm{eff}}\neq \bm{0}\)) analyzed in~\Cref{subsec:cfg_pitfalls}.

\subsection{Geometric Consistency Analysis}
To formalize the subspace-preserving property of $\ours$, we analyze the guidance correction relative to the data subspace $\mathcal{S}$ under~\Cref{assump:linear_subspace}.

\begin{theorem}[Manifold preservation of the correction]
\label{thm:geometric_invariance}
Let $\mathbf{P}_{\perp} = \mathbf{I}_D - \mathbf{U}\mathbf{U}^\top$ project onto the normal complement of $\mathcal{S}$, and let $\bm{g}_{\mathrm{vog}} \triangleq \mathbf{J}_t^\top \nabla_{\hat{\bm{x}}_0} f(\hat{\bm{x}}_0)$ denote the guidance correction that $\ours$ \emph{adds} to the conditional baseline. Under~\Cref{assump:linear_subspace} and the Bayes-optimal linear MMSE denoiser $\hat{\bm{x}}_0(\bm{x}_t)=\kappa_t \mathbf{U}\mathbf{U}^\top \bm{x}_t$, this correction lies strictly in the data subspace:
\begin{equation}
    \mathbf{P}_{\perp}\, \bm{g}_{\mathrm{vog}} = \bm{0}.
\end{equation}
\end{theorem}

\begin{proof}[Proof sketch]
For the linear subspace denoiser, $\mathbf{J}_t = \kappa_t \mathbf{U}\mathbf{U}^\top$ is symmetric with column and row spaces contained in $\mathcal{S}$. Hence $\bm{g}_{\mathrm{vog}} = \kappa_t \mathbf{U}\mathbf{U}^\top \nabla_{\hat{\bm{x}}_0}f \in \mathcal{S}$, so $\mathbf{P}_\perp \bm{g}_{\mathrm{vog}} = \bm{0}$. The full proof is given in~\Cref{app:voc_geometric_invariance}.
\end{proof}

\Cref{thm:geometric_invariance} constrains the added correction rather than the complete noisy trajectory. It assumes an exact linear MMSE denoiser and an exact VJP, and does not establish invariance for nonlinear denoisers or for the factorized surrogate used in practice (\Cref{rem:theorem_scope}).

\subsection{Practical Implementation Strategies}
\label{subsec:practical_tricks}

We combine three strategies to balance computational efficiency and numerical stability in practice.

\noindent \textbf{Strided gradient computation.}
Because adjacent denoising steps exhibit strong temporal correlation, we recompute the Jacobian pullback every $K_{\mathrm{VOG}}$ steps and reuse the cached correction on intermediate steps, amortizing backpropagation overhead. This inference stride $K_{\mathrm{VOG}}$ is distinct from DVOC's active-step count $K$.

\noindent \textbf{Factorized VJP surrogate.}
Centering the contrast at the differentiated coordinate splits the pullback exactly as $\mathbf{J}_{c,t}^\top\bm{g}_t=(\mathbf{J}_{c,t}^{\top}\mathbf{1})\odot\bm{g}_t+\bm{r}_t$, with $\mathbf{J}_{c,t}^{\top}\mathbf{1}=\nabla_{\bm{x}_{t}}(\mathbf{1}^{\top}\bm{\epsilon}_{c})$ and a residual $\bm{r}_t$ weighted by the variation of $\bm{g}_t$ across the denoiser's influence set. We drop $\bm{r}_t$. Accuracy requires only that the CFG contrast is locally smooth, not that the Jacobian is diagonal, permitting a dense Jacobian when $\bm{g}_t$ varies slowly over each influence set. Dropping $\bm{r}_t$ also acts as a support-preserving regularizer, as the surrogate vanishes wherever $\bm{g}_t$ vanishes and thus never injects guidance where the two predictions already agree. \Cref{appendix:voc_simplification} provides the identity, error bound, justification for differentiating the non-extrapolated conditional branch, and matched ablation.

\noindent \textbf{Momentum-enhanced guidance.}
High-frequency fluctuations in the raw guidance contrast $\bm{g}_t = \bm{\epsilon}_\theta(\bm{x}_t, c) - \bm{\epsilon}_\theta(\bm{x}_t, \emptyset)$ can trigger over-suppression under Jacobian regularization. We smooth it via exponential momentum, $\bm{m}_t = \beta \bm{m}_{t+1} + \bm{g}_t$, and substitute $\bm{m}_t$ for $\bm{g}_t$ in the update rule; here $\beta$ is the momentum coefficient, distinct from the schedule $\zeta(t)$ in~\Cref{eq:final_form}. The full inference algorithm is summarized in~\Cref{alg:vog_sampling}.

\section{Reward Alignment via Distilled VOC}
\label{sec:training}
DVOC realizes the parameter-space counterpart of the local variational control principle in Eq.~\eqref{eq:local_control_principle}: it constructs a bounded reward-directed target at the actual sampled endpoint and amortizes it into LoRA parameters $\phi$ through a sparse parameter pullback. ``Distilled'' refers to amortizing these locally optimized terminal targets directly into the generator without a VOG teacher. While VOC motivates improving a terminal outcome under a deviation budget, the Euclidean endpoint ball serves as a tractable local trust-region surrogate rather than a stochastic path-space KL constraint. The implementation-matched realization condition below makes this connection testable.

\paragraph{Projected endpoint target construction.}
Given a prompt $c$, we cache a no-gradient Euler rollout $\{\bm{x}_{t_i}\}_{i=0}^{N}$ over $t_0=T>\cdots>t_N=0$, recording velocities $\bar{\bm{v}}_i$ and step increments $h_i=\sigma_{i+1}-\sigma_i$. The sampled endpoint is $\bm{y}=\bm{x}_{t_N}=\bm{x}_0$, sharing VOG's reverse-time state indexing. Let $D$ denote the frozen VAE decoder and $R(D(\bm{y}),c)$ the differentiable training reward. We seek a local terminal improvement within a relative Euclidean trust region $\mathcal{B}(\bm{y},r)=\{\bm{u}:\|\bm{u}-\bm{y}\|_2\le r\}$ of radius $r=\rho\|\bm{y}\|_2$:
\begin{equation}
 \max_{\|\bm{u}-\bm{y}\|_2\le r} R(D(\bm{u}),c),
 \qquad r=\rho\|\bm{y}\|_2.
 \label{eq:dvoc_endpoint_objective}
\end{equation}
Starting from $\bm{y}^{(0)}=\bm{y}$, we approximate Eq.~\eqref{eq:dvoc_endpoint_objective} with $J$ normalized projected ascent steps:
\begin{equation}
\label{eq:target_noise}
\begin{aligned}
 \bm{g}^{(j)} &\coloneqq \mathbf{J}_D(\bm{y}^{(j)})^\top\nabla_{\bm{I}} R(\bm{I},c)\big|_{\bm{I}=D(\bm{y}^{(j)})},\\
 \bm{y}^{(j+1)} &\coloneqq \Pi_{\mathcal{B}(\bm{y},r)}\!\left(\bm{y}^{(j)}+
 \frac{\eta r}{J}\frac{\bm{g}^{(j)}}{\|\bm{g}^{(j)}\|_2+\varepsilon}\right),
\end{aligned}
\end{equation}
where $\mathbf{J}_D(\bm{u})=\partial D(\bm{u})/\partial \bm{u}$. Each $\bm{g}^{(j)}$ is thus an image-to-latent Jacobian pullback. We set the detached target $\bm{y}^+\coloneqq\sg(\bm{y}^{(J)})$ and displacement $\bm{\delta}\coloneqq \bm{y}^+-\bm{y}$. This finite, normalized, and projected displacement differs from a single raw reward gradient: projection bounds $\|\bm{\delta}\|_2\le r$ without guaranteeing monotonic reward improvement. Target construction differentiates only the reward and decoder, never the sampler.

\paragraph{Sparse surrogate backpropagation.}
To decouple trajectory generation from parameter differentiation, we sample an active set $\mathcal{A}$ of $K$ step indices uniformly without replacement and re-evaluate velocities with gradients only at those cached, detached states:
\begin{equation}
\label{eq:dvoc_connector}
\begin{aligned}
 \widetilde{\bm{v}}_i(\phi) &\coloneqq
 \begin{cases}\bm{v}_\phi(\sg(\bm{x}_{t_i}),t_i,c),&i\in\mathcal{A},\\
 \sg(\bar{\bm{v}}_i),&i\notin\mathcal{A},
 \end{cases}\\
 \widehat{\bm{y}}_\phi &\coloneqq \bm{x}_T+\sum_{i=0}^{N-1}h_i\widetilde{\bm{v}}_i(\phi),\\
 \bm{y}_{\mathrm{conn}} &\coloneqq \widehat{\bm{y}}_\phi+\sg(\bm{y}-\widehat{\bm{y}}_\phi).
\end{aligned}
\end{equation}
The forward-pinned connector $\bm{y}_{\mathrm{conn}}$ equals the sampled endpoint $\bm{y}$ in the forward pass, while its backward Jacobian (holding the cache and pinning offset fixed) is $\mathbf{J}_{\mathcal{A}}=\sum_{i\in\mathcal{A}}h_i\partial_\phi \bm{v}_\phi(\sg(\bm{x}_{t_i}),t_i,c)$. To transfer the target displacement $\bm{\delta}$ into $\phi$, we minimize the self-normalized endpoint regression objective
\begin{equation}
\label{eq:distillation_loss}
\begin{aligned}
 \mathcal{L}_{\mathrm{DVOC}}(\phi) &=
 \frac{\|\bm{y}_{\mathrm{conn}}-\bm{y}^+\|_2^2}{d\,w},\\
 w &\coloneqq\sg\!\left(\max\!\left\{\frac{\|\bm{\delta}\|_1}{d},10^{-5}\right\}\right),
\end{aligned}
\end{equation}
where $d$ is the number of latent coordinates. Consequently, the negative loss gradient realizes the parameter-space pullback of Eq.~\eqref{eq:local_control_principle}:
\begin{equation}
 \bm{u}_{\mathcal{A}}=-\nabla_\phi\mathcal{L}_{\mathrm{DVOC}}(\phi)
 =\frac{2}{d\,w}\mathbf{J}_{\mathcal{A}}^\top\bm{\delta}.
 \label{eq:dvoc_gradient}
\end{equation}
No derivative passes through PGD and no $N/K$ rescaling is applied. Dividing by $w$ normalizes the target scale, so the radius $r$ alone does not bound the parameter gradient.

\paragraph{Target realization condition.}
Let $F(\phi)=R(D(S_\phi(\bm{z},c)),c)$ for the complete sampler $S_\phi$, with $\bm{g}=\nabla_{\bm{y}} R$ and $\mathbf{J}_{\mathrm{full}}=\partial_\phi S_\phi$. The local reward change along the implemented direction is
\begin{equation}
 \left.\frac{dF(\phi+\alpha \bm{u}_{\mathcal{A}})}{d\alpha}\right|_{0}
 =\frac{2}{d\,w}\langle\mathbf{J}_{\mathrm{full}}^\top \bm{g},
 \mathbf{J}_{\mathcal{A}}^\top\bm{\delta}\rangle.
 \label{eq:dvoc_realization}
\end{equation}
Positive inner-product alignment yields a first-order reward increase. \Cref{app:dvoc_realization} decomposes the discrepancy into target, frozen-state, and subset errors; finite AdamW updates require separate empirical verification.

We use $N=9$, $K=3$, $J=2$, $\rho=0.1$, $\eta=1$, and zero ascent momentum, with $K\in\{1,2\}$ in ablations. At inference, the adapter uses the standard nine-step sampler without reward calls. The diagnostic in \Cref{app:dvoc_mechanism} verifies finite-step target transfer while showing that sparse directions need not align closely with the full-chain gradient.

\section{Experiments}
\label{sec:experiments}
We evaluate sampling-time VOG and reward-based DVOC in their respective settings: VOG retains the synthetic and class-conditional ImageNet experiments, while DVOC uses a matched post-training comparison on Z-Image-Turbo. Because the two realizations act on different variables and optimize distinct objectives, they do not form a single joint ablation.
\begin{table*}[!htbp]
\centering
\caption{\textbf{Quantitative evaluation of $\ours$ on class-conditional ImageNet.} Shaded rows apply $\ours$ as a plug-in correction atop the guidance baseline immediately above them across four backbone--resolution settings (50,000 samples per entry). We report FID ($\downarrow$), Precision ($\uparrow$), and Recall ($\uparrow$), with the best result in each column in \textbf{bold} and the second best \underline{underlined}. Dashes indicate unsupported baseline configurations.}
\label{tab:main_results}
\resizebox{\textwidth}{!}{%
\begin{tabular}{l |ccc |ccc |ccc |ccc}
\toprule
& \multicolumn{3}{c|}{EDM2-S (512 $\times$ 512)} & \multicolumn{3}{c|}{DiT-XL/2 (256 $\times$ 256)} & \multicolumn{3}{c|}{EDM2-L (64 $\times$ 64)} & \multicolumn{3}{c}{DiT-XL/2 (512 $\times$ 512)} \\
Guidance & FID & Precision & Recall & FID & Precision & Recall & FID & Precision & Recall & FID & Precision & Recall \\
\midrule
CFG \cite{ho2022classifierfreediffusionguidance}
    & 2.29 & 0.751 & 0.582 
    & 2.27 & 0.829 & 0.584 
    & 1.99 & 0.710 & \textbf{0.601} 
    & 3.04 & \textbf{0.838} & 0.541 \\ 
\rowcolor{lightCyan}
\quad + $\ours$ (Ours)
    & 2.03 & 0.755 & 0.591 
    & 2.04 & 0.756 & 0.607 
    & \textbf{1.42} & 0.714 & \underline{0.597} 
    & 2.89 & 0.745 & 0.560 \\
Interval \cite{kynkaanniemi2024applying}
    & \underline{1.87} & 0.760 & 0.598 
    & \underline{1.97} & 0.801 & \textbf{0.632} 
    & 1.64 & \textbf{0.730} & \underline{0.597} 
    & \underline{2.41} & \underline{0.792} & \textbf{0.611} \\ 
\rowcolor{lightCyan}
\quad  + $\ours$ (Ours)
    & \textbf{1.67} & 0.742 & 0.614 
    & \textbf{1.94} & 0.765 & 0.611 
    & \underline{1.52} & 0.720 & 0.593 
    & \textbf{2.39} & 0.760 & 0.567 \\ 
Cosine \cite{gao2023cosine}
    & 2.15 & \textbf{0.770} & \underline{0.619} 
    & 2.30 & \textbf{0.861} & 0.520 
    & 1.67 & 0.726 & 0.587 
    & 3.48 & 0.713 & \underline{0.586} \\ 
\rowcolor{lightCyan}
\quad  + $\ours$ (Ours)
    & 2.06 & 0.758 & 0.589 
    & 2.03 & 0.757 & 0.612 
    & 1.65 & \underline{0.727} & 0.588 
    & 2.83 & 0.747 & 0.563 \\
 Scheduler \cite{wang2024analysis}
    & 2.03 & \underline{0.762} & 0.591 
    & 2.14 & \underline{0.840} & 0.614 
    & - & - & - 
    & - & - & - \\
APG \cite{sadat2024eliminating}
    & 2.13 & 0.756 & \textbf{0.640} 
    & 2.11 & 0.815 & \underline{0.628} 
    & - & - & - 
    & - & - & - \\
ReCFG \cite{xia2025rectified} 
    & 2.23 & 0.761 & 0.608 
    & 2.13 & 0.799 & 0.601 
    & - & - & - 
    & - & - & - \\
\bottomrule
\end{tabular}%
}
\end{table*}
\vspace{-0.2cm}

\subsection{Evaluation Protocols}
\paragraph{Sampling-time guidance.}
We evaluate VOG on EDM2~\cite{Karras2024EDM2} and DiT~\cite{Peebles2023DiT} backbones on ImageNet~\cite{deng2009imagenet} at $64^2$, $256^2$, and $512^2$ resolutions. Baselines include standard CFG~\cite{ho2022classifierfreediffusionguidance}, Cosine~\cite{gao2023cosine} and Interval~\cite{kynkaanniemi2024applying} schedules, AutoGuidance~\cite{Karras2024autoguidance}, Scheduler~\cite{wang2024analysis}, APG~\cite{sadat2024eliminating}, and ReCFG~\cite{xia2025rectified}. We report FID~\cite{heusel2017gans} over 50,000 generated samples alongside Precision and Recall~\cite{sajjadi2018assessing} to separate sample fidelity from mode coverage; full architecture-specific configurations are in~\Cref{subsec:implementaion}.

\paragraph{Reward-based post-training.}
We initialize DVOC, ReFL~\cite{Xu2023ImageReward}, DRTune~\cite{wu2024drtune}, DiffusionOPSD~\cite{zhou2026diffusionopsd}, and DiffusionNFT~\cite{Zheng2025DiffusionNFT} from Z-Image-Turbo~\cite{ZImageTurbo}, and also evaluate its untrained original weights as the common reference. All trained methods share the training prompt stream, the weighted reward $0.03846\,R_{\mathrm{HPSv3++}}+R_{\mathrm{CLIP}}$, and a nominal rollout budget of 576 generated trajectories per optimizer update (48 prompt groups with 12 trajectories each). We evaluate their EMA adapters at \textbf{update 20}, which denotes 20 completed parameter updates rather than Euler denoising steps. Each run uses eight GPUs with colocated reward models, LoRA rank 32 ($\alpha=64$), and AdamW~\cite{loshchilov2019decoupled} with learning rate $3\times10^{-4}$ and weight decay $10^{-4}$ (DVOC, DiffusionOPSD, DiffusionNFT) or $10^{-2}$ (ReFL, DRTune). Complete method settings and measured per-update wall-clock times are reported in~\Cref{tab:dvoc_method_settings,tab:dvoc_training_time}.

The test set is the complete 1,000-prompt Pick-a-Pic split~\cite{Kirstain2023PickaPic}. For prompt index $i\in\{0,\ldots,999\}$, every method uses deterministic seed $20260923+i$, $1024\times1024$ resolution, nine deterministic Euler steps, and guidance scale 0, with verified identical initial latent hashes across all methods. All seven metrics score these exact same saved 8-bit PNG files. We classify HPSv3++~\cite{liu2026hpsv3scalingrewardmodels} (iteration condition 0.3) and CLIP~\cite{radford2021clip} as \emph{in-domain} evaluators because they directly enter the training reward, and treat Q-Align quality and aesthetics~\cite{QAlign2024}, AES~\cite{Schuhmann2022LaionAesthetics}, ImageReward~\cite{Xu2023ImageReward}, and DeQA~\cite{DeQA2025} as \emph{out-of-domain} (OOD) evaluators held out from optimization. Here OOD denotes held-out evaluators rather than a shifted prompt distribution or disjoint evaluator pretraining data. All main runs use one shared training seed.

\subsection{Sampling-Time Guidance Results}
\paragraph{Synthetic geometry.}
On the 2D branching-tree benchmark (\Cref{fig:toy_example}), increasing class overlap from two to three classes raises the total covariance trace from 0.228 to 0.403 and makes peripheral mode coverage substantially harder (\Cref{tab:synthetic_metrics}). Standard CFG ($\omega=4$) suppresses outliers but contracts branch support, reducing Recall to 0.876 and 0.779, respectively; VOG preserves Recall at 0.957 and 0.933 while achieving the highest Precision, F1-score, and lowest Wasserstein distance among guided samplers (\Cref{tab:toy-merged}).

\paragraph{ImageNet.}
\Cref{tab:main_results} evaluates VOG atop three guidance baselines across four model--resolution settings. VOG reduces FID in all twelve pairings, including from 1.99 to 1.42 on EDM2-L at $64^2$ and from 2.27 to 2.04 on DiT-XL/2 at $256^2$. Precision and Recall shift with the baseline operating point: on DiT-XL/2 with Cosine guidance, Recall rises from 0.520 to 0.612 while trading off Precision, yielding stronger overall distribution fidelity rather than uniform dominance on every metric.

\paragraph{Guidance scale and sampling budget.}
VOG remains stable over a wider guidance-scale range than CFG (\Cref{fig:ablation_omega}): on DiT-XL/2, its FID stays below 2.60 at $\omega=1.7$, where CFG exceeds 3.40 (\Cref{fig:ablation_edm2,fig:ablation_dit}). Gains also persist under reduced sampling steps (\Cref{tab:steps_25,tab:step_fid}), lowering FID from 7.58 to 4.44 at 25 steps and from 3.69 to 2.99 at 50 steps on DiT-XL/2, and from 6.20 to 3.89 at 25 steps on SiT-XL/2. The stride $K_{\mathrm{VOG}}$ and momentum $\beta$ sweep in \Cref{fig:ablation_momentum} confirms stable inference performance at moderate overhead (\Cref{tab:performance_comparison}).

\subsection{DVOC: Matched 20-Update Comparison}
\begin{table*}[t]
\centering
\caption{\textbf{Reward alignment after 20 optimization updates on Z-Image-Turbo.} The untrained base and five post-trained methods use the same 1,000 Pick-a-Pic prompts and initial noise ($1024^2$, nine-step Euler). HPSv3++ and CLIP are in-domain rewards; the other five are OOD evaluators. $\Delta$: paired difference vs.\ untrained base ($\uparrow$). Best/second-best: \textbf{bold}/\underline{underlined}.}
\label{tab:t2i_results}
\setlength{\tabcolsep}{2pt}
\renewcommand{\arraystretch}{1.12}
\resizebox{\textwidth}{!}{%
\begin{tabular}{@{}l cc cc cc cc cc cc cc@{}}
\toprule
\multirow{3}{*}{Method} & \multicolumn{4}{c}{\textbf{In-Domain}} & \multicolumn{10}{c}{\textbf{Out-of-Domain}} \\
\cmidrule(lr){2-5}\cmidrule(lr){6-15}
 & \multicolumn{2}{c}{HPSv3++} & \multicolumn{2}{c}{CLIP} & \multicolumn{2}{c}{Q-Align Q.} & \multicolumn{2}{c}{Q-Align A.} & \multicolumn{2}{c}{AES} & \multicolumn{2}{c}{ImageReward} & \multicolumn{2}{c}{DeQA} \\
\cmidrule(lr){2-3} \cmidrule(lr){4-5} \cmidrule(lr){6-7} \cmidrule(lr){8-9} \cmidrule(lr){10-11} \cmidrule(lr){12-13} \cmidrule(lr){14-15}
 & Score$\uparrow$ & \dt{$\Delta$} & Score$\uparrow$ & \dt{$\Delta$} & Score$\uparrow$ & \dt{$\Delta$} & Score$\uparrow$ & \dt{$\Delta$} & Score$\uparrow$ & \dt{$\Delta$} & Score$\uparrow$ & \dt{$\Delta$} & Score$\uparrow$ & \dt{$\Delta$} \\
\midrule
\multicolumn{15}{l}{\colorbox{groupShade}{\textbf{\textit{Z-Image-Turbo, 1024$^2$, nine-step Euler sampler}}}} \\
\midrule
Z-Image-Turbo (untrained) & 6.72 & \dt{---} & 0.2757 & \dt{---} & 4.825 & \dt{---} & 3.533 & \dt{---} & 5.412 & \dt{---} & 0.959 & \dt{---} & \best{4.439} & \dt{---} \\
ReFL~\cite{Xu2023ImageReward} & 12.22 & \dt{$+5.50$} & 0.2744 & \dt{$-0.0013$} & 4.844 & \dt{$+0.020$} & 4.223 & \dt{$+0.691$} & 5.955 & \dt{$+0.543$} & 1.159 & \dt{$+0.200$} & 4.133 & \dt{$-0.306$} \\
DRTune~\cite{wu2024drtune} & 12.70 & \dt{$+5.97$} & 0.2688 & \dt{$-0.0070$} & \second{4.932} & \dt{$+0.107$} & \second{4.489} & \dt{$+0.956$} & 6.095 & \dt{$+0.683$} & \second{1.225} & \dt{$+0.266$} & \second{4.262} & \dt{$-0.177$} \\
DiffusionNFT~\cite{Zheng2025DiffusionNFT} & 2.70 & \dt{$-4.03$} & \best{0.2794} & \dt{$+0.0036$} & 4.118 & \dt{$-0.706$} & 3.391 & \dt{$-0.142$} & 5.316 & \dt{$-0.097$} & 0.849 & \dt{$-0.110$} & 3.838 & \dt{$-0.601$} \\
DiffusionOPSD~\cite{zhou2026diffusionopsd} & \second{13.56} & \dt{$+6.83$} & \second{0.2782} & \dt{$+0.0025$} & 4.886 & \dt{$+0.061$} & 4.452 & \dt{$+0.920$} & \second{6.101} & \dt{$+0.689$} & 1.196 & \dt{$+0.237$} & 4.014 & \dt{$-0.425$} \\
\midrule
\rowcolor{lightCyan}
\textbf{DVOC} ($K=3$) & \best{13.79} & \dt{$+7.06$} & 0.2765 & \dt{$+0.0008$} & \best{4.943} & \dt{$+0.118$} & \best{4.618} & \dt{$+1.085$} & \best{6.288} & \dt{$+0.876$} & \best{1.279} & \dt{$+0.320$} & 4.237 & \dt{$-0.202$} \\
\bottomrule
\end{tabular}}
\end{table*}

\begin{figure*}[t]
\centering
\includegraphics[width=\textwidth]{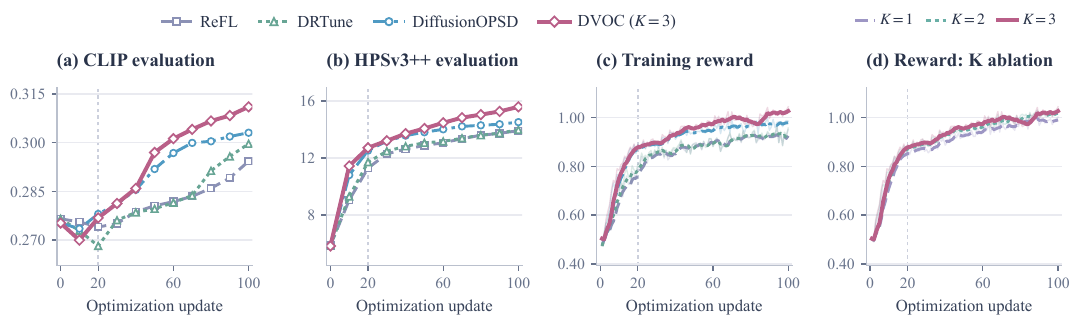}
\caption{\textbf{Evaluation and training dynamics.} \textbf{(a,b)} Online EMA evaluation every ten updates on CLIP and HPSv3++ (condition 0.0) under each trainer's noise protocol. \textbf{(c)} Logged training reward $0.03846\,R_{\mathrm{HPSv3++}}+R_{\mathrm{CLIP}}$: ReFL/DRTune score late-state proxies; DVOC/OPSD score sampled endpoints. \textbf{(d)} DVOC training reward for $K=1,2,3$. Training panels show raw means (faint) and trailing five-update averages; evaluation is unsmoothed. The vertical line marks update 20.}
\label{fig:dvoc_training_overview}
\end{figure*}
\paragraph{In-domain alignment and held-out scores.}
\Cref{tab:t2i_results} compares all five post-trained methods against the untrained reference, with \textbf{DVOC at $K=3$}. DVOC raises HPSv3++ from 6.72 to 13.79 ($+7.06$), achieving the highest trained mean. CLIP changes from 0.2757 to 0.2765 ($+0.0008$); its paired 95\% confidence interval includes zero, trailing DiffusionOPSD (0.2782) and DiffusionNFT (0.2794), whereas ReFL and DRTune decrease CLIP relative to the base.

On the held-out evaluators, DVOC leads all trained methods on Q-Align quality (4.943, $+0.118$), Q-Align aesthetics (4.618, $+1.085$), AES (6.288, $+0.876$), and ImageReward (1.279, $+0.320$), improving over the untrained base on all four. Conversely, DeQA drops from 4.439 to 4.237 ($-0.202$), with all five trained methods scoring below the base. Both the five positive gains (including HPSv3++) and the DeQA decline have paired intervals excluding zero (\Cref{tab:dvoc_paired}, unadjusted for multiple metrics).

\paragraph{Qualitative examples.}
\Cref{fig:dvoc_qualitative} shows the first two test prompts with matched seeds across the base and all five trained methods, visually complementing the full-set quantitative evaluation.

\paragraph{Optimization dynamics.}
\Cref{fig:dvoc_training_overview}(a,b) tracks the two in-domain evaluators through 100 updates: from update 20 to 100, DVOC's online CLIP rises from 0.2768 to 0.3111 and HPSv3++ from 12.71 to 15.59. Panels (c,d) show logged training rewards across methods and across $K\in\{1,2,3\}$. Because these online curves follow each trainer's internal noise protocol and HPSv3++ condition 0.0, they complement rather than replace the paired offline evaluation in \Cref{tab:t2i_results}; extended histories are in \Cref{app:training_curves}.

\subsection{DVOC Ablations and Computational Cost}
\begin{table*}[t]
\centering
\caption{\textbf{Active-step ablation.} Only $K$ changes; $J=2$ and $\rho=0.1$ are fixed. Means use the same 1,000 paired prompts at update 20. Best/second-best within this table: \textbf{bold}/\underline{underlined}.}
\label{tab:dvoc_ablation}
\small
\setlength{\tabcolsep}{5pt}
\renewcommand{\arraystretch}{1.12}
\begin{tabular}{@{}lrrrrrrr@{}}
\toprule
\multirow{2}{*}{DVOC variant} & \multicolumn{2}{c}{\textbf{In-Domain}} & \multicolumn{5}{c}{\textbf{Out-of-Domain}} \\
\cmidrule(lr){2-3}\cmidrule(lr){4-8}
 & HPSv3++$\uparrow$ & CLIP$\uparrow$ & Q-Align Q.$\uparrow$ & Q-Align A.$\uparrow$ & AES$\uparrow$ & ImageReward$\uparrow$ & DeQA$\uparrow$ \\
\midrule
$K=1$ & 13.36 & 0.2729 & \best{4.948} & \best{4.628} & 6.272 & 1.224 & \second{4.249} \\
$K=2$ & \second{13.69} & \second{0.2751} & \second{4.947} & 4.610 & \second{6.274} & \second{1.257} & \best{4.263} \\
\rowcolor{lightCyan}
$K=3$ (default) & \best{13.79} & \best{0.2765} & 4.943 & \second{4.618} & \best{6.288} & \best{1.279} & 4.237 \\
\bottomrule
\end{tabular}
\end{table*}

\paragraph{Active-step budget.}
\Cref{tab:dvoc_ablation} varies only the count $K\in\{1,2,3\}$ of cached velocities recomputed with gradients during backpropagation. $K=3$ leads on HPSv3++ (13.79), CLIP (0.2765), AES (6.288), and ImageReward (1.279); $K=1$ leads on Q-Align quality (4.948) and aesthetics (4.628), while $K=2$ leads on DeQA (4.263). Because the implemented loss applies no $N/K$ rescaling, increasing $K$ alters both the effective surrogate gradient and per-update compute cost without monotonically improving every held-out evaluator; we retain $K=3$ as our default throughout the main comparison.

\paragraph{Target construction.}
Three paired 20-update runs fix $K=3$ and isolate the target ascent count $J$ and anchor state (\Cref{app:dvoc_target_ablation,tab:dvoc_target_ablation,tab:dvoc_target_intervals}). Anchoring two-step ascent ($J=2$) at the actual endpoint improves Q-Align quality ($+0.0202$), AES ($+0.0692$), and DeQA ($+0.1203$) over a single ascent step ($J=1$), whereas $J=1$ attains higher CLIP ($+0.0021$) and takes 10.3\% less time per update (\Cref{tab:dvoc_target_cost}). Relative to a two-step current-policy late-state proxy ($\sigma_q\simeq 0.3050$), anchoring at the actual sampled endpoint improves Q-Align quality ($+0.0125$), Q-Align aesthetics ($+0.0215$), and DeQA ($+0.0100$) with paired intervals excluding zero, while intervals on the remaining four evaluators include zero. Direct probes before and after actual AdamW updates (\Cref{tab:dvoc_adamw_probes}) confirm finite target transfer while showing that individual parameter updates need not monotonically increase reward on every sample.

\paragraph{Computational cost.}
The nine-step forward rollout and two target-gradient evaluations are shared across the active-step variants $K\in\{1,2,3\}$. On one eight-GPU node, the observed mean update times are 325.8\,s ($K=1$), 391.2\,s ($K=2$), and 458.0\,s ($K=3$), compared to 273.3\,s for ReFL, 336.9\,s for DRTune, 396.3\,s for DiffusionOPSD, and 962.0\,s for DiffusionNFT (\Cref{tab:dvoc_training_time}). Relative to $K=3$, $K=1$ reduces per-update time by 28.9\% and $K=2$ by 14.6\%. At inference, all DVOC variants use the ordinary nine-step sampler without reward evaluations.

\section{Conclusion and Limitations}
\label{sec:conclusion_limitations}
We connected guidance and post-training through a shared Jacobian pullback interface. VOG derives a covariance-preconditioned state correction from conditional VOC and improves ImageNet fidelity and recall. DVOC constructs bounded projected ascent targets at the sampled endpoint and fits them into LoRA parameters through a sparse, forward-pinned Euler surrogate. After 20 updates on Z-Image-Turbo, DVOC improves HPSv3++ and four held-out scores over the untrained reference and leads trained comparators on these five means, while CLIP remains close to the base and DeQA declines across all trained methods.

The geometric subspace-preservation guarantee assumes an idealized linear-subspace MMSE denoiser and an exact VJP. DVOC's endpoint ball does not constrain path-space KL, and its local directional condition guarantees neither finite AdamW gains nor uniform OOD transfer. Varying $K$ alters both the gradient and compute cost. Post-training uses one backbone, one seed, and a fixed checkpoint; scorers may share training data, and broader claims require multiple seeds, human evaluation, and matched wall-clock budgets.
\FloatBarrier
\appendix
\setcounter{topnumber}{4}
\setcounter{bottomnumber}{3}
\setcounter{totalnumber}{6}
\renewcommand{\topfraction}{0.9}
\renewcommand{\bottomfraction}{0.8}
\renewcommand{\textfraction}{0.08}
\renewcommand{\floatpagefraction}{0.75}
\setlength{\textfloatsep}{12pt plus 2pt minus 2pt}
\setlength{\floatsep}{10pt plus 2pt minus 2pt}
\setlength{\intextsep}{12pt plus 2pt minus 2pt}
\setcounter{figure}{0}\setcounter{table}{0}\setcounter{algorithm}{0}
\renewcommand{\thefigure}{S\arabic{figure}}
\renewcommand{\thetable}{S\arabic{table}}
\renewcommand{\thealgorithm}{S\arabic{algorithm}}
\crefalias{section}{appendix}\crefalias{subsection}{appendix}
\section{Preliminaries and Notation}
\label{sec:preliminaries}
This appendix reviews the foundations of denoising diffusion probabilistic models and controllable generation through the lens of score estimation, and summarizes the notation used throughout the paper.
\subsection{Diffusion Models}
Denoising diffusion models~\cite{Ho2020DDPM,Song2021SDE} learn to reverse a forward diffusion process that gradually corrupts clean data $\bm{x}_0 \sim p_{\text{data}}(\bm{x})$ with Gaussian noise.
The forward marginal $q(\bm{x}_t|\bm{x}_0)$ at timestep $t \in [0, T]$ is parameterized as
\begin{equation}
    q(\bm{x}_t | \bm{x}_0) = \mathcal{N}(\bm{x}_t; \sqrt{\bar{\alpha}_t} \bm{x}_0, (1 - \bar{\alpha}_t) \mathbf{I}),
\end{equation}
where $\bar{\alpha}_t$ follows a prescribed variance schedule.
To generate samples, a denoising network $\bm{\epsilon}_\theta(\bm{x}_t, t)$ is trained to predict the injected noise $\bm{\epsilon} \sim \mathcal{N}(\bm{0}, \mathbf{I})$ by minimizing the reweighted mean squared error:
\begin{equation}
    \mathcal{L}_{\text{simple}} = \mathbb{E}_{\bm{x}_0, t, \bm{\epsilon}} \left[ \| \bm{\epsilon} - \bm{\epsilon}_\theta(\bm{x}_t, t) \|^2 \right].
\end{equation}
From a score-based perspective~\cite{Song2019Gradients,Song2021SDE}, the optimal noise predictor is proportional to the score of the noisy marginal distribution, $\bm{\epsilon}_\theta(\bm{x}_t, t) \approx -\sqrt{1 - \bar{\alpha}_t} \nabla_{\bm{x}_t} \log p_t(\bm{x}_t)$, enabling sample generation via reverse-time stochastic differential equations (SDEs) or the corresponding probability flow ODE.

\subsection{Guidance Mechanisms}
Controllable generation seeks to sample from a conditional distribution $p(\bm{x}|c)$ given a condition $c$ (e.g., a class label or text prompt) by steering the score estimate during inference.

\paragraph{Classifier Guidance (CG).}
Dhariwal and Nichol~\cite{Dhariwal2021DiffusionBeatGANs} steer sampling using an external noise-aware classifier $p_\phi(c|\bm{x}_t)$.
By Bayes' rule, the conditional score decomposes as
\begin{equation}
    \nabla_{\bm{x}_t} \log p_t(\bm{x}_t | c) = \nabla_{\bm{x}_t} \log p_t(\bm{x}_t) + \gamma \nabla_{\bm{x}_t} \log p_\phi(c | \bm{x}_t),
    \label{eq:cg_update}
\end{equation}
which is exact at $\gamma=1$ and amplifies conditioning when $\gamma>1$. In noise-prediction space, the guided prediction becomes $\hat{\bm{\epsilon}}_t = \bm{\epsilon}_\theta(\bm{x}_t) - \gamma \sqrt{1 - \bar{\alpha}_t} \nabla_{\bm{x}_t} \log p_\phi(c | \bm{x}_t)$. Although effective, CG requires training a dedicated classifier across noise levels.

\paragraph{Classifier-Free Guidance (CFG).}
To eliminate the external classifier, Ho and Salimans~\cite{ho2022classifierfreediffusionguidance} jointly train a single network for conditional prediction $\bm{\epsilon}_\theta(\bm{x}_t, c)$ and unconditional prediction $\bm{\epsilon}_\theta(\bm{x}_t, \emptyset)$.
At inference time, the guided noise prediction is formed via linear extrapolation:
\begin{equation}
    \tilde{\bm{\epsilon}}_t(\bm{x}_t, c) = \bm{\epsilon}_\theta(\bm{x}_t, \emptyset) + \omega \left( \bm{\epsilon}_\theta(\bm{x}_t, c) - \bm{\epsilon}_\theta(\bm{x}_t, \emptyset) \right),
    \label{eq:cfg_update}
\end{equation}
where $\omega$ is the guidance scale.
Eq.~\eqref{eq:cfg_update} implicitly ascends the log-density ratio $\log [ p_t(c|\bm{x}_t) / p(c) ]$ by using the conditional--unconditional contrast as a gradient direction.
However, large guidance scales $\omega$ often push trajectories off the data manifold, producing oversaturated artifacts or mode collapse~\cite{sadat2024eliminating,chung2024cfg++}.

\subsection{Summary of Notation}
\label{subsec:notation}
\Cref{tab:notation} summarizes the notation used throughout the paper. To avoid ambiguity between closely related hyperparameters: $\omega$ denotes the standard CFG guidance scale ($w=\omega-1$ is the conditional-centered control strength), $\zeta(t)$ is the step schedule of the $\ours$ noise-space correction, $\beta$ is the momentum coefficient smoothing the guidance contrast across steps, and $\beta_{\mathrm{vp}}(t)$ is the forward noise rate in our VP probability flow ODE analysis. Similarly, $K_{\mathrm{VOG}}$ denotes the inference-time interval between VJP recomputations in VOG, whereas $K$ denotes the number of active velocity recomputations per trajectory in DVOC.

\begin{table}[ht]
\centering
\small
\caption{\textbf{Summary of notation used in the main text and supplement.} Symbols are grouped by sampling-time state variables, VOG guidance parameters, DVOC post-training parameters, and theoretical control operators.}
\label{tab:notation}
\begin{tabular}{ll}
\toprule
Symbol & Meaning \\
\midrule
$\bm{x}_t$, $\hat{\bm{x}}_0$ & noisy state at time $t$, and one-step denoised clean prediction \\
$\bm{\epsilon}_c$, $\bm{\epsilon}_u$, $\bm{\epsilon}_{\mathrm{vog}}$ & conditional, unconditional, and $\ours$-corrected noise predictions \\
$\bar\alpha_t$, $\sigma_t$ & cumulative signal and noise coefficients of the forward process \\
$c$, $f(\bm{x}_0,c)$ & conditioning signal and terminal reward on the clean sample \\
$\omega$, $w=\omega-1$ & CFG guidance scale and its conditional-centered control strength \\
$\lambda$ & reward temperature in the conditional VOC objective \\
$\zeta(t)$ & positive step schedule of the $\ours$ noise-space correction \\
$\beta$ & momentum coefficient for smoothing the guidance contrast \\
$K_{\mathrm{VOG}}$ & inference-time stride between VJP recomputations in VOG \\
$K$, $J$, $\rho$ & DVOC active-step count, PGD ascent steps, and relative ball radius \\
$\mathbf{J}_t=\partial\hat{\bm{x}}_0/\partial \bm{x}_t$ & denoising Jacobian used in the state-space VJP $\mathbf{J}_t^\top\nabla f$ \\
$h_t(\bm{x})$, $\ell_t(\bm{x})=\log h_t(\bm{x})$ & desirability function and log-desirability function \\
$\bm{b}_\theta$, $\mathbf{a}_t=\bm{\sigma}_t\bm{\sigma}_t^\top$ & drift and diffusion covariance of the baseline sampler \\
$\mathcal{S}=\mathrm{span}(\mathbf{U})$, $\mathbf{P}_\parallel$, $\mathbf{P}_\perp$ & linear data subspace and orthogonal projectors onto $\mathcal{S}$ and $\mathcal{S}^\perp$ \\
\bottomrule
\end{tabular}
\end{table}

\section{From VOC to \texorpdfstring{$\ours$}{VOG} and DVOC}
\label{app:voc_to_jacobian_reg}

This appendix derives the general $\ours$ correction of~\Cref{thm:vog} from a continuous-time, current-state conditional variational optimal control problem. The CFG-specific controller we actually run follows in~\Cref{appendix:voc_simplification} and the DVOC endpoint construction and surrogate gradient in~\Cref{app:dvoc_teacher}.

We use the sampling-time parameterization $t\in[0,T]$, with $t=T$ the initial noise and $t=0$ the clean sample. At time $t$ the state $\bm{X}_t=\bm{x}$ is already realized, so guidance can modify only the future continuation from $t$ to $0$; it cannot resample the initial noise or reweight the past trajectory.

\subsection{Setup: the reference path measure}
\label{app:reference_measure}

\begin{definition}[Baseline policy and reference path measure]
\label{def:reference_measure}
Let $p_b(\,\cdot\mid \bm{x}_s,c)$ be a fixed reverse transition kernel, the \emph{baseline policy}. Conditioned on $\bm{X}_t=\bm{x}$, the induced future path measure is
\begin{equation}
P_{b,t:0}(\bm{x}_{0:t}\mid \bm{X}_t=\bm{x},c)
=
\prod_{s=1}^{t} p_b(\bm{x}_{s-1}\mid \bm{x}_s,c),
\label{app:reference_path_measure}
\end{equation}
with continuous-time analogue~\Cref{app:base_continuous_process}. We write $Q_{t:0}$ for the controlled measure that replaces it.
\end{definition}

For the stochastic VOG analysis, $p_b$ is the pretrained conditional reverse policy. Its zero-control limit is ordinary conditional sampling. DVOC uses a different practical regularization: a bounded displacement around an endpoint generated by the current policy. The following table distinguishes these settings; it does not identify a deterministic endpoint trust region with a stochastic path-space KL.
\begin{table}[ht]
\centering\small
\caption{Theoretical and practical instantiations of local control.}
\label{tab:reference_instantiations}
\begin{tabular}{@{}lp{0.30\linewidth}p{0.44\linewidth}@{}}
\toprule
Method & Reference or anchor & Objective / regularization \\
\midrule
VOG & Pretrained conditional process & Endogenous CFG log-ratio with path-space KL in the variational formulation \\
DVOC & Sampled endpoint of the current policy & HPSv3++ plus CLIP reward; relative latent-space ball for target construction \\
\bottomrule
\end{tabular}
\end{table}
The conditional reference separates the model's own conditioning from additional extrapolation. This interpretation applies to the VOG derivation below. DVOC uses no explicit CFG extrapolation in the reported Z-Image-Turbo experiments.

We write $f$ generically below and specialize only where the choice matters.

\subsection{The conditional VOC problem and its optimal solution}
\label{app:gibbs}

Let the baseline sampler induce the diffusion process
\begin{equation}
d\bm{X}_\tau
=
\bm{b}_\theta(\bm{X}_\tau,\tau)\,d\tau
+
\bm{\sigma}_\tau(\bm{X}_\tau)\,d\bm{W}_\tau,
\qquad
\tau\in[0,T],
\qquad
\mathbf{a}_\tau(\bm{x})
\triangleq
\bm{\sigma}_\tau(\bm{x})\bm{\sigma}_\tau(\bm{x})^\top .
\label{app:base_continuous_process}
\end{equation}
Given $\bm{X}_t=\bm{x}$, we solve the current-state conditional VOC problem
\begin{equation}
\max_{Q_{t:0}(\cdot\mid \bm{X}_t=\bm{x})}
\left\{
\mathbb{E}_{Q_{t:0}}
\left[
\lambda f(\bm{X}_0,c)
\right]
-
D_{\mathrm{KL}}
\left(
Q_{t:0}(\cdot\mid \bm{X}_t=\bm{x})
\,\|\,
P_{b,t:0}(\cdot\mid \bm{X}_t=\bm{x})
\right)
\right\},
\label{app:conditional_voc_continuous}
\end{equation}
which is conditional on the realized state and therefore does not optimize over the initial noise marginal. Define the desirability function and its logarithm,
\begin{equation}
h_t(\bm{x})
\triangleq
\mathbb{E}_{P_{b,t:0}(\cdot\mid \bm{X}_t=\bm{x})}
\left[
\exp(\lambda f(\bm{X}_0,c))
\right],
\qquad
\ell_t(\bm{x})
\triangleq
\log h_t(\bm{x}).
\label{app:continuous_desirability}
\end{equation}

\begin{lemma}[Gibbs variational principle]
\label{lem:gibbs}
Assume $h_t(\bm{x})<\infty$. Then Eq.~\eqref{app:conditional_voc_continuous} has the unique maximizer
\begin{equation}
\frac{dQ_{t:0}^{*}(\cdot\mid \bm{X}_t=\bm{x})}{dP_{b,t:0}(\cdot\mid \bm{X}_t=\bm{x})}
=
\frac{\exp(\lambda f(\bm{X}_0,c))}{h_t(\bm{x})},
\label{app:continuous_conditional_tilt}
\end{equation}
and the optimal value is $\ell_t(\bm{x})=\log h_t(\bm{x})$.
\end{lemma}

\begin{proof}
Write $P=P_{b,t:0}(\cdot\mid \bm{X}_t=\bm{x})$ and $Q=Q_{t:0}(\cdot\mid \bm{X}_t=\bm{x})$. Restricting to $Q\ll P$ and setting $\rho=dQ/dP$, so that $\rho\ge 0$ and $\int\rho\,dP=1$, the objective becomes
\begin{align}
\mathcal{J}(Q)
&=
\mathbb{E}_Q[\lambda f(\bm{X}_0,c)]
-
D_{\mathrm{KL}}(Q\|P)
\nonumber\\
&=
\int
\rho(\tau)
\left[
\lambda f(\bm{X}_0,c)-\log\rho(\tau)
\right]
P(d\tau).
\label{app:gibbs_step2}
\end{align}
Introducing a Lagrange multiplier $\nu$ for the normalization constraint, we define the Lagrangian
$\mathcal{L}(\rho,\nu)=\int\rho[\lambda f-\log\rho]\,dP+\nu(\int \rho\,dP-1)$. Its first variation is
\begin{align}
\frac{\delta\mathcal{L}}{\delta\rho(\tau)}
&=
\lambda f(\bm{X}_0,c)-\log\rho(\tau)-1+\nu
\nonumber\\
&\stackrel{!}{=}0,
\label{app:gibbs_step4}
\end{align}
so $\rho^*(\tau)=C\exp(\lambda f(\bm{X}_0,c))$ with $C=\exp(\nu-1)$. Normalization forces $C=1/h_t(\bm{x})$ by the definition in Eq.~\eqref{app:continuous_desirability}, which gives Eq.~\eqref{app:continuous_conditional_tilt}.

For uniqueness and optimality, let $Q^*$ be defined by Eq.~\eqref{app:continuous_conditional_tilt} and let $Q$ be any competitor. Then
\begin{align}
D_{\mathrm{KL}}(Q\|Q^*)
&=
\int
\left[
\log\frac{dQ}{dP}
-
\lambda f(\bm{X}_0,c)
+
\log h_t(\bm{x})
\right]dQ
\nonumber\\
&=
D_{\mathrm{KL}}(Q\|P)
-
\mathbb{E}_Q[\lambda f(\bm{X}_0,c)]
+
\log h_t(\bm{x}).
\label{app:gibbs_step6}
\end{align}
Rearranging, $\mathcal{J}(Q)=\log h_t(\bm{x})-D_{\mathrm{KL}}(Q\|Q^*)$. Since the KL divergence is non-negative and vanishes only at $Q=Q^*$, the maximum is attained uniquely at $Q^*$ with value $\log h_t(\bm{x})$.
\end{proof}

\subsection{Continuous Doob transform and the optimal drift}
\label{app:doob}

\Cref{lem:gibbs} characterizes the optimal path measure but not the dynamics that realize it. We recover those next.

Let $0\le r<t\le T$ and denote the baseline transition kernel by $p_b(t,\bm{x};r,d\bm{y})$. For any measurable $A$, Eq.~\eqref{app:continuous_conditional_tilt} gives
\begin{align}
Q^*(\bm{X}_r\in A\mid \bm{X}_t=\bm{x})
&=
\frac{
\mathbb{E}_{P_b}
\left[
\mathbf{1}_{\{\bm{X}_r\in A\}}
\exp(\lambda f(\bm{X}_0,c))
\;\middle|\; \bm{X}_t=\bm{x}
\right]
}{
h_t(\bm{x})
}
\nonumber\\
&=
\frac{1}{h_t(\bm{x})}
\int_A
h_r(\bm{y})\,
p_b(t,\bm{x};r,d\bm{y}),
\label{app:doob_step2}
\end{align}
so that the tilted kernel is the classical Doob $h$-transform,
\begin{equation}
Q^*(\bm{X}_r\in d\bm{y}\mid \bm{X}_t=\bm{x})
=
p_b(t,\bm{x};r,d\bm{y})\,
\frac{h_r(\bm{y})}{h_t(\bm{x})}.
\label{app:continuous_doob_transition}
\end{equation}

Since denoising runs from $t=T$ toward $t=0$, we define the generator along the denoising direction, $\mathcal{L}_t\varphi(\bm{x})=\lim_{\delta\downarrow 0}\delta^{-1}\{\mathbb{E}_{P_b}[\varphi(\bm{X}_{t-\delta})\mid \bm{X}_t=\bm{x}]-\varphi(\bm{x})\}$, which for the diffusion of Eq.~\eqref{app:base_continuous_process} takes the standard form
\begin{equation}
\mathcal{L}_t\varphi(\bm{x})
=
\bm{b}_\theta(\bm{x},t)^\top\nabla\varphi(\bm{x})
+
\tfrac12
\mathrm{Tr}
\left(
\mathbf{a}_t(\bm{x})\nabla^2\varphi(\bm{x})
\right).
\label{app:base_generator}
\end{equation}
The desirability satisfies the backward equation with the boundary condition imposed at the clean endpoint,
\begin{equation}
-\partial_t h_t(\bm{x})
+
\mathcal{L}_t h_t(\bm{x})
=
0,
\qquad
h_0(\bm{x})=\exp(\lambda f(\bm{x},c)),
\label{app:backward_kolmogorov}
\end{equation}
where the boundary sits at $t=0$ because that is the terminal (clean) state of the denoising process.

\begin{proposition}[Optimal drift correction]
\label{prop:doob_drift}
The process induced by $Q^*$ has the same diffusion covariance as the baseline and a drift shifted by the covariance-preconditioned desirability gradient:
\begin{equation}
\mathbf{a}_t^{h}(\bm{x})=\mathbf{a}_t(\bm{x}),
\qquad
\bm{b}^{*}(\bm{x},t)
=
\bm{b}_\theta(\bm{x},t)
+
\underbrace{\mathbf{a}_t(\bm{x})\nabla_{\bm{x}}\ell_t(\bm{x})}_{\Delta \bm{b}^{*}(\bm{x},t)} .
\label{app:continuous_optimal_drift}
\end{equation}
\end{proposition}

\begin{proof}
Let $P^h_{t,t-\delta}\varphi(\bm{x})\triangleq\mathbb{E}^{h}[\varphi(\bm{X}_{t-\delta})\mid \bm{X}_t=\bm{x}]$ be the transformed semigroup. By Eq.~\eqref{app:continuous_doob_transition},
\begin{equation}P^h_{t,t-\delta}\varphi(\bm{x})=h_t(\bm{x})^{-1}\mathbb{E}_{P_b}[h_{t-\delta}(\bm{X}_{t-\delta})\varphi(\bm{X}_{t-\delta})\mid \bm{X}_t=\bm{x}].\end{equation}
Differentiating this semigroup at $\delta=0$,
\begin{align}
\mathcal{L}_t^h\varphi(\bm{x})
&=
\lim_{\delta\downarrow 0}
\frac{P^h_{t,t-\delta}\varphi(\bm{x})-\varphi(\bm{x})}{\delta}
\nonumber\\
&=
\frac{-\partial_t(h_t\varphi)(\bm{x})+\mathcal{L}_t(h_t\varphi)(\bm{x})}{h_t(\bm{x})}
\nonumber\\
&=
\frac{
h_t\mathcal{L}_t\varphi
+
\varphi\left(-\partial_t h_t+\mathcal{L}_t h_t\right)
+
(\nabla h_t)^\top \mathbf{a}_t\nabla\varphi
}{h_t}
\nonumber\\
&=
\mathcal{L}_t\varphi(\bm{x})
+
\left(\mathbf{a}_t\nabla_{\bm{x}}\log h_t\right)^\top\nabla\varphi(\bm{x})
\nonumber\\
&=
\left[
\bm{b}_\theta+\mathbf{a}_t\nabla_{\bm{x}}\ell_t
\right]^\top\nabla\varphi
+
\tfrac12\mathrm{Tr}\!\left(\mathbf{a}_t\nabla^2\varphi\right).
\label{app:doobgen_step5}
\end{align}
For an It\^o diffusion the drift is the coefficient of $\nabla\varphi$ and the diffusion covariance the coefficient of $\nabla^2\varphi$. Matching these against Eq.~\eqref{app:doobgen_step5} gives Eq.~\eqref{app:continuous_optimal_drift}.
\end{proof}

\subsection{Quadratic control form of the KL}
\label{app:quadratic}

The same correction can be read as a control energy, which is the form that makes the trust-region reading of~\Cref{def:reference_measure} concrete. Consider the control-affine SDE
\begin{equation}
d\bm{X}_\tau^u
=
\left[
\bm{b}_\theta(\bm{X}_\tau^u,\tau)
+
\bm{\sigma}_\tau(\bm{X}_\tau^u)\bm{u}(\bm{X}_\tau^u,\tau)
\right]d\tau
+
\bm{\sigma}_\tau(\bm{X}_\tau^u)d\bm{W}_\tau .
\label{app:controlled_sde}
\end{equation}
Under standard Girsanov conditions the conditional path KL is the quadratic control energy
\begin{equation}
D_{\mathrm{KL}}
\left(
Q^u_{t:0}
\,\|\,
P_{b,t:0}
\right)
=
\tfrac12
\mathbb{E}_{Q^u}
\left[
\int_0^t
\|\bm{u}(\bm{X}_\tau,\tau)\|_2^2\,d\tau
\;\middle|\;
\bm{X}_t=\bm{x}
\right],
\label{app:girsanov_quadratic}
\end{equation}
and, writing the drift correction as $\Delta \bm{b}(\bm{x},\tau)=\bm{\sigma}_\tau(\bm{x})\bm{u}(\bm{x},\tau)$ with $\mathbf{a}_\tau$ invertible, equivalently
\begin{equation}
\tfrac12
\mathbb{E}_{Q}
\left[
\int_0^t
\Delta \bm{b}(\bm{X}_\tau,\tau)^\top
\mathbf{a}_\tau(\bm{X}_\tau)^{-1}
\Delta \bm{b}(\bm{X}_\tau,\tau)\,d\tau
\;\middle|\;
\bm{X}_t=\bm{x}
\right].
\label{app:quadratic_cost}
\end{equation}
VOC therefore penalizes deviations from the baseline dynamics in a covariance-weighted quadratic norm.

\subsection{Short-time discretization and the local mean shift}
\label{app:shorttime}

Eq.~\eqref{app:continuous_optimal_drift} is exact at the continuous Doob level, but $\ell_t$ is intractable. Before approximating it we express the drift correction as a mean shift of a single sampler step.

For a small denoising step $\delta>0$, the baseline transition admits the Euler--Maruyama approximation
\begin{equation}
\bm{X}_{t-\delta}\mid \bm{X}_t=\bm{x}
\approx
\mathcal{N}
\left(
\bm{\mu}_t(\bm{x}),\,
\bm{\Sigma}_t
\right),
\qquad
\bm{\mu}_t(\bm{x})=\bm{x}+\bm{b}_\theta(\bm{x},t)\delta,
\qquad
\bm{\Sigma}_t=\mathbf{a}_t(\bm{x})\delta,
\label{app:euler_gaussian}
\end{equation}
with density $p_b(\bm{y}\mid \bm{x})$ from $\bm{X}_t=\bm{x}$ to $\bm{X}_{t-\delta}=\bm{y}$. Applying Eq.~\eqref{app:continuous_doob_transition} at the density level,
\begin{equation}
q^*(\bm{y}\mid \bm{x})
=
p_b(\bm{y}\mid \bm{x})\,
\frac{h_{t-\delta}(\bm{y})}{h_t(\bm{x})}
\propto
p_b(\bm{y}\mid \bm{x})\,h_{t-\delta}(\bm{y}),
\label{app:short_time_doob_density}
\end{equation}
where the proportionality holds because $h_t(\bm{x})$ does not depend on the candidate next state $\bm{y}$.

\begin{proposition}[Short-time tilt preserves Gaussianity]
\label{prop:local_mean_shift}
Under Eq.~\eqref{app:euler_gaussian} and a first-order expansion of $\ell$ around $(\bm{x},t)$, the tilted transition is Gaussian with unchanged covariance and mean shifted by $\bm{\Sigma}_t \bm{g}_t$, where $\bm{g}_t\triangleq\nabla_{\bm{x}}\ell_t(\bm{x})$:
\begin{equation}
\bm{\mu}_t^{\ours}(\bm{x})
=
\bm{\mu}_t(\bm{x})
+
\Delta\bm{\mu}_t,
\qquad
\Delta\bm{\mu}_t
=
\bm{\Sigma}_t\,\nabla_{\bm{x}}\ell_t(\bm{x}).
\label{app:local_mean_shift}
\end{equation}
\end{proposition}

\begin{proof}
Expanding the log-desirability about $(\bm{x},t)$ and substituting $\bm{y}-\bm{x}=(\bm{y}-\bm{\mu}_t(\bm{x}))+\bm{b}_\theta(\bm{x},t)\delta$,
\begin{align}
\ell_{t-\delta}(\bm{y})
&=
\ell_t(\bm{x})
-
\delta\,\partial_t\ell_t(\bm{x})
+
\nabla_{\bm{x}}\ell_t(\bm{x})^\top (\bm{y}-\bm{x})
+
O(\|\bm{y}-\bm{x}\|^2+\delta^2)
\label{app:taylor_step1}\\
&=
\mathrm{const}_t
+
\bm{g}_t^\top(\bm{y}-\bm{\mu}_t(\bm{x}))
+
O(\|\bm{y}-\bm{x}\|^2+\delta^2),
\label{app:taylor_step2}
\end{align}
where $\mathrm{const}_t$ is absorbed into the normalization of the tilted kernel. Substituting the leading $\bm{y}$-dependent term into Eq.~\eqref{app:short_time_doob_density},
\begin{align}
q^*(\bm{y}\mid \bm{x})
&\propto
\mathcal{N}(\bm{y};\bm{\mu}_t(\bm{x}),\bm{\Sigma}_t)\,
\exp\!\left(\bm{g}_t^\top(\bm{y}-\bm{\mu}_t(\bm{x}))\right)
\nonumber\\
&=
\mathcal{N}
\left(
\bm{y};\,
\bm{\mu}_t(\bm{x})+\bm{\Sigma}_t \bm{g}_t,\,
\bm{\Sigma}_t
\right),
\label{app:tilt_step2}
\end{align}
which is Eq.~\eqref{app:local_mean_shift}.
\end{proof}

\subsection{Accumulated error of the local expansion}
\label{app:errorbound}

\Cref{prop:local_mean_shift} is a per-step statement. The following bounds propagation of a stipulated one-step error; it is separate from the bias introduced by the reward proxy and factorized controller.

\begin{proposition}[First-order global error]
\label{prop:global_error}
Consider a grid $0=t_0<\cdots<t_N=T$ with $\Delta t_k\le\Delta$. Let $F_k$ be a target one-step mean map and $\widetilde{F}_k$ its approximation. Assume $F_k$ is $L_k$-Lipschitz with $L_k\le1+K\Delta t_k$, and explicitly assume $\sup_{\bm{x}}\|F_k(\bm{x})-\widetilde{F}_k(\bm{x})\|_2\le M\Delta t_k^2$. Composing these maps from the same initial state gives
\begin{equation}
\|\bm{x}_N^*-\widetilde{\bm{x}}_N\|_2
\le
M\,T\exp(KT)\,\Delta
=
O(\Delta),
\label{app:first_order_global_error}
\end{equation}
Here $M$ bounds the assumed local error. These compositions are not generally the expectations of nonlinear stochastic trajectories. In particular, the hypothesis is not established for the denoised-reward proxy or the cached, factorized implementation.
\end{proposition}

\begin{proof}
Write $r_k\triangleq\sup_{\bm{x}}\|F_k(\bm{x})-\widetilde{F}_k(\bm{x})\|_2$ for the one-step error. Splitting the propagated error and using the Lipschitz property,
\begin{align}
\|\bm{x}_{k+1}^*-\widetilde{\bm{x}}_{k+1}\|_2
&=
\|F_k(\bm{x}_k^*)-\widetilde{F}_k(\widetilde{\bm{x}}_k)\|_2
\nonumber\\
&\le
\|F_k(\bm{x}_k^*)-F_k(\widetilde{\bm{x}}_k)\|_2
+
\|F_k(\widetilde{\bm{x}}_k)-\widetilde{F}_k(\widetilde{\bm{x}}_k)\|_2
\nonumber\\
&\le
L_k\|\bm{x}_k^*-\widetilde{\bm{x}}_k\|_2+r_k .
\label{app:one_step_error_recursion}
\end{align}
Unrolling the recursion, and using $\bm{x}_0^*=\widetilde{\bm{x}}_0$ because both processes start from the same realized state in the conditional formulation,
\begin{equation}
\|\bm{x}_N^*-\widetilde{\bm{x}}_N\|_2
\le
\sum_{k=0}^{N-1}
\Bigl(\textstyle\prod_{j=k+1}^{N-1}L_j\Bigr)
r_k .
\label{app:accumulated_error_bound}
\end{equation}
Using the assumed $r_k\le M\Delta t_k^2$ and $L_k\le1+K\Delta t_k$ in Eq.~\eqref{app:accumulated_error_bound} gives
\begin{equation}
\|\bm{x}_N^*-\widetilde{\bm{x}}_N\|_2
\le
M\exp(KT)
\sum_{k=0}^{N-1}\Delta t_k^2
\le
M\,T\exp(KT)\,\Delta .
\label{app:accumulated_taylor_error_bound}
\end{equation}
\end{proof}

\subsection{Proxy desirability and the Jacobian pullback}
\label{app:pullback}

The exact desirability $\ell_t(\bm{x})=\log\mathbb{E}_{P_{b,t:0}}[\exp(\lambda f(\bm{X}_0,c))\mid \bm{X}_t=\bm{x}]$ integrates over all future continuations and is not computable. $\ours$ replaces it with the current denoised prediction,
\begin{equation}
\ell_t(\bm{x})
\approx
\lambda f(\hat{\bm{x}}_0(\bm{x}),c),
\qquad
\hat{\bm{x}}_0(\bm{x})
=
\frac{\bm{x}-\sqrt{1-\bar\alpha_t}\,\bm{\epsilon}_c(\bm{x},t)}{\sqrt{\bar\alpha_t}},
\label{app:proxy_desirability}
\end{equation}
built from the reference policy $\bm{\epsilon}_c$ of~\Cref{def:reference_measure}, so that with $\mathbf{J}_t(\bm{x})=\partial\hat{\bm{x}}_0(\bm{x})/\partial \bm{x}$ the chain rule gives the pullback
\begin{equation}
\nabla_{\bm{x}}\ell_t(\bm{x})
\approx
\lambda\,
\mathbf{J}_t(\bm{x})^\top
\nabla_{\hat{\bm{x}}_0}f(\hat{\bm{x}}_0,c).
\label{app:jacobian_pullback}
\end{equation}

\begin{remark}[Relation to training-free guidance for external objectives]
\label{rem:dps_relation}
The clean-prediction chain rule is also used by DPS, FreeDoM, Universal Guidance, h-Edit, and related Doob-transform methods~\cite{chung2023diffusion,yu2023freedom,bansal2024universal,nguyen2025hedit,denker2024deft}. Our VOG setting uses an endogenous conditional--unconditional contrast. We approximate the clean-space log-ratio gradient with the score contrast at the current noise level, then apply a conditional-branch pullback and its factorized approximation. These substitutions define the practical controller; they do not make the chain rule itself novel or establish equivalence to an exact terminal reward tilt.
\end{remark}

\Cref{tab:pullback_comparison} summarizes the available objective and gradient information in these constructions.

\begin{table}[ht]
\centering
\small
\caption{Training-free corrections that share the clean-prediction pullback of Eq.~\eqref{app:jacobian_pullback}. Here $\mathbf{A}_t\triangleq\partial\hat{\bm{x}}_0/\partial \bm{x}_t$, $\bm{g}_t=\bm{\epsilon}_c-\bm{\epsilon}_u$, $\rho_t=\sqrt{1-\bar\alpha_t}$, and $w=\omega-1$ is the control strength relative to the conditional reference. ``Aux.'' marks methods requiring an externally supplied object beyond the pretrained denoiser: a measurement, an auxiliary or trained network, or a queryable reward. The last two rows are constructions in our setting. The second-to-last is the ablated alternative of~\Cref{app:direct_cfg}, whose Jacobian is taken through the extrapolated predictor and therefore retains an $O(\omega^2)$ self-coupling; on DiT-XL/2 ($256\times256$) it reaches FID $5.48$ against $2.04$ for the controller we run.}
\label{tab:pullback_comparison}
\resizebox{\textwidth}{!}{%
\begin{tabular}{@{}llccl@{}}
\toprule
Method & Guidance source / task & Aux. & Access to $\nabla_{\hat{\bm{x}}_0}f$ & Final correction $\bm{\Delta}_t$ \\
\midrule
DPS~\cite{chung2023diffusion} & Inverse-problem likelihood & \checkmark & Direct autodiff & $\gamma_t \mathbf{A}_t^\top\nabla_{\hat{\bm{x}}_0}\log p(\bm{y}\mid\hat{\bm{x}}_0)$ \\
FreeDoM~\cite{yu2023freedom} & External energy & \checkmark & Direct autodiff & $-\gamma_t \mathbf{A}_t^\top\nabla_{\hat{\bm{x}}_0}E(\hat{\bm{x}}_0,c)$ \\
Universal Guidance~\cite{bansal2024universal} & Arbitrary external objective & \checkmark & Direct autodiff & $\gamma_t \mathbf{A}_t^\top\nabla_{\hat{\bm{x}}_0}R(\hat{\bm{x}}_0,c)$ \\
h-Edit~\cite{nguyen2025hedit} & Source-image editing bridge & \checkmark & Direct autodiff & $\gamma_t\nabla_{\bm{x}_t}\log h_t^{\mathrm{edit}}(\bm{x}_t)$ \\
DEFT~\cite{denker2024deft} & Learned Doob correction & \checkmark & Learned, never computed & $\gamma_t\bm{\delta}_\phi(\bm{x}_t,t,c)$ \\
DOIT~\cite{zhu2026doit} & Non-differentiable reward & \checkmark & MC over $M$ reward queries & $\gamma_t\widehat{\nabla_{\bm{x}_t}\log h_t}^{\,\mathrm{MC}}$ \\
\midrule
Direct CFG pullback (ablated) & Conditional generation & $\times$ & Unavailable; use $\bm{g}_t$ & $w\,\mathbf{A}_{\mathrm{CFG},t}^\top\bm{g}_t$ \\
\rowcolor{lightCyan}
$\ours$-Inference (ours) & Conditional generation & $\times$ & Unavailable; use $\bm{g}_t$ & $w\left[\mathbf{1}-\rho_t\nabla_{\bm{x}_t}(\mathbf{1}^\top\bm{\epsilon}_c)\right]\odot\bm{g}_t$ \\
\bottomrule
\end{tabular}}
\end{table}

DOIT~\cite{zhu2026doit} estimates corrections from pointwise queries to a potentially non-differentiable reward. VOG instead uses the conditional--unconditional score contrast already exposed by the pretrained network, without evaluating a terminal density ratio. DEFT learns a correction network~\cite{denker2024deft}; our separate DVOC experiment fits a detached projected endpoint target through a sparse surrogate. These methods share a control interpretation but use different approximations and information.

Substituting Eq.~\eqref{app:jacobian_pullback} into Eq.~\eqref{app:local_mean_shift} gives the mean-space correction $\Delta\bm{\mu}_t\propto\bm{\Sigma}_t\mathbf{J}_t^\top\nabla_{\hat{\bm{x}}_0}f$. It remains to map it into noise-prediction space. Under the DDPM parameterization the reverse mean is affine in the noise prediction,
\begin{equation}
\bm{\mu}_t(\bm{x},\bm{\epsilon})
=
A_t \bm{x}
-
C_t \bm{\epsilon},
\qquad
A_t
=
\frac{1}{\sqrt{\alpha_t}},
\qquad
C_t
=
\frac{1-\alpha_t}{\sqrt{\alpha_t}\sqrt{1-\bar\alpha_t}}
>0,
\label{app:mean_eps_relation}
\end{equation}
so writing $\bm{\epsilon}_{\mathrm{vog}}=\bm{\epsilon}_c+\Delta\bm{\epsilon}_t$ gives $\Delta\bm{\mu}_t=-C_t\Delta\bm{\epsilon}_t$. Because $C_t>0$, a positive correction in mean space is a \emph{negative} correction in noise space:
\begin{align}
\Delta\bm{\epsilon}_t
&=
-C_t^{-1}\Delta\bm{\mu}_t
\nonumber\\
&\propto
-\,\mathbf{J}_t^\top\nabla_{\hat{\bm{x}}_0}f(\hat{\bm{x}}_0,c).
\label{app:sign_conclusion}
\end{align}
Absorbing $C_t^{-1}$, $\bm{\Sigma}_t$, the reward scale $\lambda$ and the remaining sampler constants into a positive schedule $\zeta(t)$, we obtain the general form quoted in~\Cref{thm:vog}:
\begin{equation}
\boxed{\;
\bm{\epsilon}_{\mathrm{vog}}(\bm{x},t,c)
=
\bm{\epsilon}_c(\bm{x},t)
-
\zeta(t)\,
\mathbf{J}_t^\top
\nabla_{\hat{\bm{x}}_0} f(\hat{\bm{x}}_0,c).
\;}
\label{app:vog_noise_space}
\end{equation}

\section{The CFG-Specific \texorpdfstring{$\ours$}{VOG} Controller}
\label{appendix:voc_simplification}

Eq.~\eqref{app:vog_noise_space} is stated for a generic terminal reward $f$. Here we instantiate it for the endogenous CFG objective of~\Cref{eq:cfg_terminal_reward} and arrive at the controller we actually run. Two choices in this instantiation are specific to $\ours$: the pullback is taken through the \emph{non-extrapolated conditional} branch (\Cref{app:direct_cfg}), and the vector-Jacobian product is replaced by a \emph{factorized surrogate} (\Cref{app:factorized}).

Write the two network evaluations at the current state as
\begin{equation}
\bm{\epsilon}_c \triangleq \bm{\epsilon}_\theta(\bm{x}_t,t,c),
\qquad
\bm{\epsilon}_u \triangleq \bm{\epsilon}_\theta(\bm{x}_t,t,\varnothing),
\qquad
\bm{g}_t \triangleq \bm{\epsilon}_c - \bm{\epsilon}_u,
\label{app:cfg_contrast_def}
\end{equation}
so that standard CFG is $\bm{\epsilon}_{\mathrm{cfg}}=\bm{\epsilon}_u+\omega\bm{g}_t=\bm{\epsilon}_c+w\bm{g}_t$ with $w\triangleq\omega-1$. Since the reference is $\bm{\epsilon}_c$ (\Cref{def:reference_measure}), it is $w$ rather than $\omega$ that measures the control: the reward of~\Cref{eq:cfg_terminal_reward} enters with strength $w$, and $w=0$ leaves the reference untouched.

\subsection{From the CFG terminal reward to the score contrast}
\label{app:score_contrast}

The clean-space gradient of Eq.~\eqref{eq:cfg_terminal_reward} is a difference of two log-densities that we cannot evaluate at $\hat{\bm{x}}_0$. The pretrained network does, however, give us exactly this difference at the current noise level. Reading the steps in order,
\begin{align}
\nabla_{\hat{\bm{x}}_0} f_{\mathrm{cfg}}(\hat{\bm{x}}_0,c)
&=
\omega\left[
\nabla_{\hat{\bm{x}}_0}\log p_\theta(\hat{\bm{x}}_0\mid c)
-
\nabla_{\hat{\bm{x}}_0}\log p_\theta(\hat{\bm{x}}_0)
\right]
\label{app:sc_step1}\\
\nabla_{\bm{x}_t}\log\frac{p_{\theta,t}(\bm{x}_t\mid c)}{p_{\theta,t}(\bm{x}_t)}
&=
\bm{s}_c(\bm{x}_t,t)-\bm{s}_u(\bm{x}_t,t)
\nonumber\\
&=
-\frac{\bm{\epsilon}_c-\bm{\epsilon}_u}{\sigma_t}
=
-\frac{\bm{g}_t}{\sigma_t},
\label{app:sc_step3}
\end{align}
so the CFG contrast $\bm{g}_t$ \emph{is} the conditional-versus-unconditional score difference up to the factor $-1/\sigma_t$.

\begin{approximation}[Gradient replacement]
\label{approx:gradient_replacement}
We use the time-$t$ score contrast of Eq.~\eqref{app:sc_step3} as a proxy for the terminal clean-space reward gradient of Eq.~\eqref{app:sc_step1}, absorbing the sign and the factor $1/\sigma_t$ into the control schedule. This is exact when the conditional and unconditional log-densities differ by a function that the forward kernel transports without distortion, and approximate otherwise; the same substitution is standard in classifier-free and training-free guidance.
\end{approximation}

\subsection{Pullback through the conditional clean prediction}
\label{app:conditional_pullback}

With the reward gradient replaced by $\bm{g}_t$, we pull it back through the clean-prediction map. Write the conditional clean prediction and its Jacobian as
\begin{equation}
\hat{\bm{x}}_{0,c}
=
a_t \bm{x}_t - b_t\,\bm{\epsilon}_c(\bm{x}_t,t),
\qquad
a_t=\frac{1}{\sqrt{\bar\alpha_t}},
\qquad
b_t=\frac{\sqrt{1-\bar\alpha_t}}{\sqrt{\bar\alpha_t}},
\qquad
\mathbf{J}_{c,t}=\frac{\partial\bm{\epsilon}_c}{\partial \bm{x}_t}.
\label{app:conditional_clean_pred}
\end{equation}
Then the pullback of the contrast is
\begin{align}
\mathbf{A}_{c,t}
&\triangleq
\frac{\partial\hat{\bm{x}}_{0,c}}{\partial \bm{x}_t}
=
a_t\mathbf{I} - b_t\mathbf{J}_{c,t}
\nonumber\\
\mathbf{A}_{c,t}^\top\bm{g}_t
&=
a_t\bm{g}_t - b_t\mathbf{J}_{c,t}^\top\bm{g}_t
\nonumber\\
&=
a_t
\left[
\bm{g}_t - \rho_t\,\mathbf{J}_{c,t}^\top\bm{g}_t
\right],
\qquad
\rho_t\triangleq\frac{b_t}{a_t}=\sqrt{1-\bar\alpha_t}.
\label{app:pb_step3}
\end{align}
The scalar $a_t$ is absorbed into the guidance schedule, leaving the bracket as the object to compute.

\subsection{Factorized VJP surrogate}
\label{app:factorized}

We now replace $\mathbf{J}_{c,t}^\top\bm{g}_t$ by an element-wise product. We stress at the outset that this is \emph{not} done to avoid an intractable computation: $\mathbf{J}_{c,t}^\top\bm{g}_t$ is one reverse-mode VJP with $\bm{g}_t$ as the upstream vector, costing a single extra backward pass and never forming the dense Jacobian. The substitution is instead justified by a property of the vector being pulled back, and the resulting condition is weaker than the one usually invoked for such element-wise forms.

The starting point is an exact identity, obtained by centering the contrast at the coordinate being differentiated.

\begin{lemma}[Exact centered decomposition]
\label{lem:centered_decomposition}
For any $\mathbf{J}_{c,t}\in\mathbb{R}^{D\times D}$ and any $\bm{g}_t\in\mathbb{R}^{D}$,
\begin{equation}
\mathbf{J}_{c,t}^\top\bm{g}_t
=
\left(\mathbf{J}_{c,t}^\top\mathbf{1}\right)\odot\bm{g}_t
\;+\;
\bm{r}_t,
\qquad
r_{t,i}
=
\sum_{j=1}^{D}
\frac{\partial \epsilon_{c,j}}{\partial x_{t,i}}
\left(g_{t,j}-g_{t,i}\right).
\label{app:centered_identity}
\end{equation}
\end{lemma}

\begin{proof}
Writing out the $i$-th coordinate and adding and subtracting $g_{t,i}$ inside the sum,
\begin{align}
\left[\mathbf{J}_{c,t}^\top\bm{g}_t\right]_i
&=
\sum_{j=1}^{D}
\frac{\partial \epsilon_{c,j}}{\partial x_{t,i}}\,g_{t,j}
\nonumber\\
&=
g_{t,i}\sum_{j=1}^{D}\frac{\partial \epsilon_{c,j}}{\partial x_{t,i}}
+
\sum_{j=1}^{D}\frac{\partial \epsilon_{c,j}}{\partial x_{t,i}}\left(g_{t,j}-g_{t,i}\right)
\nonumber\\
&=
\left[\left(\mathbf{J}_{c,t}^\top\mathbf{1}\right)\odot\bm{g}_t\right]_i
+
r_{t,i},
\label{app:centering_done}
\end{align}
which is Eq.~\eqref{app:centered_identity}.
\end{proof}

\Cref{lem:centered_decomposition} is exact and makes the approximation a single, explicit act: dropping $\bm{r}_t$.

\begin{approximation}[Locally smooth guidance contrast]
\label{approx:local_smoothness}
Let $\mathcal{I}(i)=\{j:\partial\epsilon_{c,j}/\partial x_{t,i}\neq 0\}$ be the set of output coordinates that input coordinate $i$ influences. We set $\bm{r}_t\approx\bm{0}$, which by Eq.~\eqref{app:centered_identity} incurs the coordinate-wise error
\begin{equation}
\left|r_{t,i}\right|
\;\le\;
\underbrace{\Bigl(\textstyle\sum_{j}\bigl|\partial\epsilon_{c,j}/\partial x_{t,i}\bigr|\Bigr)}_{\text{column mass of }\mathbf{J}_{c,t}}
\cdot
\underbrace{\max_{j\in\mathcal{I}(i)}\left|g_{t,j}-g_{t,i}\right|}_{\text{oscillation of }\bm{g}_t\text{ on }\mathcal{I}(i)} .
\label{app:factorization_error}
\end{equation}
The approximation is therefore controlled by how much the CFG contrast $\bm{g}_t$ \emph{varies} over the denoiser's influence set, and is exact whenever $\bm{g}_t$ is constant there.
\end{approximation}

Eq.~\eqref{app:factorization_error} never constrains the off-diagonal structure of $\mathbf{J}_{c,t}$. The Jacobian may be arbitrarily dense, which for attention-based denoisers it is, and the bound still holds provided $\bm{g}_t$ is smooth across each influence set. This is a weaker requirement than the diagonal condition $\partial\epsilon_{c,j}/\partial x_{t,i}\approx 0$ for $j\neq i$ under which the same element-wise form would follow, since that route leaves the error governed by the off-diagonal mass. This motivates considering local smoothness of the conditional--unconditional contrast even for a dense Jacobian. We do not establish that smoothness uniformly for the learned denoiser; the practical surrogate is evaluated empirically.

Dropping $\bm{r}_t$ gives the surrogate
\begin{equation}
\mathbf{J}_{c,t}^\top\bm{g}_t
\;\approx\;
\left(\mathbf{J}_{c,t}^\top\mathbf{1}\right)\odot\bm{g}_t,
\qquad
\mathbf{J}_{c,t}^\top\mathbf{1}
=
\nabla_{\bm{x}_t}\!\left(\mathbf{1}^\top\bm{\epsilon}_c\right),
\label{app:factorized_vjp}
\end{equation}
where the right-hand side is one backward pass from the scalar $\mathbf{1}^\top\bm{\epsilon}_c$ and, being independent of $\bm{g}_t$, can be cached and reused across steps. Substituting into Eq.~\eqref{app:pb_step3} and restoring the control strength gives the controller used in~\Cref{alg:vog_sampling}:
\begin{equation}
\boxed{\;
\bm{\epsilon}_{\mathrm{vog}}
=
\bm{\epsilon}_c
+
\underbrace{w
\left[
\mathbf{1}
-
\rho_t\,\nabla_{\bm{x}_t}\!\left(\mathbf{1}^\top\bm{\epsilon}_c\right)
\right]
\odot
\bm{g}_t}_{\bm{\Delta}_t^{\ours}} .
\;}
\label{app:final_voc_approx}
\end{equation}
So $\ours$ leaves the reference prediction $\bm{\epsilon}_c$ intact and replaces the uniform extrapolation coefficient $w$ of CFG by the state-dependent element-wise coefficient $w[\mathbf{1}-\rho_t\nabla_{\bm{x}_t}(\mathbf{1}^\top\bm{\epsilon}_c)]$.

\begin{remark}[The reference choice does not change the reachable sampler family]
\label{rem:reference_equivalence}
Expanding Eq.~\eqref{app:final_voc_approx} gives $\bm{\epsilon}_{\mathrm{vog}}=\bm{\epsilon}_{\mathrm{cfg}}-w\rho_t\,\nabla_{\bm{x}_t}(\mathbf{1}^\top\bm{\epsilon}_c)\odot\bm{g}_t$. Had we taken the unconditional policy as the reference, the same derivation would have produced $\bm{\epsilon}_{\mathrm{cfg}}-\omega\rho_t\,\nabla_{\bm{x}_t}(\mathbf{1}^\top\bm{\epsilon}_c)\odot\bm{g}_t$, which differs only in the scalar prefactor. Since that prefactor is absorbed into the tuned schedule $\zeta(t)$ of~\Cref{thm:vog}, the two references span the same one-parameter family of samplers and every number we report is unaffected by the choice. What the conditional reference changes is the interpretation: the KL trust region is centred on the model's own conditional belief, the control is the extrapolation beyond it, and the control vanishes at $\omega=1$.
\end{remark}

A second consequence of Eq.~\eqref{app:centered_identity} is that the surrogate preserves the support of the CFG direction. The exact VJP mixes coordinates and can be nonzero at $i$ even when $g_{t,i}=0$, placing guidance where the conditional and unconditional predictions already agree; the factorized form cannot, since
\begin{equation}
g_{t,i}=0
\quad\Longrightarrow\quad
\left[(\mathbf{J}_{c,t}^\top\mathbf{1})\odot\bm{g}_t\right]_i=0 .
\label{app:support_preservation}
\end{equation}
Dropping $\bm{r}_t$ is thus also a regularization choice, and not only an approximation: it confines the correction to a coordinate-wise rescaling of the CFG direction.

\subsection{Why not differentiate the CFG-extrapolated predictor}
\label{app:direct_cfg}

A seemingly more direct instantiation applies the pullback of~\Cref{app:pullback} to the already-extrapolated predictor $\bm{\epsilon}_{\mathrm{cfg}}=\bm{\epsilon}_c+w\bm{g}_t$. This introduces a term that the conditional construction does not have. With $\mathbf{J}_{u,t}=\partial\bm{\epsilon}_u/\partial \bm{x}_t$,
\begin{align}
\mathbf{J}_{\mathrm{cfg},t}
&=
\mathbf{J}_{c,t}+w\left(\mathbf{J}_{c,t}-\mathbf{J}_{u,t}\right)
\label{app:direct_step1}\\
w\,\mathbf{A}_{\mathrm{cfg},t}^\top\bm{g}_t
&=
w a_t\bm{g}_t
-
w b_t\mathbf{J}_{c,t}^\top\bm{g}_t
-
\underbrace{w^2 b_t\left(\mathbf{J}_{c,t}-\mathbf{J}_{u,t}\right)^\top\bm{g}_t}_{\text{guidance-scale self-coupling}} .
\label{app:direct_step2}
\end{align}
The final term is quadratic in $w$ and couples the guidance direction to the guidance-dependent part of the sensitivity map, so the geometry of the correction changes as the guidance scale is varied. The same $O(w^2)$ term survives the factorization: with $\bm{h}_{c,t}=\mathbf{J}_{c,t}^\top\mathbf{1}$ and $\bm{h}_{u,t}=\mathbf{J}_{u,t}^\top\mathbf{1}$ we have $\mathbf{J}_{\mathrm{cfg},t}^\top\mathbf{1}=\bm{h}_{c,t}+w(\bm{h}_{c,t}-\bm{h}_{u,t})$, hence
\begin{equation}
w\left[a_t\mathbf{1}-b_t \bm{h}_{c,t}\right]\odot\bm{g}_t
-
w^2 b_t\left(\bm{h}_{c,t}-\bm{h}_{u,t}\right)\odot\bm{g}_t .
\label{app:direct_factorized}
\end{equation}
Taking the sensitivity from the non-extrapolated conditional branch, as in Eq.~\eqref{app:final_voc_approx}, keeps the geometry independent of $w$ and removes this coupling. \Cref{tab:controller_ablation} reports the corresponding matched comparison on DiT-XL/2.

\begin{table}[ht]
\centering
\small
\caption{The alternative construction of this section against the controller we run, on DiT-XL/2 ($256\times256$), both at $K=4$. Runtime is relative to CFG. Differentiating the already-extrapolated predictor, and thereby retaining the $O(w^2)$ self-coupling, is substantially worse.}
\label{tab:controller_ablation}
\begin{tabular}{lcc}
\toprule
Controller & FID ($\downarrow$) & Relative runtime \\
\midrule
Pullback through $\epsilon_{\mathrm{cfg}}$, Eq.~\eqref{app:direct_step2} & 5.48 & -- \\
\rowcolor{lightCyan}
Pullback through $\epsilon_c$, Eq.~\eqref{app:final_voc_approx} & \textbf{2.04} & $1.55\times$ \\
\bottomrule
\end{tabular}
\end{table}

\section{DVOC Endpoint Targets and Sparse Gradients}
\label{app:dvoc_teacher}
This section characterizes the exact algorithm used in the reward experiment. It uses a projected endpoint target and a sparse parameter-gradient path. The stochastic variational derivation motivates local control but does not prove that this finite-step deterministic algorithm realizes the optimal path measure.

\subsection{Relative-radius target construction}
Let $\bm{y}$ be the detached final latent from the sampled rollout, $d$ its dimensionality, and $r=\rho\|\bm{y}\|_2$. Define $\mathcal{B}(\bm{y},r)=\{\bm{u}:\|\bm{u}-\bm{y}\|_2\leq r\}$. Its projection is
\begin{equation}
\Pi_{\mathcal{B}(\bm{y},r)}(\bm{u})=\bm{y}+(\bm{u}-\bm{y})\min\left\{1,\frac{r}{\|\bm{u}-\bm{y}\|_2+\varepsilon}\right\}.
\end{equation}
We initialize $\bm{y}^{(0)}=\bm{y}$ and apply Eq.~\eqref{eq:target_noise} twice, differentiating the weighted reward through the FP32 VAE decoder. Explicitly, the chain rule gives
\begin{equation}
 \bm{g}^{(j)}=\nabla_{\bm{y}^{(j)}}R(D(\bm{y}^{(j)}),c)
 =\mathbf{J}_D(\bm{y}^{(j)})^\top\nabla_{\bm{I}} R(\bm{I},c)\big|_{\bm{I}=D(\bm{y}^{(j)})}.
\end{equation}
This is the first pullback, from image reward to latent coordinates. Normalized ascent and projection turn these local gradients into the finite displacement $\bm{\delta}$; in general $\bm{\delta}$ is not equal to a single $\bm{g}^{(j)}$. The target and all target-construction gradients are detached before the second, parameter-space pullback. The implementation uses $\varepsilon=10^{-12}$ for the normalization and projection denominators and zero momentum. If the reward gradient vanishes, the normalized update also vanishes.

\Needspace{7\baselineskip}
\begin{proposition}[Endpoint-radius bound]
\label{prop:dvoc_target}
For the projected iterates above, the displacement $\bm{\delta}=\bm{y}^{(J)}-\bm{y}$ satisfies $\|\bm{\delta}\|_2\leq\rho\|\bm{y}\|_2$, up to floating-point error.
\end{proposition}
\begin{proof}
Every projected iterate belongs to the closed ball $\mathcal{B}(\bm{y},r)$ by construction. The claim follows for the last iterate.
\end{proof}
This is a feasibility bound, not a reward-monotonicity or global-optimality guarantee. The fixed normalized step has no line search. It also does not establish a KL bound on the fine-tuned model or prevent cumulative parameter drift across updates.

\subsection{Gradient identity}
The no-gradient cache satisfies the numerical Euler composition $\bm{y}=\bm{x}_T+\sum_i h_i\bar{\bm{v}}_i$. At an update, re-evaluating only indices in $\mathcal{A}$ gives the surrogate $\widehat{\bm{y}}_\phi$ of Eq.~\eqref{eq:dvoc_connector}. In the connector, $\bm{y}$ and the entire correction $\bm{y}-\widehat{\bm{y}}_\phi$ are stopped in the backward pass. Its numerical value is $\bm{y}$; its specified derivative is
\begin{equation}
\mathbf{J}_{\mathcal{A}}\triangleq\frac{\partial \bm{y}_{\mathrm{conn}}}{\partial\phi}
=\sum_{i\in\mathcal{A}}h_i\frac{\partial \bm{v}_\phi(\sg(\bm{x}_{t_i}),t_i,c)}{\partial\phi}.
\end{equation}
Because $\bm{y}^+$ and $w$ are detached, the residual equals $\bm{y}-\bm{y}^+=-\bm{\delta}$. Applying the chain rule to the coordinate-mean squared loss yields Eq.~\eqref{eq:dvoc_gradient}. The target radius bounds $\bm{\delta}$, but the parameter-gradient magnitude also depends on the velocity Jacobians and the inverse weight $1/w$.

The active states are not recursively fed into one another during re-evaluation, so this derivative contains a sum of per-state parameter Jacobians, rather than a product of state-transition Jacobians. The detached trajectory cache remains necessary. Smaller $K$ reduces differentiable velocity evaluations but does not reduce the nine-step sampling rollout or the two reward-gradient evaluations used to build each target. There is no $N/K$ rescaling in the implemented loss, so varying $K$ changes its effective gradient as well as its computational cost.

\subsection{From terminal improvement to parameter realization}
\label{app:dvoc_realization}
The connection to VOC is a local design principle with two distinct operations: improve a terminal outcome under a deviation budget, then realize that improvement using an available control variable. VOG obtains a state correction from a stochastic control approximation. DVOC instead approximately optimizes an endpoint target and amortizes it into parameters using a regression surrogate. The following standard local analysis states what this second operation requires. It does not equate Euclidean endpoint distance with path-space KL or require a VOG teacher.

For fixed noise $\bm{z}$ and condition $c$, let $S_\phi(\bm{z},c)$ be the entire deterministic Euler sampler and set $F(\phi)=R(D(S_\phi(\bm{z},c)),c)$. All derivatives below are of the underlying differentiable computation at the current parameters. Write
\begin{align}
 \mathbf{M}&=\partial_\phi S_\phi(\bm{z},c),\qquad
 \bm{q}=\mathbf{M}^\top \bm{g},\quad \bm{g}=\nabla_{\bm{y}} R(D(\bm{y}),c),\\
 \mathbf{B}&=\sum_{i=0}^{N-1}h_i\partial_\phi \bm{v}_\phi(\sg(\bm{x}_{t_i}),t_i,c).
\end{align}
Here $\mathbf{B}$ holds every cached state fixed, whereas $\mathbf{M}$ retains all state-to-state dependencies. In particular, increasing $K$ to $N$ gives $\mathbf{B}$, not $\mathbf{M}$.

\paragraph{An observable ascent condition.}
Hold the target, cache, and $w$ fixed, and use the DVOC direction $\bm{u}_{\mathcal{A}}$ from Eq.~\eqref{eq:dvoc_gradient}. The chain rule gives Eq.~\eqref{eq:dvoc_realization}. If $F$ has an $L$-Lipschitz gradient on the segment from $\phi$ to $\phi+\alpha \bm{u}_{\mathcal{A}}$, then
\begin{equation}
 F(\phi+\alpha \bm{u}_{\mathcal{A}})-F(\phi)
 \geq \alpha\langle \bm{q},\bm{u}_{\mathcal{A}}\rangle
       -\frac{L\alpha^2}{2}\|\bm{u}_{\mathcal{A}}\|_2^2.
 \label{eq:dvoc_finite_bound}
\end{equation}
Thus positive alignment is sufficient for improvement at a sufficiently small step under this smoothness assumption. A positive target gain $R(D(\bm{y}+\bm{\delta}),c)-R(D(\bm{y}),c)$ alone does not ensure that alignment. Neither the unknown $L$ nor a safe step size is established for the neural reward. AdamW's preconditioned update is also generally different from $\bm{u}_{\mathcal{A}}$; its actual parameter displacement would have to be substituted into the local analysis.

\paragraph{Three approximation errors.}
For analysis only, rescale the target direction as $\bar{\bm{u}}_{\mathcal{A}}=(N/K)\mathbf{J}_{\mathcal{A}}^\top\bm{\delta}$. For any $\lambda>0$,
\begin{align}
 \bar{\bm{u}}_{\mathcal{A}}&=\lambda \bm{q}+\bm{e}_{\rm target}+\bm{e}_{\rm state}+\bm{e}_{\rm subset},\\
 \bm{e}_{\rm target}&=\mathbf{M}^\top(\bm{\delta}-\lambda \bm{g}),\\
 \bm{e}_{\rm state}&=(\mathbf{B}-\mathbf{M})^\top\bm{\delta},\\
 \bm{e}_{\rm subset}&=\big((N/K)\mathbf{J}_{\mathcal{A}}-\mathbf{B}\big)^\top\bm{\delta}.
 \label{eq:dvoc_error_decomposition}
\end{align}
The first term measures the finite target's departure from a scaled reward gradient; the second isolates frozen-state backpropagation; the third isolates subset sampling. Substitution proves the identity. By Cauchy--Schwarz, if $\bm{q}\neq\bm{0}$ and
\begin{equation}
 \|\bm{e}_{\rm target}\|_2+\|\bm{e}_{\rm state}\|_2+\|\bm{e}_{\rm subset}\|_2
 <\lambda\|\bm{q}\|_2,
\end{equation}
then $\langle \bm{q},\bar{\bm{u}}_{\mathcal{A}}\rangle>0$. This sufficient condition is conservative, not an assumed property of DVOC. It motivates measuring direction alignment and actual finite changes separately. The rescaling is used only to compare errors; the implementation has no $N/K$ factor.

For a uniform size-$K$ subset, conditional on the fixed trajectory and target, each index is included with probability $K/N$. Consequently,
\begin{equation}
 \mathbb{E}[\mathbf{J}_{\mathcal{A}}]=\frac{K}{N}\mathbf{B},
 \qquad \mathbb{E}[\bm{e}_{\rm subset}]=\bm{0}.
\end{equation}
Sparsity therefore adds subset variation around the frozen-state direction, without removing its bias relative to the complete sampler. Even an averaged sparse direction need not ascend the true reward. In BF16 execution, automatic differentiation follows the usual floating-point differentiation convention; finite perturbations and repeated forward passes are needed to assess whether a measured reward change is distinguishable from numerical effects.

\subsection{Algorithms}
\begin{algorithm}[ht]
\caption{VOG sampling: conditional reference}
\label{alg:vog_sampling}
\small
\begin{algorithmic}[1]
\REQUIRE Conditional/unconditional predictor, condition $c$, guidance strength $\omega$, momentum $\beta$, VJP stride $K_{\mathrm{VOG}}$, correction schedule $\zeta(t)$.
\STATE Sample initial noise; initialize momentum $\bm{m}=\bm{0}$ and cached correction $\bm{q}=\bm{0}$.
\FOR{each reverse sampling step $t$}
\STATE Compute conditional and unconditional predictions $\bm{\epsilon}_c,\bm{\epsilon}_u$.
\STATE Set contrast $\bm{g}=\bm{\epsilon}_c-\bm{\epsilon}_u$ and update $\bm{m}\leftarrow\beta \bm{m}+\bm{g}$.
\IF{first step or recomputation step}
\STATE Recompute the factorized conditional-branch pullback using Eq.~\eqref{app:final_voc_approx}; cache it as $\bm{q}$.
\ENDIF
\STATE Form the practical momentum-smoothed prediction and apply correction $-\zeta(t)\bm{q}$.
\STATE Advance the baseline numerical sampler with the corrected prediction.
\ENDFOR
\STATE \textbf{return} Final sample.
\end{algorithmic}
\end{algorithm}

\begin{algorithm}[ht]
\caption{DVOC: projected endpoint target and sparse Euler regression}
\label{alg:dvoc_implementation}
\small
\begin{algorithmic}[1]
\REQUIRE Frozen backbone with trainable LoRA $\phi$, decoder $D$, reward $R$, $N=9$, $K=3$, $J=2$, $\rho=0.1$, $\eta=1$.
\FOR{each optimization update}
\STATE Sample training prompts and generate no-gradient trajectories with the current policy.
\STATE Cache states $\bm{x}_{t_i}$, velocities $\bar{\bm{v}}_i$, increments $h_i$, and endpoint $\bm{y}=\bm{x}_{t_N}=\bm{x}_0$.
\STATE Set $r=\rho\|\bm{y}\|_2$ and $\bm{y}^{(0)}=\bm{y}$ for each sample.
\FOR{$j=0,\ldots,J-1$}
\STATE Compute reward pullback $\bm{g}^{(j)}=\mathbf{J}_D(\bm{y}^{(j)})^\top\nabla_{\bm{I}} R$ at $\bm{I}=D(\bm{y}^{(j)})$.
\STATE Take a normalized ascent step of length $\eta r/J$ and project onto $\mathcal{B}(\bm{y},r)$.
\ENDFOR
\STATE Detach target $\bm{y}^+=\bm{y}^{(J)}$ and displacement $\bm{\delta}=\bm{y}^+-\bm{y}$.
\FOR{each optimization microbatch}
\STATE Sample $K$ active indices uniformly without replacement from all $N$ steps.
\STATE Re-evaluate active velocities at detached cached states; retain gradients only for LoRA.
\STATE Compose $\widehat{\bm{y}}_\phi$ and form $\bm{y}_{\mathrm{conn}}=\widehat{\bm{y}}_\phi+\sg(\bm{y}-\widehat{\bm{y}}_\phi)$.
\STATE Compute detached weight $w=\max(\|\bm{\delta}\|_1/d,10^{-5})$ and loss in Eq.~\eqref{eq:distillation_loss}.
\STATE Accumulate gradients, scaled by the number of microbatches.
\ENDFOR
\STATE Average gradients across GPUs, clip the norm, update LoRA with AdamW, and zero gradients.
\STATE Update the EMA adapter used for checkpoint evaluation.
\ENDFOR
\end{algorithmic}
\end{algorithm}

\section{Mean Shift of CFG under the VP Probability Flow ODE}
\label{sec:cfg_span_shift}

This appendix derives the mean shift induced by Classifier-Free Guidance under the Variance Preserving (VP) probability flow ODE, supplying the result quoted in~\Cref{eq:cfg_final_shift}. We keep the guidance term time-dependent rather than freezing it. Only the denoiser is idealized, through the two assumptions below; every other hypothesis is carried explicitly by the statement that uses it. \Cref{lem:cfg_drift} identifies the drift CFG injects, \Cref{prop:cfg_mean_shift} integrates it, \Cref{prop:cfg_offspan} resolves the result along the data subspace, and \Cref{rem:offspan_mechanism} reduces the subspace question to a single measurable property of the trained network.

\paragraph{Setting.}
We analyze the deterministic sampler: the probability flow ODE of the VP-SDE~\cite{Song2021SDE}, integrated from $t=T$ (noise) to $t=0$ (data). With $\beta_{\mathrm{vp}}(t)$ the VP noise rate and $\bar\alpha_t=\exp(-\int_0^t\beta_{\mathrm{vp}}(s)\,ds)$, the forward process is $\bm x_t=\sqrt{\bar\alpha_t}\bm x_0+\sqrt{1-\bar\alpha_t}\bm\epsilon$; $\beta_{\mathrm{vp}}$ is unrelated to the momentum coefficient $\beta$ and to the correction schedule $\zeta(t)$ of~\Cref{eq:final_form}. Classes are indexed by $y$, with $\bm\mu_y$ the mean of $q(\bm x_0\mid y)$ and $\bm\mu_{\mathrm{global}}=\mathbb E_q[\bm x_0]$. The one state-dependent object in the analysis is the posterior-weighted \emph{dynamic mean}
\begin{equation}
\bar{\bm\mu}(\bm x_t)\triangleq\sum_y p(y\mid \bm x_t)\,\bm\mu_y .
\label{app:dynamic_mean}
\end{equation}
We write $\omega$ for the guidance scale and $w\triangleq\omega-1$, so that $\tilde{\bm\epsilon}=\bm\epsilon_\theta(\bm x_t,c)+w\,(\bm\epsilon_\theta(\bm x_t,c)-\bm\epsilon_\theta(\bm x_t,\varnothing))$. Finally $S=\operatorname{span}(U)$ is the subspace of~\Cref{assump:linear_subspace}, $P$ the orthogonal projector onto it, and $P_\perp=I_D-P$.

\begin{assumption}[Bayes-optimal denoiser]
\label{ass:mmse}
$\bm\epsilon_\theta$ attains the global optimum of the $L_2$ objective, so $\bm\epsilon_\theta^*(\bm x_t,t)=\mathbb E[\bm\epsilon\mid \bm x_t]$ and equivalently $\hat{\bm x}_0(\bm x_t)=\mathbb E[\bm x_0\mid\bm x_t]$. It enters only through the tower rule, $\hat{\bm x}_0^{\mathrm{uc}}=\sum_y p(y\mid\bm x_t)\,\hat{\bm x}_0^{y}$, and through the score representation behind Eq.~\eqref{app:final_ode}.
\end{assumption}

\begin{assumption}[Locally affine class denoisers]
\label{ass:affine}
Near their respective means the class-conditional denoisers are affine, with a symmetric contraction $\bm M_t\in\mathbb R^{D\times D}$ that is \emph{shared across classes}:
\begin{equation}
\hat{\bm x}_0^{y}(\bm x_t)=\bm\mu_y+\bm M_t\,(\bm x_t-\bm\mu_y)
\qquad\text{for every class }y .
\label{app:affine_denoiser}
\end{equation}
This is a local affine idealization, not the exact VP Gaussian posterior formula: a Gaussian class mean is scaled by $\sqrt{\bar\alpha_t}$ in the noisy coordinate. Sharing $\bm M_t$ isolates class dependence in the mean terms. The conclusions below apply to this stated local model.
\end{assumption}

\begin{remark}[The unconditional denoiser is derived, not assumed]
\label{rem:uc_not_affine}
\Cref{ass:mmse,ass:affine} already determine the unconditional branch, $\hat{\bm x}_0^{\mathrm{uc}}(\bm x_t)=\bar{\bm\mu}(\bm x_t)+\bm M_t(\bm x_t-\bar{\bm\mu}(\bm x_t))$, and it is \emph{not} affine: $\bar{\bm\mu}$ moves with the state. That motion is the mechanism the analysis is after, so assuming it away would empty the result.
\end{remark}

\subsection{The drift that CFG injects}

\begin{lemma}[CFG as a dynamic drift in denoiser space]
\label{lem:cfg_drift}
Under~\Cref{ass:mmse,ass:affine},
\begin{equation}
\hat{\bm x}_0^{\mathrm{cfg}}(\bm x_t)
=
\hat{\bm x}_0^{\mathrm{uc}}(\bm x_t)
+
\omega\,(\bm I-\bm M_t)\big(\bm\mu_c-\bar{\bm\mu}(\bm x_t)\big).
\label{app:cfg_denoiser}
\end{equation}
\end{lemma}

\begin{proof}
The clean-prediction map $\hat{\bm x}_0=(\bm x_t-\sqrt{1-\bar\alpha_t}\,\bm\epsilon)/\sqrt{\bar\alpha_t}$ is affine in $\bm\epsilon$, so the CFG combination transfers verbatim from noise space to denoiser space. Reading the steps in order,
\begin{align}
\hat{\bm x}_0^{\mathrm{cfg}}
&=\hat{\bm x}_0^{c}+w\big(\hat{\bm x}_0^{c}-\hat{\bm x}_0^{\mathrm{uc}}\big)
\label{app:cfg_den_step1}\\
\hat{\bm x}_0^{\mathrm{uc}}
&=\bar{\bm\mu}(\bm x_t)+\bm M_t\big(\bm x_t-\bar{\bm\mu}(\bm x_t)\big)
\nonumber\\
\hat{\bm x}_0^{c}-\hat{\bm x}_0^{\mathrm{uc}}
&=(\bm I-\bm M_t)\big(\bm\mu_c-\bar{\bm\mu}(\bm x_t)\big)
\label{app:cfg_den_step3}
\end{align}
and substituting Eq.~\eqref{app:cfg_den_step3} into Eq.~\eqref{app:cfg_den_step1}, with $1+w=\omega$, gives Eq.~\eqref{app:cfg_denoiser}.
\end{proof}

Two features of Eq.~\eqref{app:cfg_denoiser} drive everything that follows. The drift is proportional to $\bm\mu_c-\bar{\bm\mu}(\bm x_t)$, the residual uncertainty about the class, so it vanishes when the posterior mean equals the target class mean. And it is filtered by $\bm I-\bm M_t$, the part of the state the denoiser does \emph{not} contract; \Cref{rem:offspan_mechanism} shows that this factor carries the entire geometric content of the result.

\subsection{The guided probability flow ODE}

The VP probability flow ODE reads $d\bm x_t/dt=-\tfrac12\beta_{\mathrm{vp}}(t)\bm x_t-\tfrac12\beta_{\mathrm{vp}}(t)\nabla_{\bm x}\log p_t(\bm x_t)$, and~\Cref{ass:mmse} supplies the score representation $\nabla_{\bm x}\log p_t(\bm x_t)=-(\bm x_t-\sqrt{\bar\alpha_t}\,\hat{\bm x}_0(\bm x_t))/(1-\bar\alpha_t)$. Hence
\begin{align}
\frac{d\bm x_t}{dt}
&=-\tfrac12\beta_{\mathrm{vp}}(t)\bm x_t
+\tfrac12\beta_{\mathrm{vp}}(t)\frac{\bm x_t-\sqrt{\bar\alpha_t}\,\hat{\bm x}_0(\bm x_t)}{1-\bar\alpha_t}
\nonumber\\
&=\Lambda_1(t)\,\bm x_t+\Lambda_2(t)\,\hat{\bm x}_0(\bm x_t),
\label{app:final_ode}
\end{align}
where, for $t\in(0,T)$,
\begin{equation}
\Lambda_1(t)=\frac{\beta_{\mathrm{vp}}(t)\bar\alpha_t}{2(1-\bar\alpha_t)}>0,
\qquad
\Lambda_2(t)=-\frac{\beta_{\mathrm{vp}}(t)\sqrt{\bar\alpha_t}}{2(1-\bar\alpha_t)}<0 .
\label{app:lambda_defs}
\end{equation}
Substituting~\Cref{lem:cfg_drift} splits the guided field into the unguided flow plus a guidance term proportional to $\omega$:
\begin{equation}
\frac{d\bm x_t}{dt}
=
\underbrace{\Lambda_1(t)\bm x_t+\Lambda_2(t)\,\hat{\bm x}_0^{\mathrm{uc}}(\bm x_t)}_{\text{unguided probability flow}}
\;+\;
\omega\,\Lambda_2(t)\,(\bm I-\bm M_t)\big(\bm\mu_c-\bar{\bm\mu}(\bm x_t)\big).
\label{app:split_field}
\end{equation}
The unguided field is the one the model is trained to match, so it transports $\bm x_T$ to a sample of $q$ and yields $\mathbb E[\bm x_0]=\bm\mu_{\mathrm{global}}$ at $\omega=0$. The guidance term is therefore a perturbation of a flow whose mean we already know, which is what makes the next step a variation-of-constants calculation.

\subsection{Integrating the perturbation}

\begin{proposition}[CFG mean shift]
\label{prop:cfg_mean_shift}
Let $\bm A(t)=\partial_{\bm x}[\Lambda_1(t)\bm x+\Lambda_2(t)\hat{\bm x}_0^{\mathrm{uc}}(\bm x)]$ be the Jacobian of the unguided field, $\bm\Phi(0,\tau)$ the associated state transition matrix, and define the \emph{flow kernel}
\begin{equation}
\bm\Psi(\tau)
\triangleq
\bm\Phi(0,\tau)\,
\frac{\beta_{\mathrm{vp}}(\tau)\sqrt{\bar\alpha_\tau}}{2(1-\bar\alpha_\tau)}\,
(\bm I-\bm M_\tau).
\label{app:flow_kernel}
\end{equation}
If the kernel is integrable against the class-uncertainty gap,
\begin{equation}
\int_0^T\big\|\bm\Psi(\tau)\big\|\,\big\|\bm\mu_c-\bar{\bm\mu}(\bm x_\tau)\big\|_2\,d\tau<\infty ,
\label{app:integrability}
\end{equation}
then the linearized response to the guidance perturbation is
\begin{equation}
\mathbb E[\bm x_0]
\approx
\bm\mu_{\mathrm{global}}
+
\omega\,\bm\Delta_{\mathrm{eff}},
\qquad
\bm\Delta_{\mathrm{eff}}
\triangleq
\mathbb E\!\left[\int_0^T\bm\Psi(\tau)\big(\bm\mu_c-\bar{\bm\mu}(\bm x_\tau)\big)\,d\tau\right],
\label{app:mean_shift_final}
\end{equation}
which is Eq.~\eqref{eq:cfg_final_shift}.
\end{proposition}

\begin{proof}
Let $\bm x_\tau^{\mathrm{uc}}$ denote the unguided trajectory started from the same $\bm x_T$, and $\tilde{\bm x}_\tau\triangleq\bm x_\tau-\bm x_\tau^{\mathrm{uc}}$ the deviation caused by the guidance term of Eq.~\eqref{app:split_field}. Linearizing the unguided field about $\bm x_\tau^{\mathrm{uc}}$ gives the linear time-varying system $d\tilde{\bm x}_\tau/d\tau=\bm A(\tau)\tilde{\bm x}_\tau+\bm b(\tau)$ with $\bm b(\tau)=\omega\Lambda_2(\tau)(\bm I-\bm M_\tau)(\bm\mu_c-\bar{\bm\mu}(\bm x_\tau))$ and $\tilde{\bm x}_T=\bm 0$. Then
\begin{align}
\tilde{\bm x}_0
&=\int_T^0\bm\Phi(0,\tau)\,\bm b(\tau)\,d\tau
\nonumber\\
&=\omega\int_0^T\bm\Phi(0,\tau)\,\frac{\beta_{\mathrm{vp}}(\tau)\sqrt{\bar\alpha_\tau}}{2(1-\bar\alpha_\tau)}\,(\bm I-\bm M_\tau)\big(\bm\mu_c-\bar{\bm\mu}(\bm x_\tau)\big)\,d\tau
\nonumber\\
&=\omega\int_0^T\bm\Psi(\tau)\big(\bm\mu_c-\bar{\bm\mu}(\bm x_\tau)\big)\,d\tau .
\label{app:integral_sol}
\end{align}
Condition~\eqref{app:integrability} lets us take expectations under the integral sign, and $\mathbb E[\bm x_0^{\mathrm{uc}}]=\bm\mu_{\mathrm{global}}$ gives Eq.~\eqref{app:mean_shift_final}.
\end{proof}

\begin{remark}[Condition~\eqref{app:integrability} is not automatic]
\label{rem:integrability}
As $\tau\to0$ we have $1-\bar\alpha_\tau\approx\beta_{\mathrm{vp}}(0)\tau$, so the kernel is singular, $\bm\Psi(\tau)=O(1/\tau)$. Integrability therefore requires the class posterior to sharpen fast enough that the gap closes faster than the kernel blows up. A sufficient condition is a H\"older rate $\|\bm\mu_c-\bar{\bm\mu}(\bm x_\tau)\|_2=O\big((1-\bar\alpha_\tau)^{1+\gamma}\big)$ with $\gamma>0$, which makes the integrand $O(\tau^\gamma)$. We state this as a condition on the integral rather than as a structural assumption because it is exactly what the argument consumes, and nothing downstream depends on the particular rate.
\end{remark}

\subsection{Resolving the shift along the data subspace}

\begin{proposition}[Component outside the data subspace]
\label{prop:cfg_offspan}
Assume in addition that every class mean lies in $S$, so that $\bm v(\tau)\triangleq\mathbb E[\bm\mu_c-\bar{\bm\mu}(\bm x_\tau)]\in S$. Then
\begin{equation}
P_\perp\,\mathbb E[\bm x_0]\approx\omega\,\mathbb E\!\left[\int_0^T P_\perp\bm\Psi(\tau)P\big(\bm\mu_c-\bar{\bm\mu}(\bm x_\tau)\big)\,d\tau\right].
\label{app:proj_expect}
\end{equation}
\end{proposition}

\begin{proof}
Apply $P_\perp$ to Eq.~\eqref{app:mean_shift_final}. Since $\bm\mu_{\mathrm{global}}\in S$ the first term drops. The class-mean contrast lies in $S$ for each trajectory, so $P$ can be inserted inside the expectation.
\end{proof}

Eq.~\eqref{app:proj_expect} describes the linearized response and retains the path-dependent kernel inside the expectation. Its normal component depends on the cross-block $P_\perp\bm\Psi(\tau)P$.

\begin{remark}[The subspace question reduces to one property of the denoiser]
\label{rem:offspan_mechanism}
Every factor of $\bm\Psi$ that could move a vector out of $S$ is governed by $\bm M_t$. The generator is $\bm A(t)=\Lambda_1(t)\bm I+\Lambda_2(t)\bm M_t+\Lambda_2(t)(\bm I-\bm M_t)\,\partial_{\bm x}\bar{\bm\mu}$, in which $\Lambda_1\bm I$ is isotropic and $\partial_{\bm x}\bar{\bm\mu}$ has range in $\operatorname{span}\{\bm\mu_y\}\subseteq S$; the filter $\bm I-\bm M_\tau$ maps $S$ into $S$ exactly when $\bm M_\tau$ does. Consequently
\begin{equation}
P_\perp\bm M_t P=0\ \text{ for a.e. }t
\quad\Longrightarrow\quad
P_\perp\bm\Psi(\tau)P=0\ \text{ for a.e. }\tau
\quad\Longrightarrow\quad
P_\perp\,\mathbb E[\bm x_0]=\bm 0 .
\label{app:offspan_iff}
\end{equation}
Subspace preservation by the contraction is sufficient to eliminate this normal component. Failure of that condition only permits leakage: the class contrast may not excite the offending directions, and contributions can cancel along the trajectory. We do not assert the converse implication.
\end{remark}

\begin{remark}[What the leak requires, and its relation to~\Cref{assump:linear_subspace}]
\label{rem:offspan_requires_imperfection}
\Cref{assump:linear_subspace} and~\Cref{ass:mmse} cannot both hold alongside $P_\perp\bm M_tP\neq0$. If the data lie exactly in $S$ then $\mathbb E[\bm x_0\mid\bm x_t]\in S$ for every $\bm x_t$, hence $P_\perp\bm M_t=0$, and Eq.~\eqref{app:offspan_iff} keeps the guided mean inside $S$ at every guidance scale. The off-subspace component is thus not a property of exact CFG applied to an exact denoiser; it is what an \emph{imperfect} denoiser produces, one whose linearization couples tangential to normal directions. That is the regime real samplers operate in, and it is also the regime in which the contrast with $\ours$ has content: $\ours$ passes its correction through $\bm J_t^\top$ and so confines it to directions the \emph{learned} denoiser is sensitive to, whereas CFG adds to $\bm\epsilon$ directly and inherits no such constraint. \Cref{thm:geometric_invariance} makes the complementary point in the opposite idealization, taking the denoiser exact and showing the correction is then a projection. Neither idealization covers both effects at once; \Cref{sec:conclusion_limitations} records this.
\end{remark}

\begin{remark}[Dependence on $\omega$]
\label{rem:omega_dependence}
The perturbation formula is first order in the guidance-induced trajectory deviation. Its kernel and class-posterior gap depend on the trajectory, so it does not establish a globally linear dependence on $\omega$. Writing $\omega=1+w$ separates conditional prediction from further extrapolation at the predictor level, not an exact decomposition of the final sample distribution.
\end{remark}

\section{Proof of Manifold Preservation}
\label{app:voc_geometric_invariance}

\begin{theorem}[Manifold preservation, restating~\Cref{thm:geometric_invariance}]
    \label{thm:geometric_invariance_app}
    Let $\mathcal{S} = \mathrm{span}(\mathbf{U}) \subset \mathbb{R}^D$ with $\mathbf{U} \in \mathbb{R}^{D \times d}$ having orthonormal columns ($\mathbf{U}^\top\mathbf{U}=\mathbf{I}_d$), and let $\mathbf{P}_{\perp} = \mathbf{I}_D - \mathbf{U}\mathbf{U}^\top$. Under the optimal linear MMSE denoiser $\hat{\bm{x}}_0(\bm{x}_t)=\kappa_t \mathbf{U}\mathbf{U}^\top \bm{x}_t$, the $\ours$ correction direction $\bm{g}_{\mathrm{vog}} := \mathbf{J}_t^\top \nabla_{\hat{\bm{x}}_0} f(\hat{\bm{x}}_0)$ satisfies $\mathbf{P}_{\perp}\bm{g}_{\mathrm{vog}} = \bm{0}$.
\end{theorem}

\begin{proof}
    By the chain rule the correction direction is $\bm{g}_{\mathrm{vog}} = [\partial \hat{\bm{x}}_0(\bm{x}_t)/\partial \bm{x}_t]^\top \nabla_{\hat{\bm{x}}_0} f(\hat{\bm{x}}_0)$. For the linear subspace denoiser, with $\kappa_t \in \mathbb{R}$ the time-dependent signal scaling,
    \begin{align}
        \mathbf{J}_t
        &= \frac{\partial}{\partial \bm{x}_t} \left( \kappa_t \mathbf{U}\mathbf{U}^\top \bm{x}_t \right)
        = \kappa_t \mathbf{U}\mathbf{U}^\top
        \nonumber\\
        \bm{g}_{\mathrm{vog}}
        &= \left( \kappa_t \mathbf{U}\mathbf{U}^\top \right)^\top \nabla_{\hat{\bm{x}}_0} f(\hat{\bm{x}}_0)
        = \mathbf{U}\underbrace{\left( \kappa_t \mathbf{U}^\top \nabla_{\hat{\bm{x}}_0} f(\hat{\bm{x}}_0) \right)}_{\bm{z}\,\in\,\mathbb{R}^d}.
        \label{app:geo_step2}
    \end{align}
    Any vector of the form $\mathbf{U}\bm{z}$ lies in the column space of $\mathbf{U}$, which is $\mathcal{S}$. Hence $\bm{g}_{\mathrm{vog}} \in \mathcal{S}$ and $\mathbf{P}_\perp\bm{g}_{\mathrm{vog}}=\bm{0}$.
\end{proof}

\begin{remark}[Scope of the guarantee]
\label{rem:theorem_scope}
\Cref{thm:geometric_invariance_app} constrains only the added correction under the stated linear MMSE model. It does not compare exact CFG and exact VOG under an imperfect-denoiser assumption, and it does not establish invariance of a full noisy trajectory. For nonlinear denoisers the Jacobian need not be a projector. The factorized controller introduces another approximation. Numerical schemes can preserve invariant linear subspaces, so discretization alone is not a reason to infer normal drift.
\end{remark}

\section{Detailed Experimental Settings}
\subsection{General Configuration and Hyperparameters}
\label{subsec:implementaion}

\begin{table}[ht]
    \centering
    \footnotesize

    \setlength{\tabcolsep}{3.5pt} 
    \renewcommand{\arraystretch}{1.1}

    \caption{Hyperparameters and settings for ImageNet-64 and ImageNet-512. The experiments are conducted using EDM2-L and EDM2-S architectures. The column $\omega$ denotes the guidance scale. The column Guidance interval specifies the range of noise levels $\sigma$ where classifier-free guidance is applied. For our proposed method (+Ours), $k$ represents the frequency of regularization steps (computing gradients every $k$ steps), and $\beta$ indicates the momentum coefficient. The parameter $s$ refers to the exponent used in Guidance cosine. The columns $\text{EMA}_m$ and $\text{EMA}_g$ indicate the EMA decay lengths for the main generative model and the guiding model, respectively.}
    \label{tab:hyperparameters_edm}

    \begin{NiceTabular}{llccccccc}
        \toprule
        Resolution & Method & $\omega$ & $K$ & $\beta$ & $s$ & \Block{}{Guidance\\interval} & $\text{EMA}_m$ & $\text{EMA}_g$ \\
        \midrule

        \Block[fill=white]{10-1}{$64 \times 64$} 
            & \textbf{EDM2-L} & & & & & & & \\
            & Classifier-free guidance & 1.15 & - & - & - & Full & 0.040 & 0.040 \\
            \rowcolor{lightCyan}
            & + $\ours$ (Ours) & 1.05 & 2 & 0.15 & - & Full & 0.040 & 0.040 \\
            & Guidance interval & 1.15 & - & - & - & (0.3,1.5] & 0.040 & 0.040 \\
            \rowcolor{lightCyan}
            & + $\ours$ (Ours) & 1.10 & 2 & 0.15 & - & (0.3,1.5] & 0.040 & 0.040 \\
            & Guidance cosine & 1.05 & - & - & 0.3 & Full & 0.040 & 0.040 \\
            \rowcolor{lightCyan}
            & + $\ours$ (Ours) & 1.05 & 2 & 0.15 & 0.3 & Full & 0.040 & 0.040 \\
            & \textbf{EDM2-S} & & & & & & & \\
            & Autoguidance & 1.70 & - & - & - & Full & 0.045 & 0.110 \\
            \rowcolor{lightCyan}
            & + $\ours$ (Ours) & 1.70 & 2 & 0.15 & - & Full & 0.045 & 0.110 \\

        \midrule

        \Block[fill=white]{7-1}{$512 \times 512$} 
            & \textbf{EDM2-S} & & & & & & & \\
            & Classifier-free guidance & 1.40 & - & - & - & Full & 0.025 & 0.025 \\
            \rowcolor{lightCyan}
            & + $\ours$ (Ours) & 1.60 & 2 & 0.15 & - & Full & 0.025 & 0.025 \\
            & Guidance interval & 2.10 & - & - & - & (0.28,2.90] & 0.025 & 0.025 \\
            \rowcolor{lightCyan}
            & + $\ours$ (Ours) & 2.30 & 2 & 0.15 & - & (0.28,2.90] & 0.025 & 0.025 \\
            & Guidance cosine & 1.60 & - & - & 0.3 & Full & 0.025 & 0.025 \\
            \rowcolor{lightCyan}
            & + $\ours$ (Ours) & 1.65 & 2 & 0.15 & 0.3 & Full & 0.025 & 0.025 \\

        \bottomrule
    \end{NiceTabular}
\end{table}

\begin{table}[ht]
    \centering
    \footnotesize
    \setlength{\tabcolsep}{3.5pt} 
    \renewcommand{\arraystretch}{1.2}

    \caption{Hyperparameters and settings for ImageNet-256 and ImageNet-512. The experiments are conducted using the DiT-XL/2 architecture. The column $\omega$ denotes the guidance scale. The column Guidance interval specifies the range of noise levels $\sigma$ where classifier-free guidance is applied. For our proposed method (+Ours), $k$ represents the frequency of regularization steps (computing gradients every $k$ steps), and $\beta$ indicates the momentum coefficient. The parameter $s$ refers to the exponent used in Guidance cosine.}
    \label{tab:hyperparameters_dit}

    \begin{NiceTabular}{llccccc}
        \toprule
        Resolution & Method & $\omega$ & $K$ & $\beta$ & $s$ & \Block{}{Guidance\\interval} \\
        \midrule

        \Block[fill=white]{6-1}{$256 \times 256$} 
            & Classifier-free guidance & 1.50 & - & - & - & Full \\
            \rowcolor{lightCyan}
            & + $\ours$ (Ours) & 1.50 & 4 & 0.15 & - & Full \\
            & Guidance interval & 1.50 & - & - & - & (0.34, 1.02] \\
            \rowcolor{lightCyan}
            & + $\ours$ (Ours) & 1.90 & 2 & 0.15 & - & (0.34, 1.02] \\
            & Guidance cosine & 1.50 & - & - & 0.3 & Full \\
            \rowcolor{lightCyan}
            & + $\ours$ (Ours) & 1.50 & 2 & 0.15 & 0.3 & Full \\

        \midrule

        \Block[fill=white]{6-1}{$512 \times 512$} 
            & Classifier-free guidance & 1.50 & - & - & - & Full \\
            \rowcolor{lightCyan}
            & + $\ours$ (Ours) & 1.50 & 2 & 0.15 & - & Full \\
            & Guidance interval & 2.50 & - & - & - & (0.34, 1.02] \\
            \rowcolor{lightCyan}
            & + $\ours$ (Ours) & 2.50 & 2 & 0.15 & - & (0.34, 1.02] \\
            & Guidance cosine & 1.50 & - & - & 0.3 & Full \\
            \rowcolor{lightCyan}
            & + $\ours$ (Ours) & 1.50 & 2 & 0.20 & 0.3 & Full \\

        \bottomrule
    \end{NiceTabular}
\end{table}

The VOG class-conditional sampling experiments are implemented in JAX, whereas the DVOC and baseline post-training comparisons are implemented in PyTorch. All ImageNet FID scores are computed over 50,000 generated samples per configuration. Sampling parameters follow standard practice for each architecture: for EDM2 models we use a per-device evaluation batch size of 8, and for DiT models we use 250 sampling steps with a per-process batch size of 8 and the MSE-tuned VAE from the official Stable Diffusion repository.

Because the optimal guidance scale $\omega$ varies across architectures and guidance schedules, we conduct a systematic grid search to determine the operating range for each baseline and VOG pairing. The full hyperparameter configurations and selected $\omega$ values are reported in~\Cref{tab:hyperparameters_edm,tab:hyperparameters_dit}. All class-conditional sampling experiments run on 8 NVIDIA RTX 4090 GPUs. Each DVOC post-training run uses 8 GPUs per node (identified in runtime device records as NVIDIA L20X), and we evaluate the 20-update EMA checkpoint.

\subsection{Difficulty Metric of 2D Synthetic Example} To quantitatively evaluate the complexity and difficulty of the learned distributions in the 2D synthetic experiments, we analyze the parameters of the underlying Gaussian Mixture Model (GMM). Let the distribution be defined by $K$ components with parameters $\{w_k, \boldsymbol{\mu}_k, \boldsymbol{\Sigma}_k\}_{k=1}^K$, where $w_k$ represents the mixture weight, $\boldsymbol{\mu}_k \in \mathbb{R}^2$ is the mean, and $\boldsymbol{\Sigma}_k \in \mathbb{R}^{2 \times 2}$ is the covariance matrix of the $k$-th component.

We employ three metrics to measure the diversity and concentration of the mixture weights. These metrics indicate whether the distribution is dominated by a few modes or spread across many.

\begin{itemize}
    \item \textbf{Effective Number of Components (Entropy-based), $K_{\text{ent}}$}: 
    Computed as the exponential of the Shannon entropy of the weights. This represents the theoretical number of uniform components required to produce the same entropy.
    \begin{equation}
        K_{\text{ent}} = \exp\left( -\sum_{k=1}^K w_k \ln(w_k + \epsilon) \right)
    \end{equation}
    where $\epsilon$ is a negligible constant for numerical stability.

    \item \textbf{Effective Number of Components (HHI-based), $K_{\text{hhi}}$}: 
    Derived from the inverse of the Herfindahl-Hirschman Index (Simpson index), this metric is sensitive to the dominance of heavy-tailed weights.
    \begin{equation}
        K_{\text{hhi}} = \frac{1}{\sum_{k=1}^K w_k^2}
    \end{equation}

    \item \textbf{Top Mass Fraction ($\text{Mass}_{\text{top}}$)}: 
    Measures the cumulative probability mass held by the largest fraction of components. Let $w_{(1)} \ge w_{(2)} \ge \dots \ge w_{(K)}$ be the sorted weights. For a given ratio $\alpha$ (e.g., $\alpha=0.05$), we define the cutoff index $M = \lceil \alpha K \rceil$.
    \begin{equation}
        \text{Mass}_{\text{top}} = \sum_{j=1}^{M} w_{(j)}
    \end{equation}
\end{itemize}

These metrics quantify the spatial dispersion and the structural layout of the modes in the 2D plane.

\begin{itemize}
    \item \textbf{Trace of total covariance ($\text{tr}(\text{Cov})$)}: 
    This measures the total variance (energy) of the mixture distribution. By the Law of Total Variance, it decomposes into the weighted sum of internal variances and the dispersion of component means from the global mean $\bar{\boldsymbol{\mu}} = \sum_{k} w_k \boldsymbol{\mu}_k$.
    \begin{equation}
        \text{tr}(\text{Cov}) = \sum_{k=1}^K w_k \left( \text{Tr}(\boldsymbol{\Sigma}_k) + \|\boldsymbol{\mu}_k - \bar{\boldsymbol{\mu}}\|_2^2 \right)
    \end{equation}

    \item \textbf{Minimum Spanning Tree Length (MST)}: 
    To capture the structural spread of the modes, we construct a complete graph where nodes correspond to component means $\{\boldsymbol{\mu}_k\}$. The edge weight between nodes $i$ and $j$ is the Euclidean distance $\|\boldsymbol{\mu}_i - \boldsymbol{\mu}_j\|_2$. The metric is defined as the total weight of the Minimum Spanning Tree ($T$) of this graph:
    \begin{equation}
        \mathcal{L}_{\text{MST}} = \sum_{(i,j) \in E(T)} \|\boldsymbol{\mu}_i - \boldsymbol{\mu}_j\|_2
    \end{equation}
    A higher MST length indicates that the modes are widely separated across the manifold, increasing the difficulty of traversing the density landscape.
\end{itemize}
\begin{table}[h]
    \centering
    \caption{Quantitative difficulty metrics for the 2D synthetic datasets (Class A).}
    \label{tab:synthetic_metrics}
    \setlength{\tabcolsep}{5pt}
    \begin{tabular}{lcccccc}
        \toprule
        \textbf{Dataset} & $\bm{K}$ & $\bm{K_{\text{ent}}}$ & $\bm{K_{\text{hhi}}}$ & \textbf{Top 5\% Mass} & $\bm{\text{tr}(\text{Cov})}$ & \textbf{MST} \\
        \midrule
        CLS-2 (Class A) & 1016 & 249.81 & 90.29  & 0.5429 & 0.2275 & 15.35 \\
        CLS-3 (Class A) & 1036 & 214.05 & 102.91 & 0.6361 & 0.4029 & 18.57 \\
        \bottomrule
    \end{tabular}
\end{table}
\noindent As indicated in \Cref{tab:synthetic_metrics}, the \textbf{CLS-3} dataset is the harder of the two. Geometrically, CLS-3 exhibits a much larger spatial dispersion, evidenced by the higher total variance ($\text{tr}(\text{Cov}) \approx 0.40$ vs. $0.23$) and a more extensive Minimum Spanning Tree ($18.57$ vs. $15.35$). Statistically, the distribution of CLS-3 is more skewed; despite having a similar number of total components, its probability mass is more concentrated in the top few modes (Top 5\% Mass $\approx 63.6\%$), resulting in a lower entropy-based effective number of components ($K_{\text{ent}}$). This combination of high spatial variance and weight imbalance requires the generative model to capture a broader manifold structure while handling sharper density peaks.

\subsection{DVOC and Baseline Reproduction Details}
\label{subsec:dvoc_protocol}
The main DVOC configuration uses $K=3$ active velocity evaluations per trajectory, with $K=1$ and $K=2$ evaluated in the active-step ablation. We compare against the untrained Z-Image-Turbo reference and four trained methods: ReFL, DRTune, DiffusionOPSD, and DiffusionNFT. All scores are generated in the same comparison workspace rather than imported from other backbones or budgets.

\paragraph{Hardware and update accounting.}
Each trained main method and both DVOC $K$ ablations use one node with eight GPUs, with reward models colocated on the same devices; no additional reward-server GPUs are reserved. The untrained reference also uses eight GPUs for generation and scoring, with no training update. The scheduler resource class is \texttt{h200}, whereas the eight-rank runtime checks report \texttt{NVIDIA L20X}, each with 150,121,283,584 bytes (approximately 139.8\,GiB) of device memory. We retain both identifiers rather than infer a physical GPU model from the scheduler label. The execution metadata records Python 3.10.21 and PyTorch 2.6.0 with CUDA 12.4.

One \emph{training step} denotes one optimizer update, including its generated trajectories, reward or target computation, accumulated backward passes, and parameter update. It is distinct from an Euler denoising step. The effective trajectory batch is $8\times1\times72=576$, comprising 48 prompt groups with 12 trajectories each. Thus update 20 corresponds to 11,520 training trajectories. ReFL and DRTune stop their training rollouts at the selected late states; the common trajectory budget does not imply that they decode 576 full terminal images. Their separate 512-prompt reward calibration is initialization overhead, outside this per-update trajectory count. The configured runs continue to 100 updates; the main table evaluates their update-20 EMA checkpoints.

\begin{table}[ht]
\centering\small
\caption{Common training and evaluation settings. Method-specific losses, weight decay, and backward paths are detailed below.}
\label{tab:dvoc_setup}
\begin{tabular}{@{}p{.28\linewidth}p{.67\linewidth}@{}}
\toprule
Item & Setting \\
\midrule
Backbone / resolution & Z-Image-Turbo; $1024\times1024$ \\
Sampler & Nine deterministic Euler steps; guidance scale 0 \\
Hardware per run & One node; eight GPUs; colocated reward models \\
Prompt / trajectory budget & Pick-a-Pic training prompts; 48 groups $\times$ 12 trajectories per update \\
Optimizer batch & One trajectory per GPU microbatch; accumulation over 72 trajectories per GPU \\
Training reward & $0.03846\,R_{\mathrm{HPSv3++}}+R_{\mathrm{CLIP}}$ \\
Trainable parameters & Transformer LoRA, rank 32 and alpha 64; frozen base weights, VAE, text encoder, and reward models \\
LoRA injection & Attention \texttt{to\_q}, \texttt{to\_k}, \texttt{to\_v}, \texttt{to\_out.0}; MLP \texttt{w1}, \texttt{w2}, \texttt{w3}; dropout 0; no trainable bias \\
Text conditioning & Maximum Qwen token sequence length 512 \\
AdamW & Constant learning rate $3\times10^{-4}$; $(\beta_1,\beta_2)=(0.9,0.999)$; $\epsilon=10^{-8}$ \\
Gradient clipping & Global parameter-gradient norm clipped at 1 \\
Precision / memory & BF16 transformer; FP32 VAE; TF32 allowed; transformer gradient checkpointing enabled \\
EMA & Updated every optimizer step; decay $\min\{(1+s)/(10+s),0.9\}$ at completed update $s$ \\
Schedule & 100 configured updates; report update 20; one fitting pass per collected batch \\
Saving / online evaluation & Every 10 updates; online evaluation also at update 0 and the end; 1,000 evaluation prompts \\
Training / offline test seeds & Training seed 20260923; paired test seed $20260923+i$ for index $i=0,\ldots,999$ \\
\bottomrule
\end{tabular}
\end{table}

\begin{table}[ht]
\centering\small
\caption{Method-specific training settings. A differentiated evaluation is a velocity evaluation carrying a parameter gradient, not necessarily a separate backward call. NFT accumulates eight state losses per trajectory, so its effective loss accumulation is $72\times8=576$ per GPU.}
\label{tab:dvoc_method_settings}
\begin{tabular}{@{}>{\raggedright\arraybackslash}p{.15\linewidth}>{\raggedright\arraybackslash}p{.20\linewidth}>{\raggedright\arraybackslash}p{.43\linewidth}>{\raggedright\arraybackslash}p{.12\linewidth}@{}}
\toprule
Method & Backward path & Reward target or loss & Weight decay \\
\midrule
ReFL & One selected late-state evaluation & Negative standardized reward; loss scale $10^{-3}$; no hinge & $10^{-2}$ \\
DRTune & Three connected evaluations in a truncated rollout & Same reward loss and scale as ReFL; early-stop ratio 0.4 & $10^{-2}$ \\
DiffusionOPSD & One queried state & Positive and negative PGD targets; old-policy mixing; $J=2$, $\rho=0.1$, $\eta=1$ & $10^{-4}$ \\
DiffusionNFT & Eight independently queried states & Advantage-weighted positive/negative regression; reference-prediction MSE coefficient $10^{-4}$ & $10^{-4}$ \\
\rowcolor{blue!7}
DVOC ($K=3$) & Three cached states; detached state dependencies & Actual-endpoint PGD target; $J=2$, $\rho=0.1$, $\eta=1$; normalized target MSE & $10^{-4}$ \\
\bottomrule
\end{tabular}
\end{table}

\paragraph{DVOC settings.}
The active indices are sampled uniformly without replacement from all nine Euler steps; no $N/K$ correction rescales their sum. PGD uses two normalized ascent steps, radius $0.1\|\bm{y}\|_2$, step multiplier 1, zero momentum, and target-construction microbatch size one. The actual endpoint, optimized target, and cached rollout states are detached. The squared fitting residual is divided by its detached mean absolute displacement, floored at $10^{-5}$. Rollout generation uses three trajectories per GPU call and 24 calls per update; fitting uses 72 one-trajectory microbatches. The $K=1,2$ ablations change only the active-step count.

\paragraph{ReFL and DRTune settings.}
Both use a fixed 512-prompt training-set calibration to standardize the mixed reward, followed by a negative-reward loss with scale $10^{-3}$ and no hinge, auxiliary pretraining loss, or reference KL/MSE. ReFL samples one evaluation state in the final 25\% of the nine-step schedule. DRTune uses three connected trainable steps with an early-stop ratio of 0.4 and no explicit early-stop-step override. The rollout prefix is computed without gradients and the unused suffix is omitted. They share deterministic random streams for prompts, noise, state selection, and calibration; exact stream settings are retained in the reproduction configuration.

\paragraph{DiffusionOPSD and DiffusionNFT settings.}
Both generate 24 batches of three trajectories per GPU and use prompt-group advantages clipped at 5, with both advantage signs enabled and policy-mixing coefficient 1. OPSD forms ascent and descent targets at the old policy's one-step clean prediction from a queried rollout state. Its requested query noise level is 0.273, mapped to the available schedule. Each target uses two normalized steps, relative radius 0.1, multiplier 1, zero momentum, and microbatch one; both targets are accepted without reward certification. It fits one queried state and applies no auxiliary reference-prediction penalty. NFT uses eight queried states, from $\lfloor9\times0.99\rfloor=8$, and a reference-prediction MSE coefficient of $10^{-4}$. This coefficient is distinct from the policy-mixing coefficient and the optimizer's weight decay. The extra state losses explain why the shared trajectory batch is not a shared backward-compute budget.

\paragraph{Training time measurement.}
For the five main-method rows, \Cref{tab:dvoc_training_time} uses original SwanLab scalar-event timestamps emitted at optimizer completion, cross-checked against all 100 corresponding loss values in each trainer's JSONL log. We use the scalar's producer timestamp, not the later storage-record timestamp. For completed update $s$, the interval is $\tau_s-\tau_{s-1}$. Within the first 20 updates, we omit update 1, which has no preceding completion event, and update 11, whose interval crosses the evaluation/checkpoint boundary at update 10. The remaining 18 intervals measure the recurring training cycle, including sampling, reward/target construction, fitting, and routine logging and communication overhead. They exclude initialization and intervals containing periodic evaluation/checkpoint saving; the separate offline seven-evaluator job is also excluded.

For $K=1,2$, the original node-local scalar archives were not retained. We reconstruct the same non-evaluation cycle intervals from modification timestamps of the original per-update reward records, written after rollout collection. These intervals span the preceding fit and the next rollout, rather than optimizer completion to completion, and are marked $\dagger$. Applying this reconstruction to $K=3$ gives 458.048\,s versus 458.046\,s from its scalar events over the same 18 intervals (largest paired interval difference: 1.40\,s). The agreement supports these reconstructed means; the file timestamps are not independent profiler measurements.

\begin{table}[ht]
\centering\small
\caption{Observed training time per optimizer update on one eight-GPU node. Means and standard deviations use 18 intervals from updates 2--20, excluding update 11 because it follows evaluation and checkpointing. Each update uses 576 training trajectories. The deviation describes within-run timing variation, not variation across training seeds. $\dagger$ marks reconstruction from original per-update reward-record timestamps; other rows use optimizer-completion events.}
\label{tab:dvoc_training_time}
\begin{tabular}{lrrrr}
\toprule
Method & GPUs & Seconds / update & Minutes / update & GPU-min / update \\
\midrule
ReFL & 8 & $272.8\pm1.0$ & 4.55 & 36.4 \\
DRTune & 8 & $336.8\pm1.3$ & 5.61 & 44.9 \\
DiffusionOPSD & 8 & $397.6\pm0.8$ & 6.63 & 53.0 \\
DiffusionNFT & 8 & $957.2\pm3.6$ & 15.95 & 127.6 \\
DVOC ($K=1$)$^\dagger$ & 8 & $325.8\pm0.6$ & 5.43 & 43.4 \\
DVOC ($K=2$)$^\dagger$ & 8 & $391.2\pm0.4$ & 6.52 & 52.2 \\
\rowcolor{blue!7}
DVOC ($K=3$) & 8 & $458.0\pm0.5$ & 7.63 & 61.1 \\
\bottomrule
\end{tabular}
\end{table}

The same exclusion rule over the 100-update runs gives mean times of 272.9\,s (ReFL), 337.2\,s (DRTune), 396.8\,s (DiffusionOPSD), 962.3\,s (DiffusionNFT), and 456.5\,s (DVOC), using 90 intervals per method. As a consistency check, the independent trainer timers for ReFL and DRTune average 272.7\,s and 336.6\,s over their first 20 updates. The measurements characterize these implementations and hardware allocations. They do not establish equal FLOPs or quality at an equal wall-time budget.

\paragraph{Scorer definitions.}
The reported HPSv3++ scores use the explicit iteration condition 0.3 for every method. CLIP, AES, and ImageReward use the same project scorer APIs and checkpoint files across methods. Q-Align uses the original OneAlign checkpoint and its separate \texttt{quality} and \texttt{aesthetics} tasks; DeQA uses \texttt{DeQA-Score-Mix3}. The latter two run in FP16 with eager attention and batch size one, with TF32 disabled. Their checkpoint identifiers and hashes, together with the exact score JSON, are retained with the manuscript's reproduction evidence. AES and Q-Align aesthetics are distinct evaluators and appear in separate columns.

\paragraph{Image and score provenance.}
At evaluation, each trained checkpoint has \texttt{epoch\_completed=20} and \texttt{global\_step=20}; evaluation loads its EMA adapter from \texttt{checkpoint-20/lora}. The reference uses the original Z-Image-Turbo weights, without a LoRA adapter or optimizer updates. Every method generates the complete ordered 1,000-prompt test set. Pixel values are clipped after multiplication by 255 and cast to unsigned 8-bit PNGs; all evaluators score these same files. Initial-latent and image SHA-256 hashes are retained per sample. The original in-training evaluation used a different image protocol and HPSv3++ condition 0.0; those logs are not mixed into this table.

\paragraph{Scope of the comparison.}
The trained rows compare fixed update-20 checkpoints from runs configured for a longer training schedule. They share a nominal generated-sample and optimizer-update budget but differ in gradient paths, reward calls, and targets, so their compute costs are not equal. The main table's gray differences are relative to the paired untrained original model. The in-domain/OOD split refers to evaluator use during optimization; it does not imply disjoint evaluator datasets.

\paragraph{Uncertainty and missing measurements.}
For paired metric differences $d_i$, the supplement reports $\bar d\pm1.96s_d/\sqrt{1000}$. These intervals describe variation over the fixed prompt/image pairs and are not adjusted for testing multiple metrics. They do not capture training-seed variation. Relative to the untrained reference, DVOC improves HPSv3++ and four held-out scores; CLIP's interval contains zero. DeQA decreases by 0.20233, with interval $[-0.21682,-0.18784]$: evaluator transfer is not uniform. We make no new claim of human preference, diversity preservation, convergence at update 20, or training-speed superiority. These would require additional measurements.
\begin{table}[ht]
\centering
\caption{Paired differences: default DVOC ($K=3$) minus untrained Z-Image-Turbo. Intervals are normal approximations across paired prompts, not across training seeds; they are unadjusted for multiple metrics.}
\label{tab:dvoc_paired}
\begin{tabular}{lrrl}
\toprule
Metric & Mean difference & Wins / ties / losses & 95\% interval \\
\midrule
HPSv3++ & +7.06357 & 993 / 0 / 7 & [+6.90034, +7.22679] \\
CLIP & +0.00077 & 516 / 0 / 484 & [-0.00142, +0.00297] \\
Q-Align Q. & +0.11801 & 745 / 36 / 219 & [+0.10556, +0.13046] \\
Q-Align A. & +1.08532 & 995 / 0 / 5 & [+1.05159, +1.11906] \\
AES & +0.87600 & 989 / 0 / 11 & [+0.84961, +0.90238] \\
ImageReward & +0.31966 & 798 / 0 / 202 & [+0.28264, +0.35668] \\
DeQA & -0.20233 & 138 / 6 / 856 & [-0.21682, -0.18784] \\
\bottomrule
\end{tabular}
\end{table}

\subsection{Local Endpoint-Target Realization}
\label{app:dvoc_mechanism}
We test the implementation-matched connection in \Cref{app:dvoc_realization} with a bounded diagnostic on the existing four-GPU IDE. This experiment uses the DVOC $K=3$ update-20 EMA checkpoint and makes only temporary LoRA perturbations; it performs no optimizer or EMA updates. All weights are restored after each probe. It is separate from the main seven-evaluator comparison.

\paragraph{Fixed paired protocol.}
We choose eight distinct training-distribution prompts by ascending SHA-256 hash after excluding exact final-test strings, then use two predetermined noises per prompt. Every probe uses the original $1024^2$, nine-step Euler sampler and the same two-step endpoint target ($J=2$, $\rho=0.1$, $\eta=1$, zero momentum). The differentiable reward is $0.03846R_{\mathrm{HPSv3++}}+R_{\mathrm{CLIP}}$, with HPSv3++ condition 0.0. Its training-mode reward head contains dropout with probability 0.05. We retain that path and fix one common random realization per sample for both target-gradient evaluations and every counterfactual score. The results are conditional on these fixed reward realizations, not an exact expectation over dropout. No final-test score or checkpoint selection is performed.

Using one shared cached endpoint and target displacement $\bm{\delta}$, we compute four vector--Jacobian products without forming dense Jacobians:
\begin{equation}
 \bm{q}=\mathbf{M}^\top \bm{g},\qquad \bm{v}_{\rm full}=\mathbf{M}^\top\bm{\delta},\qquad
 \bm{v}_{\rm frozen}=\mathbf{B}^\top\bm{\delta},\qquad
 \bm{v}_{\rm sparse}=\mathbf{J}_{\mathcal{A}}^\top\bm{\delta}.
\end{equation}
The complete chain $\mathbf{M}$ retains state-to-state dependencies, while $\mathbf{B}$ freezes all nine cached states. We preselect four independent uniform size-three subsets for each sample. All four are used for subset sensitivity; the first predetermined subset is used for finite-update probes. The reference $\bm{q}$ is the complete-chain automatic-differentiation result under BF16 execution, using the convention discussed in \Cref{app:dvoc_realization}.

For each direction $\bm{v}$, we temporarily set
\begin{equation}
 \phi'=\phi+\epsilon\max(\|\phi\|_2,1)\frac{\bm{v}}{\|\bm{v}\|_2},
 \qquad\epsilon\in\{10^{-3},3\times10^{-3}\},
\end{equation}
regenerate from the identical noise, score, and restore $\phi$. Both prespecified scales and all 16 samples are reported. These are equal-parameter-norm direction probes, not AdamW training steps. Intervals use 10,000 bootstrap resamples of the eight prompts, retaining both noises within each prompt. They describe this small prompt set, not training-seed uncertainty.

\begin{table}[ht]
\centering\small
\caption{Local mechanism probes. Entries are means with prompt-bootstrap 95\% intervals. Complete-chain derivatives use BF16 automatic differentiation. Finite gains are measured after regeneration with restored, paired noise and reward randomness.}
\label{tab:dvoc_mechanism}
\begin{tabular}{lccc}
\toprule
Direction & Cosine with $q$ & $\Delta R$, $\epsilon=10^{-3}$ & $\Delta R$, $\epsilon=3\times10^{-3}$ \\
\midrule
Full reward VJP & $1.0000$ $[1.0000, 1.0000]$ & $-0.0038$ $[-0.0256, 0.0200]$ & $0.0204$ $[-0.0062, 0.0551]$ \\
Full target VJP & $0.5483$ $[0.3289, 0.7726]$ & $0.0030$ $[-0.0219, 0.0295]$ & $0.0081$ $[-0.0078, 0.0246]$ \\
Frozen nine states & $0.0139$ $[-0.0057, 0.0355]$ & $0.0236$ $[0.0121, 0.0357]$ & $0.0634$ $[0.0442, 0.0845]$ \\
\rowcolor{blue!7}
DVOC ($K=3$) & $-0.0033$ $[-0.0460, 0.0401]$ & $0.0269$ $[0.0177, 0.0367]$ & $0.0498$ $[0.0294, 0.0719]$ \\
\bottomrule
\end{tabular}
\end{table}

\begin{figure}[ht]
\centering
\includegraphics[width=\linewidth]{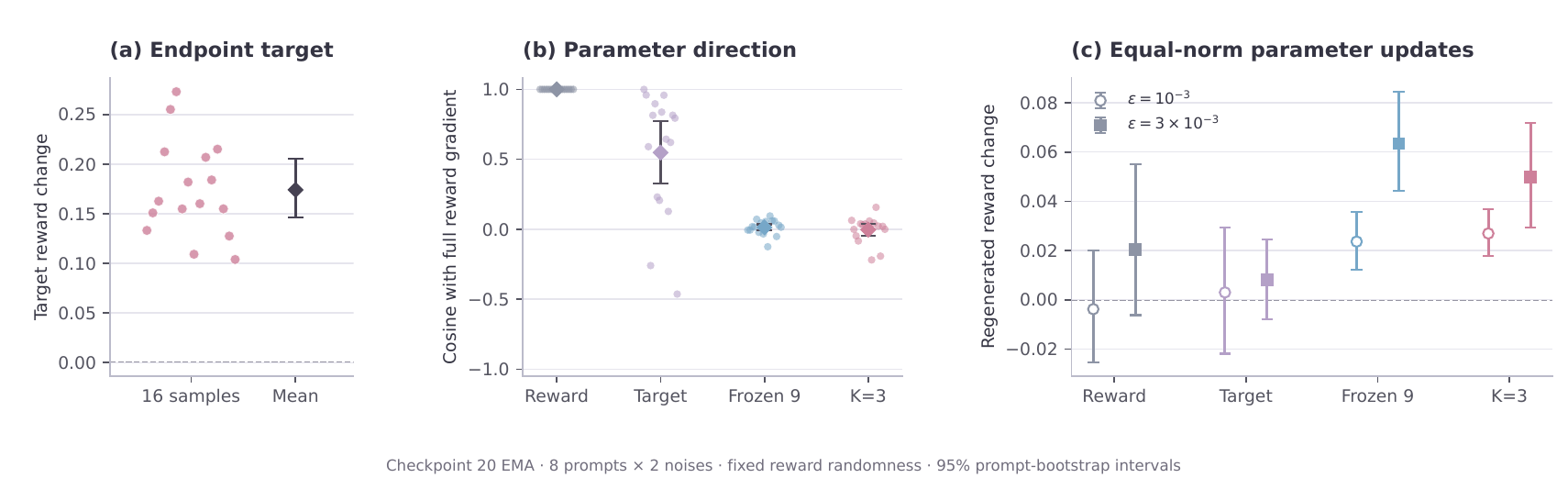}
\caption{Endpoint-target realization at the update-20 EMA checkpoint. (a) The two-step target's reward change on all 16 fixed samples. (b) Direction cosine with the complete-chain reward VJP; K=3 uses the first predetermined subset. (c) Actual regenerated reward change after the two fixed equal-norm perturbations. Diamonds and error bars in (a,b), and markers with error bars in (c), show means and prompt-bootstrap 95\% intervals. Finite reward gains do not establish agreement with the complete-chain gradient.}
\label{fig:dvoc_local_mechanism}
\end{figure}

\paragraph{What the diagnostic supports.}
All 16 endpoint targets improve the fixed-realization reward, with mean gain $0.1741$ and interval $[0.1459,0.2055]$. DVOC's K=3 finite perturbations yield mean regenerated gains of $0.0269$ and $0.0498$, with positive gains on 14/16 and 16/16 samples, respectively. This supports practical target realization for this checkpoint and probe protocol. The mean K=3 cosine with the frozen-nine direction is 0.834, while the latter is itself poorly aligned with the complete-chain reward VJP.

\paragraph{What it does not establish.}
The mean K=3 cosine with $\bm{q}$ is $-0.0033$, with interval $[-0.0460,0.0401]$. Averaging all four subsets within each sample gives $0.0174$ with interval $[-0.0034,0.0372]$; 22/64 subset directions have negative inner product with $\bm{q}$. The experiment therefore does not establish accurate full-gradient approximation or uniform satisfaction of the local ascent condition. At the prescribed scales, even the full-reward direction has mean first-order predictions of $21.74$ and $65.23$, against actual changes of $-0.0038$ and $0.0204$. These finite probes are not evidence of agreement with a first-order expansion. The diagnostic does not identify the respective roles of higher-order effects and mixed-precision differentiation, or demonstrate superiority under an equal training budget.

\paragraph{Numerical and resource checks.}
Across all 16 samples, complete-chain replay reproduces the cached endpoint exactly, and restoring the parameters gives zero endpoint and reward discrepancy. The largest relative error in the requested perturbation norm is $5.7\times10^{-8}$. This rules out a mismatched forward implementation and residual weight changes, but does not validate derivatives of floating-point rounding. The four cards identify as NVIDIA L20X; peak allocated memory is 44.22\,GiB per card. The successful sample computations total 13.60 GPU-minutes, excluding model initialization; the four-worker phase finishes in 4.67 wall-clock minutes including its initialization. The evidence supports the endpoint-target and sparse-regression interpretation, without proving equivalence to path-space VOC, finite AdamW guarantees, or OOD transfer.

\subsection{Paired Endpoint-Target Ablations}
\label{app:dvoc_target_ablation}
We isolate target construction using three new paired runs: the actual endpoint with two reward-ascent steps ($J=2$), the actual endpoint with one step ($J=1$), and a current-policy late-state proxy with two steps. All use $K=3$, 20 optimizer updates, and eight GPUs per run. These runs share an initial LoRA adapter and deterministic streams for training prompts, initial noise, active states, fitting order, and reward randomness. Their paired streams differ from the historical main experiment; the endpoint $J=2$ control therefore has its own scores and does not replace the main-table DVOC row.

\paragraph{Controlled target change.}
The endpoint control starts at the actual terminal latent $\bm{y}$. The proxy starts at $\bm{a}=\bm{x}_{t_q}-\sigma_q \bm{v}_\phi(\bm{x}_{t_q},t_q)$ from the current cached rollout, using the available schedule point nearest the requested noise level 0.273 ($q=7$, $\sigma_q\simeq0.3050$). Both use radius $r=0.1\|\bm{y}\|_2$, zero momentum, multiplier $\eta=1$, and normalized ascent-step length $r/J$. The proxy's displacement is transported to the actual endpoint, giving target $\bm{y}+(\bm{a}^+-\bm{a})$. Its anchor is detached and it introduces neither an old-policy branch nor a negative target. Thus it tests the target anchor within DVOC, rather than reproducing DiffusionOPSD. The first reward-randomness stream is shared between $J=1$ and $J=2$; the second step uses a separate fixed stream. The total nominal ascent length is the same, but the final displacement need not be. Projection is generally inactive at this step size, so this experiment does not isolate a benefit from clipping.

\paragraph{Training and evaluation.}
The runs retain the shared Z-Image-Turbo setup: $1024^2$ resolution, nine Euler steps, 576 trajectories per update, LoRA rank 32/alpha 64, and AdamW with learning rate $3\times10^{-4}$, weight decay $10^{-4}$, betas $(0.9,0.999)$, and gradient clipping at 1. Training uses $0.03846R_{\mathrm{HPSv3++}}+R_{\mathrm{CLIP}}$ with HPS condition 0.0 and the original EMA schedule. After an infrastructure interruption, all runs resume from complete update-10 checkpoints; raw adapters, optimizer, EMA, and all eight ranks' random states are verified on restoration. The final update-20 EMA adapters generate 1,000 paired PNGs with identical prompts, seeds, and initial-latent hashes. All seven evaluators score the same saved images, with HPS condition 0.3, as in the main table. The image hashes, 21,000 individual scores, and complete 20-update timing records are retained in the portable evidence.

\begin{table}[ht]
\centering\small
\caption{Paired target ablations at update 20 (EMA), with $K=3$ throughout and 1,000 matched test images. All scores are higher-is-better. Bold marks the highest mean in this three-run comparison, not statistical significance. The endpoint $J=2$ row is a fresh paired control and does not replace the historical main-table row.}
\label{tab:dvoc_target_ablation}
\setlength{\tabcolsep}{4pt}
\begin{tabular}{@{}lrr|rrrrr@{}}
\toprule
& \multicolumn{2}{c|}{In-domain} & \multicolumn{5}{c}{OOD} \\
Variant & HPSv3++ & CLIP & Q-Quality & Q-Aesth. & AES & ImageReward & DeQA \\
\midrule
\rowcolor{blue!7}
Endpoint, $J=2$ & \textbf{13.727} & 0.2760 & \textbf{4.947} & \textbf{4.640} & \textbf{6.281} & 1.2678 & \textbf{4.252} \\
Endpoint, $J=1$ & 13.653 & \textbf{0.2781} & 4.927 & 4.629 & 6.211 & 1.2673 & 4.131 \\
Proxy, $J=2$ & 13.705 & 0.2763 & 4.935 & 4.618 & 6.277 & \textbf{1.2684} & 4.242 \\
\bottomrule
\end{tabular}
\end{table}

\begin{table}[ht]
\centering\small
\caption{Metric differences relative to endpoint $J=2$: variant minus control, with paired-prompt bootstrap 95\% intervals (10,000 resamples). Negative differences favor the control. Intervals are per metric, without multiple-comparison adjustment or training-seed uncertainty.}
\label{tab:dvoc_target_intervals}
\begin{tabular}{@{}lcc@{}}
\toprule
Metric & Endpoint $J=1$ minus endpoint $J=2$ & Proxy $J=2$ minus endpoint $J=2$ \\
\midrule
HPSv3++ & $-0.0736\;[-0.1417,\,-0.0030]$ & $-0.0222\;[-0.0806,\,+0.0379]$ \\
CLIP & $+0.0021\;[+0.0010,\,+0.0032]$ & $+0.0004\;[-0.0006,\,+0.0013]$ \\
Q-Align quality & $-0.0202\;[-0.0247,\,-0.0160]$ & $-0.0125\;[-0.0158,\,-0.0094]$ \\
Q-Align aesthetics & $-0.0108\;[-0.0226,\,+0.0009]$ & $-0.0215\;[-0.0298,\,-0.0134]$ \\
AES & $-0.0692\;[-0.0829,\,-0.0556]$ & $-0.0038\;[-0.0145,\,+0.0070]$ \\
ImageReward & $-0.0005\;[-0.0198,\,+0.0185]$ & $+0.0006\;[-0.0160,\,+0.0167]$ \\
DeQA & $-0.1203\;[-0.1283,\,-0.1122]$ & $-0.0100\;[-0.0156,\,-0.0044]$ \\
\bottomrule
\end{tabular}
\end{table}

\paragraph{One versus two target steps.}
Two steps improve the mean HPSv3++, Q-Align quality, AES, and DeQA scores relative to one step, while CLIP favors $J=1$. The per-metric paired intervals exclude zero for these five differences; the HPS difference is small and its interval lies close to zero. Q-Align aesthetics and ImageReward intervals include zero. In particular, the endpoint $J=2$ gains are 0.0202 in Q-Align quality, 0.0692 in AES, and 0.1203 in DeQA, at a CLIP decrease of 0.00210. This is a tradeoff between evaluators, not uniform dominance.

\paragraph{Actual endpoint versus current-policy proxy.}
Endpoint $J=2$ improves Q-Align quality by 0.0125, Q-Align aesthetics by 0.0215, and DeQA by 0.0100 relative to proxy $J=2$; their per-metric paired intervals exclude zero. The intervals for HPSv3++, CLIP, AES, and ImageReward contain zero. The two targets have similar observed training cost. These results support a modest benefit on selected held-out evaluators for optimizing at the actual endpoint. They do not establish equivalence to a full-chain gradient, a general advantage over all proxy methods, or robustness across training seeds. The 95\% intervals resample test prompts and are not corrected for multiple comparisons.

\begin{table}[ht]
\centering\small
\caption{Training cost of the target ablations, eight GPUs per run. Mean $\pm$ sample standard deviation covers updates 2--20; instrumented phase timers exclude evaluation, diagnostic probes, target verification, and checkpoint I/O. GPU-minutes are eight times wall-clock minutes, not a FLOP estimate.}
\label{tab:dvoc_target_cost}
\begin{tabular}{@{}lrrr@{}}
\toprule
Variant & Seconds/update & GPU-minutes/update & First update (s) \\
\midrule
Endpoint, $J=2$ & $461.7\pm1.5$ & 61.56 & 459.9 \\
Endpoint, $J=1$ & $414.1\pm1.4$ & 55.21 & 415.4 \\
Proxy, $J=2$ & $463.0\pm0.8$ & 61.73 & 461.0 \\
\bottomrule
\end{tabular}
\end{table}

\paragraph{Cost and resource accounting.}
The scheduler class is \texttt{h200}; device checks identify eight NVIDIA L20X cards per run, with 139.8\,GiB per card, matching the earlier resource report. Timers record the slowest rank for the complete training update and exclude diagnostics, evaluation, and checkpoint I/O. The 19 updates from 2 through 20 include both sides of the verified restoration; downtime is excluded. Endpoint $J=1$ takes 414.1\,s/update versus 461.7\,s for $J=2$, reducing measured time by 10.3\%. Proxy $J=2$ takes 463.0\,s/update. The lower cost of $J=1$ accompanies lower quality scores on several evaluators; this fixed-update experiment does not compare quality at equal wall time.

\paragraph{Actual optimizer updates.}
To complement the temporary equal-norm probes in \Cref{app:dvoc_mechanism}, we evaluate raw adapters immediately before and after actual AdamW updates 1, 10, and 20. We regenerate the same 16 training-distribution samples with common reward-dropout realizations and HPS condition 0.0. Probe sampling uses isolated random states and does not advance training streams. The diagnostic includes the optimizer's adaptive scaling and accumulated fitting batch, so it is not the single-sample directional test used earlier. All 60 measured updates have finite nonzero gradients and parameter changes; the recorded gradient-clipping fraction is zero for all three runs.

\begin{table}[ht]
\centering\small
\caption{Actual AdamW updates: regenerated reward change on 16 fixed training-distribution samples (eight prompts, two noises each), using raw adapters immediately before and after each update. Intervals resample the eight prompts with both noises retained. These are distinct from the temporary equal-norm perturbations in \Cref{app:dvoc_mechanism}.}
\label{tab:dvoc_adamw_probes}
\begin{tabular}{@{}lrrrr@{}}
\toprule
Variant & Update & Mean reward change & 95\% prompt interval & Positive samples \\
\midrule
Endpoint, $J=2$ & 1 & $+0.02296$ & $[+0.01132,\,+0.03600]$ & 15/16 \\
Endpoint, $J=2$ & 10 & $+0.00457$ & $[-0.00841,\,+0.01586]$ & 11/16 \\
Endpoint, $J=2$ & 20 & $+0.00184$ & $[-0.00726,\,+0.01045]$ & 10/16 \\
Endpoint, $J=1$ & 1 & $+0.01464$ & $[+0.00359,\,+0.02659]$ & 13/16 \\
Endpoint, $J=1$ & 10 & $+0.01858$ & $[+0.00523,\,+0.03186]$ & 11/16 \\
Endpoint, $J=1$ & 20 & $+0.01142$ & $[+0.00292,\,+0.02286]$ & 13/16 \\
Proxy, $J=2$ & 1 & $+0.01776$ & $[+0.00202,\,+0.03426]$ & 14/16 \\
Proxy, $J=2$ & 10 & $-0.00759$ & $[-0.02719,\,+0.00878]$ & 8/16 \\
Proxy, $J=2$ & 20 & $+0.00644$ & $[-0.00533,\,+0.01794]$ & 10/16 \\
\bottomrule
\end{tabular}
\end{table}

Endpoint $J=2$ has positive mean regenerated reward changes at all three inspected updates, but only update 1 has an interval wholly above zero. At update 20 its gain is 0.00184 with interval $[-0.00726,0.01045]$, and 10/16 individual samples improve. Endpoint $J=1$ has positive intervals at all three inspected updates; the proxy's update-10 mean is negative with an interval containing zero. These observations connect the target-regression implementation to actual optimizer steps while showing that finite per-update improvement is not guaranteed. They neither reverse the earlier finding of poor complete-gradient alignment nor substitute for the held-out seven-evaluator comparison.

\section{Additional Experimental Results}
\label{appsec:additional_exp_result}
We provide supplementary quantitative and qualitative evaluations to further validate the behavior of the evaluated guidance and post-training settings.

\textbf{Computational efficiency.} \Cref{tab:performance_comparison} details the computational overhead of $\ours$. Due to the inclusion of vector-Jacobian products, $\ours$ increases runtime by $1.46\times$ to $1.98\times$ and peak GPU memory by $1.21\times$ to $1.72\times$ across varying resolutions. Despite this overhead, $\ours$ remains practically feasible, with the per-device sampling benchmark fitting on one NVIDIA RTX 4090 GPU; this statement does not describe the new DVOC training allocation.

\begin{table}[ht]
\centering
\caption{Runtime and peak memory of $\ours$ relative to CFG, reported as CFG / $\ours$ with the ratio in parentheses. Res/BS denotes resolution and batch size; all measurements are taken on a single NVIDIA RTX 4090.}
\label{tab:performance_comparison}
\begin{tabular}{lcc}
\toprule
Model    & Res/BS & CFG/ $\ours$ Runtime (sec)           \\ \midrule
EDM2-S   & 64/8   & 3.66 / 5.33 (1.46x)\\
DiT-XL/2 & 256/8  & 10.97 / 17.04 (1.55x)\\
EDM2-S   & 512/8  & 4.66 / 9.22 (1.98x)\\
DiT-XL/2 & 512/8  & 39.74 / 65.03 (1.64x)\\ \midrule \midrule
Model    & Res/BS & CFG/ $\ours$ Memory (GB)             \\ \midrule
EDM2-S   & 64/1   & 1.18 / 1.79 (1.53x)\\
DiT-XL/2 & 256/1  & 4.14 / 4.99 (1.21x)\\
EDM2-S   & 512/1  & 1.19 / 1.81 (1.53x)\\
DiT-XL/2 & 512/1  & 4.18 / 7.18 (1.72x)            \\ \bottomrule
\end{tabular}
\end{table}

\textbf{Performance in few-step regimes (\Cref{tab:steps_25}).}
\Cref{tab:steps_25} compares 25-step sampling performance on DiT and SiT architectures.
For SiT, although the model is trained to predict the interpolant velocity field $\bm{v}_\theta(\bm{x},t)$, we follow the diffusion-SDE sampling setting by converting the learned velocity into the corresponding score:
\begin{equation}
\bm{s}_\theta(\bm{x},t)
=
\frac{1}{\sigma_t}
\frac{
\alpha_t \bm{v}_\theta(\bm{x},t)-\dot{\alpha}_t \bm{x}
}{
\dot{\alpha}_t\sigma_t-\alpha_t\dot{\sigma}_t
},
\end{equation}
where $\alpha_t$ and $\sigma_t$ denote the SiT interpolant coefficients. We then apply CFG and $\ours$ on this converted score and use the same diffusion SDE sampler for evaluation.
$\ours$ improves FID on both architectures, from $7.58$ to $4.44$ on DiT-XL/2 and from $6.20$ to $3.89$ on SiT-XL/2.
The two architectures trade Precision against Recall in opposite directions: on DiT-XL/2 the gain comes from coverage (Recall $0.429\!\rightarrow\!0.545$, Precision $0.784\!\rightarrow\!0.699$), whereas on SiT-XL/2 it comes from sample quality (Precision $0.685\!\rightarrow\!0.735$, Recall $0.555\!\rightarrow\!0.516$).
These configurations also use different tuned guidance scales ($\omega{=}1.50$ for DiT-XL/2, $\omega{=}1.10$ for SiT-XL/2); in both cases $\ours$ moves the operating point toward a better FID rather than trading one axis for the other at a fixed cost.
Comparing against the 50-step results in~\Cref{tab:step_fid}, the relative FID reduction on DiT-XL/2 grows from $19\%$ at 50 steps to $41\%$ at 25 steps, showing a larger relative benefit at the smaller tested step budget.

\begin{table}[ht]
    \centering
    \footnotesize
    \caption{Quantitative comparison of 25-step sampling on DiT-XL/2 and SiT-XL/2 ($256\times256$, ImageNet). We report FID ($\downarrow$), Recall ($\uparrow$), and Precision ($\uparrow$), together with the guidance scale $\omega$, the momentum coefficient $\beta$, and the VJP recomputation interval $K$.}
    \label{tab:steps_25}

    \begin{NiceTabular}{llcccccc}
        \toprule
        \textbf{Model} & \textbf{Method} & \textbf{FID} & \textbf{Recall} & \textbf{Precision} & \textbf{$\omega$} & \textbf{$\beta$} & \textbf{$K$} \\
        \midrule
        
        \Block[fill=white]{2-1}{DiT-XL/2} 
            & CFG & 7.58 & 0.429 & 0.784 & 1.50 & - & - \\
            \rowcolor{lightCyan}
            & $\ours$ & 4.44 & 0.545 & 0.699 & 1.50 & 0.15 & 2 \\
            
        \midrule
        
        \Block[fill=white]{2-1}{SiT-XL/2} 
            & CFG & 6.20 & 0.555 & 0.685 & 1.50 & - & - \\
            \rowcolor{lightCyan}
            & $\ours$ & 3.89 & 0.516 & 0.735 & 1.10 & 0.15 & 2 \\
            
        \bottomrule
    \end{NiceTabular}
    
\end{table}

\textbf{Quantitative analysis on 2D synthetic benchmarks.} \Cref{tab:toy-merged} provides the numerical evaluation corresponding to the 2D branching-tree visualizations in~\Cref{fig:toy_example}. While standard CFG effectively filters outliers (high Precision), it loses coverage as the distribution complexity increases, evidenced by the drop in Recall to $0.7785$ on the three-class dataset. $\ours$ achieves higher F1-scores and lower Wasserstein distances in both settings, indicating that it aligns samples with the conditional target while better preserving the support of the data manifold.

\textbf{Extended sensitivity analysis.}
\begin{figure*}[htbp]
    \centering
    \begin{subfigure}[t]{0.55\linewidth}
        \includegraphics[width=\linewidth]{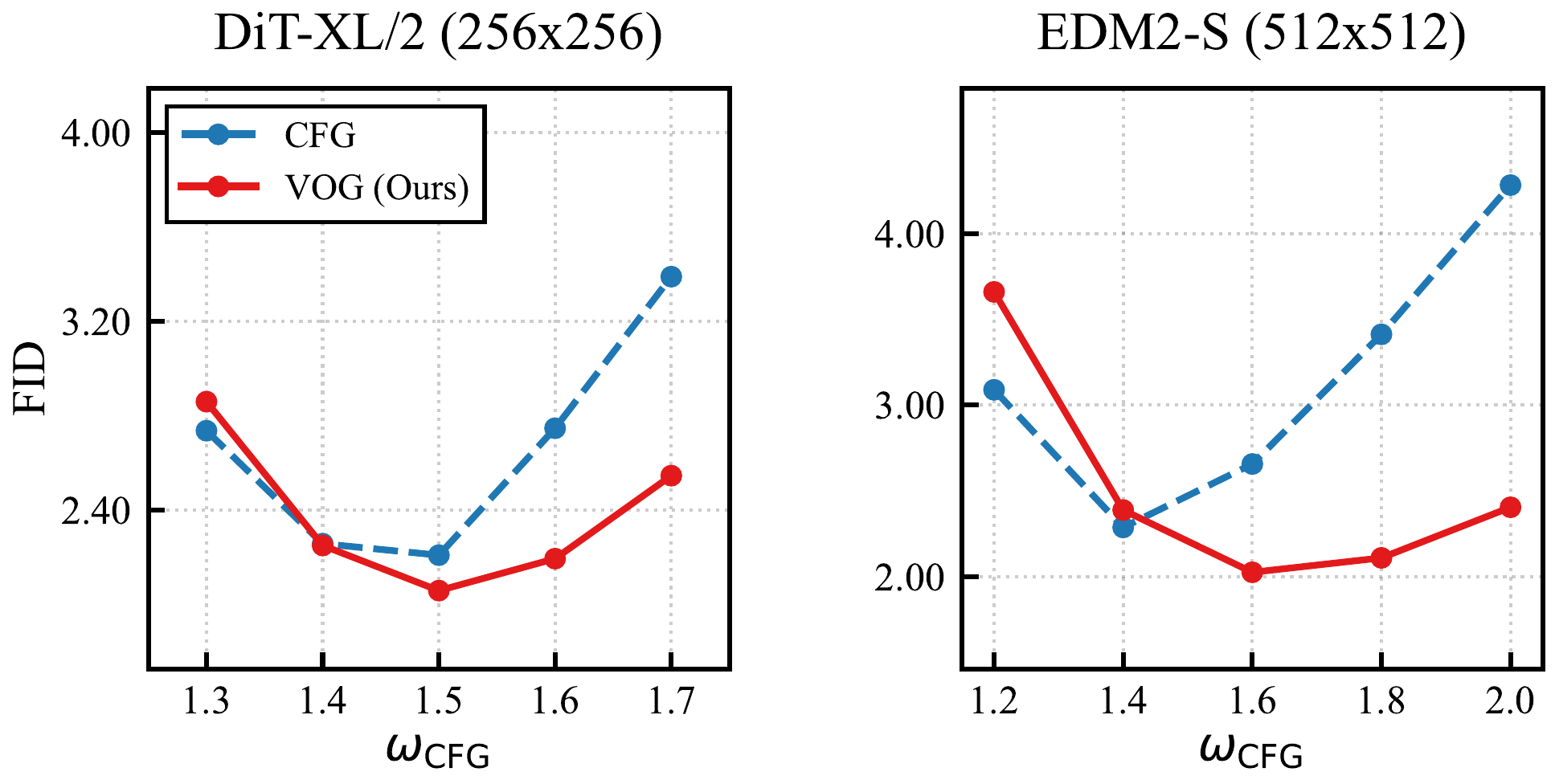}
        \caption{FID ($\downarrow$) across guidance scales $\omega_{\text{CFG}}$ on DiT-XL/2 ($256\times256$) and EDM2-S ($512\times512$).}
        \label{fig:ablation_omega}
    \end{subfigure}
    \hfill
    \begin{subfigure}[t]{0.42\linewidth}
        \includegraphics[width=\linewidth]{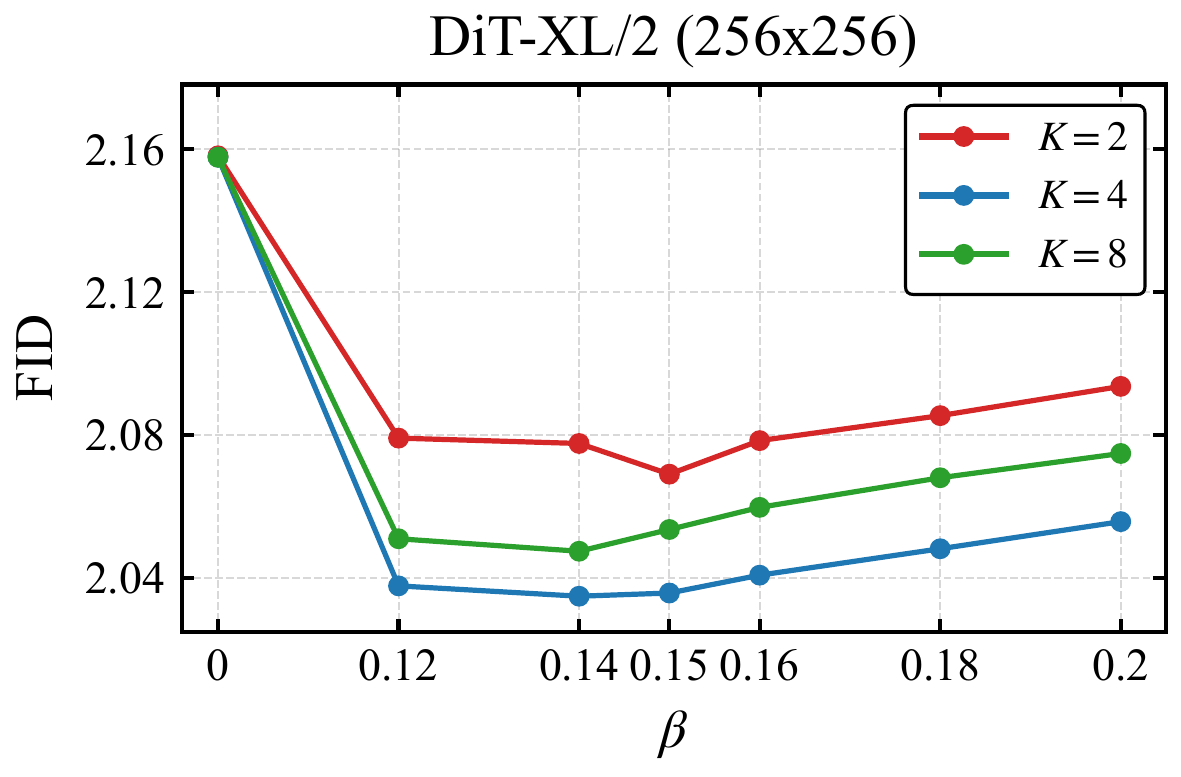}
        \caption{FID ($\downarrow$) across momentum coefficients $\beta$ and VJP recomputation intervals $K_{\mathrm{VOG}}$ on DiT-XL/2 ($256\times256$).}
        \label{fig:ablation_momentum}
    \end{subfigure}
    \caption{\textbf{Guidance-scale robustness and hyperparameter sensitivity for VOG.} \textbf{(a)} VOG widens the usable guidance-scale window compared to standard CFG. \textbf{(b)} VOG maintains stable FID gains across VJP recomputation strides $K_{\mathrm{VOG}}$ and momentum smoothing factors $\beta$.}
\end{figure*}
\Cref{fig:ablation_edm2,fig:ablation_dit} supplement the main text with guidance-scale sweeps on EDM2-S and DiT-XL/2. Across the sweep, $\ours$ mitigates the coverage loss that standard CFG exhibits at high guidance scales (e.g., maintaining Recall $\approx 0.58$ where CFG drops below $0.50$) and avoids the accompanying FID spike, which widens the range of $\omega$ over which guidance remains usable.

\textbf{Composition with other guidance heuristics.}
$\ours$ is complementary to existing heuristic schedules rather than a replacement for them. Applied on top of Cosine, Interval, and AutoGuidance (\Cref{tab:edm2_s_ag}), it improves FID in every case, for instance reducing Interval CFG from $1.637$ to $1.523$ on EDM2-L, which supports its use as a plug-in guidance correction.

\begin{table}[!htbp]
  \centering
  \caption{\textbf{Quantitative results on 2D branching-tree generation.} We report Precision, Recall, F1-score (higher is better), and Wasserstein distance (lower is better) for both Two-class and Three-class settings. Best results (excluding Ground Truth) are highlighted in \textbf{bold}.}
  \label{tab:toy-merged}
  \setlength{\tabcolsep}{3.5pt}
  \begin{tabular}{l cccc c cccc}
    \toprule
    & \multicolumn{4}{c}{\textbf{Two-class (CLS=2)}} & & \multicolumn{4}{c}{\textbf{Three-class (CLS=3)}} \\
    \cmidrule(lr){2-5} \cmidrule(lr){7-10}
    {Method} & {Prec.} & {Rec.} & {F1} & {Wass.} & & {Prec.} & {Rec.} & {F1} & {Wass.} \\
    \midrule
    Ground Truth & 0.9916 & 1.0000 & 0.9958 & 0.0021 & & 0.9862 & 1.0000 & 0.9931 & 0.0026 \\
    \midrule
    No Guidance  & 0.8948 & \textbf{0.9989} & \textbf{0.9440} & \textbf{0.0031} & & 0.9194 & \textbf{0.9970} & \textbf{0.9567} & \textbf{0.0035} \\
    Standard CFG & 0.9110 & 0.8755 & 0.8929 & 0.0064 & & 0.9392 & 0.7785 & 0.8514 & 0.0071 \\
    \rowcolor{lightCyan}
    $\ours$ (ours)     & \textbf{0.9230} & 0.9568 & 0.9396 & 0.0043 & & \textbf{0.9536} & 0.9328 & 0.9431 & 0.0042 \\
    \bottomrule
  \end{tabular}
\end{table}

\begin{figure}[!htbp]
    \centering
    \includegraphics[width=1\linewidth]{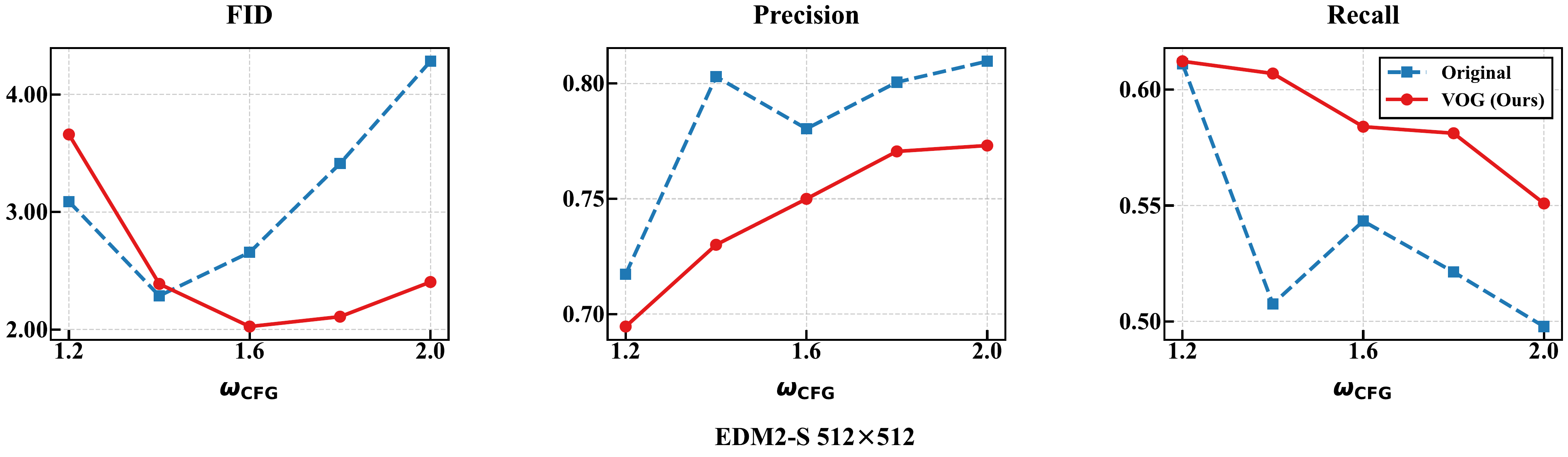}
\caption{\textbf{Quantitative comparison on EDM2-S ($512\times512$) across varying guidance scales.} We report FID ($\downarrow$), Precision ($\uparrow$), and Recall ($\uparrow$) as a function of the guidance weight $\omega_{\text{CFG}}$. VOG (red) demonstrates superior robustness at high scales compared to standard CFG (blue), maintaining an FID of $\approx 2.4$ at $\omega=2.0$ while the baseline degrades beyond 4.0.}
\label{fig:ablation_edm2}
\end{figure}

\begin{figure}[!htbp]
    \centering
    \includegraphics[width=1\linewidth]{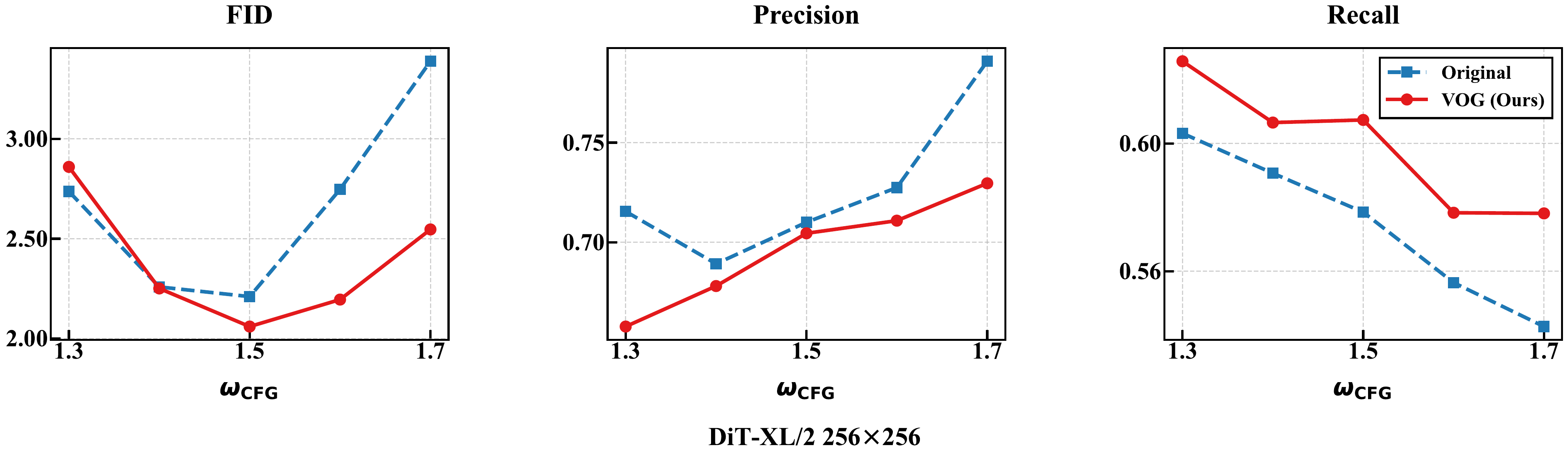}
\caption{\textbf{Quantitative comparison on DiT-XL/2 ($256\times256$) across varying guidance scales.} While standard CFG suffers from sharp performance degradation at high guidance strengths (e.g., FID $>3.40$ at $\omega=1.7$), VOG maintains structural fidelity (high Recall) and generation quality (FID $<2.60$), effectively preventing manifold drift.}
\label{fig:ablation_dit}
\end{figure}

\begin{table}[!htbp]
    \centering

    \caption{$\ours$ composed with different guidance heuristics on ImageNet $64\times64$. Each shaded row applies $\ours$ on top of the heuristic immediately above it.}
    \begin{tabular}{lccc}
        \toprule
        Benchmark & FID & Precision & Recall \\
        \midrule

        \multicolumn{4}{l}{\textbf{EDM2-S}} \\ 
        \quad +AutoGuidance & 1.044 & 0.701 & 0.626 \\
        \rowcolor{lightCyan}
        \quad \quad +$\ours$ (Ours) & 1.031 & 0.696 & 0.620 \\

        \addlinespace[1.5ex]

        \multicolumn{4}{l}{\textbf{EDM2-L}} \\
        \quad +Standard CFG & 1.993 & 0.710 & 0.601 \\
        \rowcolor{lightCyan}
        \quad \quad +$\ours$ (Ours) & 1.418 & 0.714 & 0.597 \\
        
        \addlinespace[0.5ex]
        \quad +Cosine CFG & 1.667 & 0.726 & 0.587 \\
        \rowcolor{lightCyan}
        \quad \quad +$\ours$ (Ours) & 1.649 & 0.727 & 0.588 \\
        
        \addlinespace[0.5ex]
        \quad +Interval CFG & 1.637 & 0.730 & 0.597 \\
        \rowcolor{lightCyan}
        \quad \quad +$\ours$ (Ours) & 1.523 & 0.720 & 0.593 \\
        
        \bottomrule
    \end{tabular}
    \label{tab:edm2_s_ag}
\end{table}

\begin{table}[ht]
\centering
\caption{VOG with fewer sampling steps on DiT-XL/2. These retained ImageNet results are separate from the DVOC optimizer-step ablation.}
\label{tab:step_fid}
\begin{tabular}{lcc}\toprule
Guidance & 25 steps: FID$\downarrow$ & 50 steps: FID$\downarrow$ \\
\midrule
CFG & 7.58 & 3.69 \\
CFG + VOG & 4.44 & 2.99 \\
\bottomrule\end{tabular}
\end{table}
\subsection{DVOC Qualitative Examples}
The examples below use test indices 0 and 1, selected in dataset order without filtering on reward scores. Every column uses the same initial noise for a given row. The labels refer to the 20-update EMA checkpoints. The panel illustrates individual samples and does not measure diversity.
\begin{figure}[!htbp]
\centering
\includegraphics[width=\textwidth]{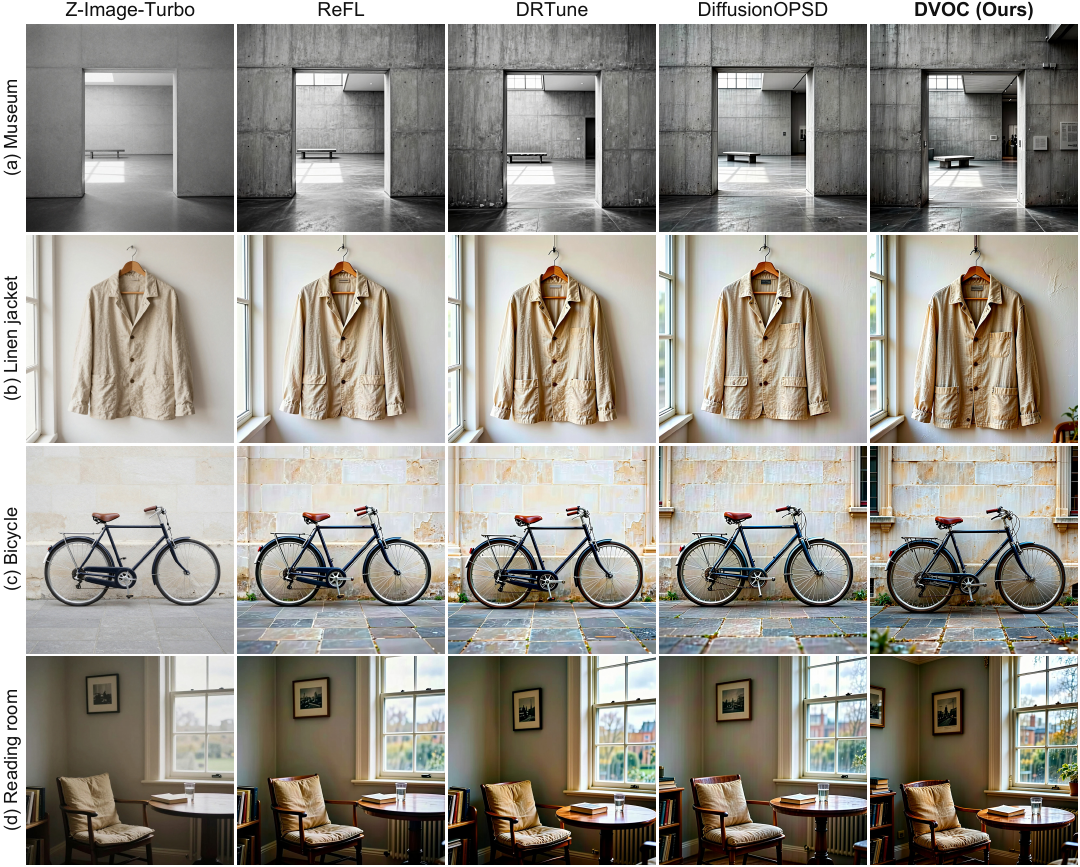}
\caption{\textbf{Qualitative comparison of reward-based adaptation.} Columns show Z-Image-Turbo, ReFL, DRTune, DiffusionOPSD, and DVOC. Rows show (a) a museum interior, (b) a linen jacket, (c) a bicycle, and (d) a reading room. The examples illustrate differences in material texture, lighting, and local detail across methods.}
\label{fig:dvoc_qualitative}
\end{figure}

\clearpage
\section{Reward and Evaluation Histories}
\label{app:training_curves}
We report the archived 100-update runs for ReFL, DRTune, DiffusionOPSD, and DVOC ($K=3$), with $K=1,2$ confined to the ablation. DiffusionNFT is omitted from these history plots. The primary seven-metric comparison remains the fixed, paired update-20 evaluation in the main paper.

\paragraph{Training reward and evaluation reward.}
The training reward is $0.03846R_{\mathrm{HPSv3++}}+R_{\mathrm{CLIP}}$. ReFL and DRTune score one-step clean predictions from sampled late states, whereas DVOC and DiffusionOPSD score complete sampled endpoints. These quantities are plotted in separate panels in \Cref{fig:dvoc_reward_curves}; the two upper panels share their vertical scale. The lower panels show the components available in the endpoint trainers' logs. ReFL/DRTune did not log those components separately, so no component trajectories are inferred for them.

All raw update means are shown as thin, translucent lines. The stronger line is a trailing mean over at most five updates, using the same window for every method. The horizontal coordinate is the optimization update associated with the sampled minibatch: endpoint trainers' zero-based, pre-update log indices are increased by one to match ReFL/DRTune's post-update counter. Thus each run contributes 100 actual observations at updates 1--100. These curves describe the reward measured on each update's training batch, not a held-out test score or a run-to-run confidence interval.

\begin{figure}[ht]
\centering
\includegraphics[width=\textwidth]{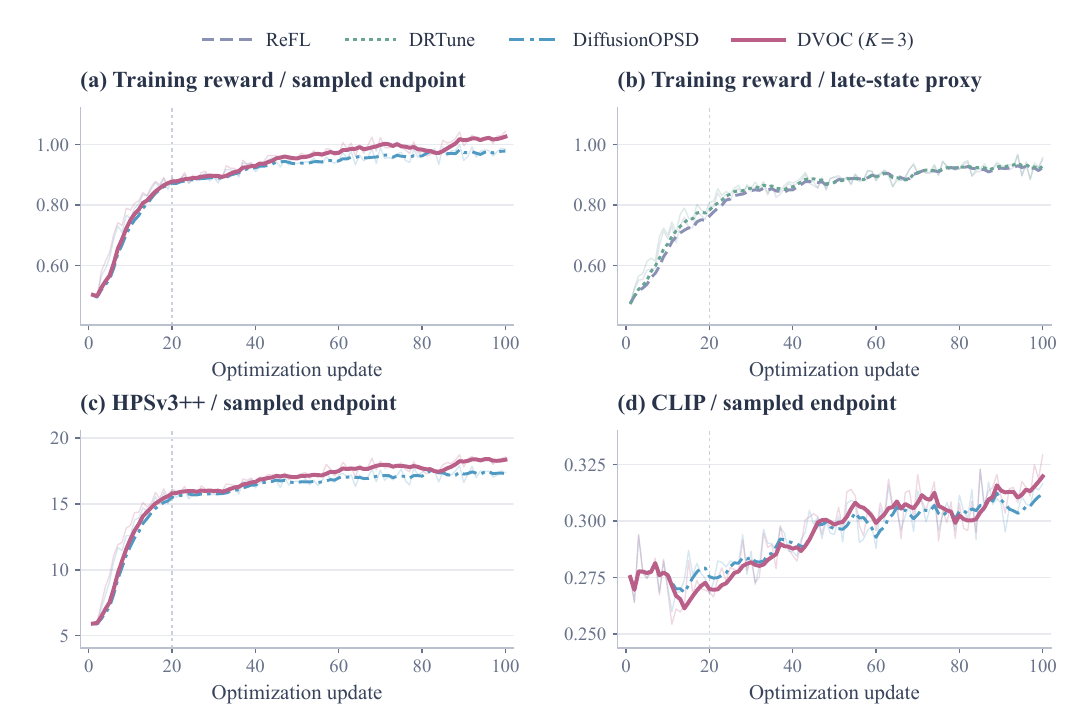}
\caption{\textbf{Training reward histories.} Separate panels distinguish sampled endpoints from late-state proxies. Thin lines show every raw update mean; strong lines show the trailing five-update mean. The vertical reference marks update 20. The HPSv3++ and CLIP panels contain the endpoint trainers for which component logs exist.}
\label{fig:dvoc_reward_curves}
\end{figure}

\clearpage
\subsection{Online Evaluation}
\label{app:online_evaluation}
The online evaluation records contain 11 checkpoints per method: initialization and every ten completed updates through update 100. Every evaluation uses the EMA adapter, 1,000 prompts, $1024^2$ output, nine sampling steps, guidance scale zero, and HPSv3++ iteration condition 0.0. In the archived workspace, \texttt{pickapic/test.txt} and \texttt{drawbench/test.txt} are byte-identical. The original trainer families nevertheless use different evaluation-noise protocols; these histories are not the paired-image evaluation of the main table. We retain each trainer's recorded mean and never insert the later offline rescoring into its online curve.

\Cref{fig:dvoc_eval_curves} reports HPSv3++ and CLIP as in-domain metrics, alongside the held-out ImageReward evaluator. No smoothing or confidence band is applied to these checkpoint means. Lines simply join adjacent recorded checkpoints; markers identify all measured points. The first two panels of \Cref{fig:dvoc_training_overview} show the same CLIP and HPSv3++ histories. Whenever a combined evaluation reward is reported below, it is recomputed as $0.03846\,\mathrm{HPSv3++}+\mathrm{CLIP}$, not the sum of the five legacy evaluator scores.

For DVOC, the online weighted evaluation reward rises from 0.766 at update 20 to 0.911 at update 100. ImageReward rises from 1.267 to 1.366 over the same interval. These are observations from one training run per method and do not establish improvement on every held-out evaluator.

\begin{figure}[ht]
\centering
\includegraphics[width=\textwidth]{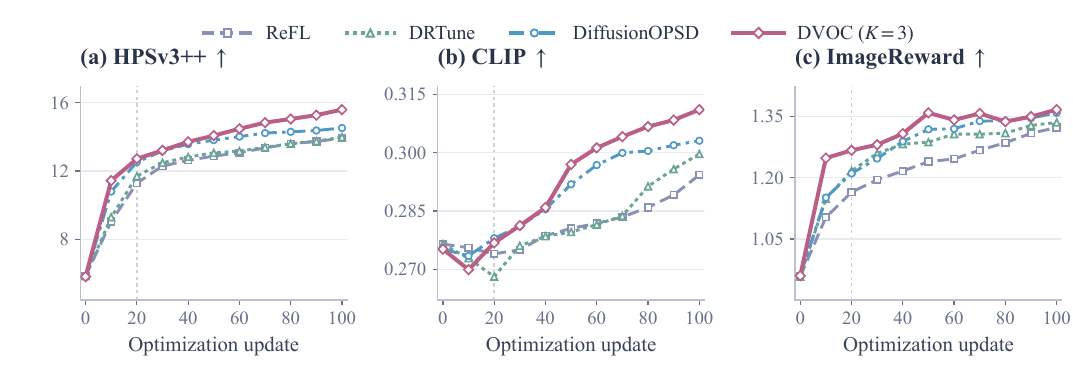}
\caption{\textbf{Online evaluation histories.} HPSv3++ and CLIP are in-domain evaluators; ImageReward is held out of training. Each marker is a recorded checkpoint mean, with no smoothing. The vertical line marks update 20. HPSv3++ condition 0.0 and the original evaluation-noise protocols distinguish these scores from the offline main table.}
\label{fig:dvoc_eval_curves}
\end{figure}

\clearpage
\subsection{Active-Step Histories}
\label{app:k_histories}
\Cref{fig:dvoc_k_curves} complements the fixed update-20 ablation with all recorded histories for $K=1,2,3$. Each run uses the same nine-step sampler and two-step PGD target construction. Panel (a) shows training reward with the same raw-line and trailing-mean convention as \Cref{fig:dvoc_reward_curves}; panels (b)--(f) show unsmoothed online EMA evaluation. The three variants have exactly matching initialization scores on all four displayed evaluators.

Larger $K$ improves the later weighted evaluation reward in these runs, while $K=1$ retains a higher late-training AES score. ImageReward varies without a consistent ordering across all checkpoints. These trends complement the update-20 trade-offs and do not support selecting a different $K$ for every metric in the main comparison. The curves compare equal update counts; they do not imply equal FLOPs or wall-clock cost.

\begin{figure}[ht]
\centering
\includegraphics[width=\textwidth]{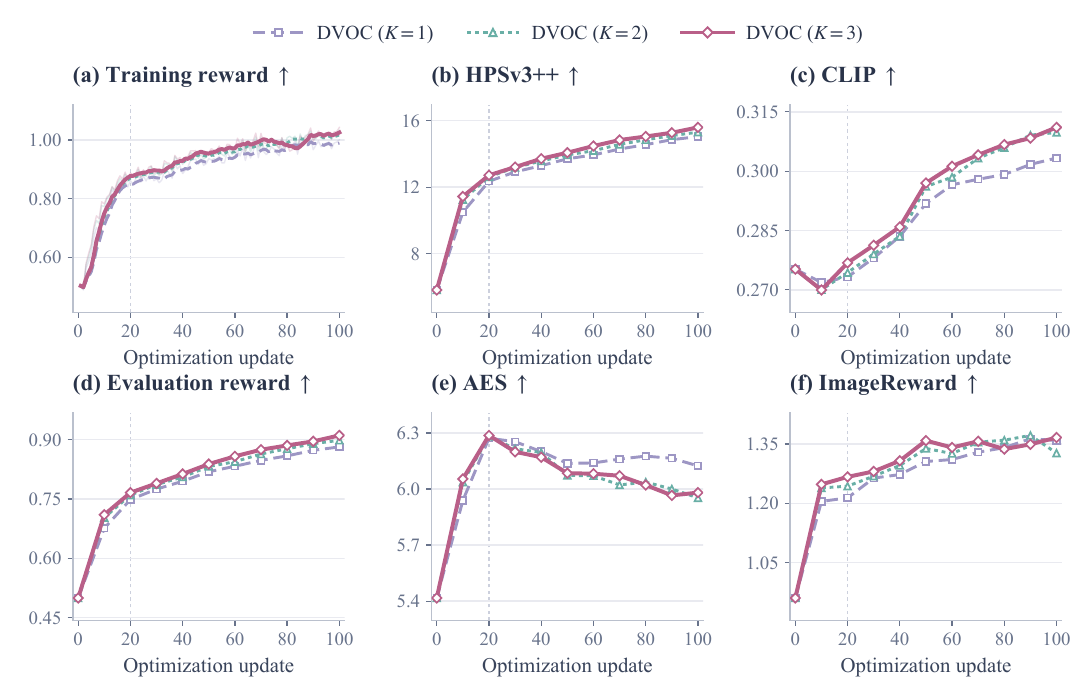}
\caption{\textbf{DVOC active-step ablation over optimization.} $K=1,2$ appear only in the ablation; the main method remains $K=3$. Training reward includes all raw points and a trailing five-update mean. Evaluation panels show only recorded checkpoint means. The vertical line marks the main paper's update-20 checkpoint.}
\label{fig:dvoc_k_curves}
\end{figure}

\FloatBarrier
{\small
\setlength{\bibsep}{3pt plus 1pt minus 1pt}
\bibliographystyle{ieeenat_fullname}\bibliography{ref}}
\end{document}